\documentclass[twoside,11pt]{article}

\usepackage{blindtext}
\usepackage[preprint]{arxiv2e}

\usepackage{amsmath}
\usepackage{amsfonts}
\usepackage{amssymb}
\usepackage{booktabs}
\usepackage{tikz}
\usepackage{comment}
\usepackage{silence}
\usepackage[compatibility=false]{caption}
\usepackage{subcaption}
\usepackage{multirow}
\usepackage{enumitem}
\usepackage{mathrsfs}
\usepackage{rotating}
\usepackage[linesnumbered,ruled,vlined]{algorithm2e}
\usepackage{etoolbox}
\makeatletter
\patchcmd{\@algocf@start}{\hbox to\algowidth\bgroup}{\hbox to\dimexpr\algowidth+34pt\relax\bgroup}{}{}
\makeatother
\usepackage{placeins}

\AtBeginDocument{\hfuzz=100pt\vfuzz=100pt\hbadness=10000\vbadness=10000}

\usepackage{lastpage}
\arxivheading{}{2026}{1--\pageref{LastPage}}{}{}{}{Hanxuan Ye and Hongzhe Li} 

\ShortHeadings{Multicalibration Boosting: Theory, Convergence, and Transferability}{Ye and Li}
\firstpageno{1}

\newcommand{\be}{\begin{eqnarray*}}
\newcommand{\ee}{\end{eqnarray*}}
\newtheorem{ass}{Condition}
\newtheorem{defn}{Definition}
\newtheorem{rem}{Remark}

\newcommand{\bbP}{\mathbb{P}}

\newcommand{\bbR}{\mathbb{R}}

\newcommand{\sH}{\mathcal{H}}

\newcommand{\sS}{\mathcal{S}}
\newcommand{\sT}{\mathcal{T}}

\newcommand{\sU}{\mathcal{U}}

\newcommand{\vf}{\mathbf{f}}

\newcommand{\vw}{\mathbf{w}}

\newcommand{\vy}{\mathbf{y}}

\newcommand{\Expect}{\mathbb{E}}

\usepackage{amsmath}
\DeclareMathOperator*{\argmax}{arg\,max}
\DeclareMathOperator*{\argmin}{arg\,min}

\usepackage{accents}

\def\v1{{\mathbf{1}}}
\def\b1{{\boldsymbol{1}}}
\newcommand{\sD}{\mathcal{D}}
\newcommand{\sF}{\mathcal{F}}
\newcommand{\sL}{\mathcal{L}}
\newcommand{\wh}{\widehat}
\newcommand{\wt}{\widetilde}

\newcommand{\Err}{\mathrm{Err}}
\newcommand{\sHa}{\mathcal{H}_{\text{aud}}}
\newcommand{\sHc}{\mathcal{H}_{\text{can}}}

\begin{document}

\title{Multicalibration Boosting: Theory, Convergence, and Transferability}

\author{\name Hanxuan Ye \email hanxuan.ye@pennmedicine.upenn.edu \\
       \addr Department of Biostatistics and Epidemiology
        \\
       University of Pennsylvania \\
       Philadelphia, PA 19104-6021, USA
       \AND
       \name Hongzhe Li \email hongzhe@pennmedicine.upenn.edu \\
       \addr Department of Biostatistics and Epidemiology \\
       University of Pennsylvania \\
       Philadelphia, PA 19104-6021, USA}

\maketitle

\begin{abstract}
Multicalibration extends classical calibration by requiring predictions to be unbiased over a rich collection of functions, encompassing both prediction slices and subpopulations. It has emerged as a powerful framework for fairness, robustness, and reliable prediction, yet the theoretical understanding of multicalibration boosting (MCBoost) remains fragmented and often relies on restrictive assumptions.
In this work, we develop a unified and refined perspective on MCBoost that subsumes existing variants, including multiaccuracy, BatchGCP, and BatchMVP. We uncover several phenomena that provide new insights into its practical behavior: even highly accurate and flexible predictors can remain substantially miscalibrated; enforcing multicalibration introduces a calibration–risk trade-off; and early stopping plays a central role in controlling this trade-off.
On the theoretical side, we establish a general framework for MCBoost under weaker and more realistic conditions. We show that the boosting iterates converge to a Bregman projection of the population-optimal predictor onto the cumulative span generated by the audit class, thereby explicitly characterizing the function space on which multicalibration is achieved. We further derive convergence rates under different smoothness assumptions, finite-sample guarantees, and principled stopping rules that ensure multicalibration at termination. Finally, we extend the theory of universal adaptability under covariate shift, providing more general transfer guarantees and clarifying when multicalibrated predictors generalize across domains. These results provide a more complete theoretical foundation and practical guidance for multicalibration boosting, positioning it as both a unifying framework and a reliable post-processing approach for modern predictive models.

\end{abstract}

\section{Introduction}
Multicalibration, introduced by~\cite{hebert2018multicalibration}, asks a predictor to be unbiased not only marginally but also after conditioning on both predicted values and subgroup membership. It therefore strengthens classical calibration, which is already regarded as a central property of probabilistic prediction~\citep{dawid1982well,naeini2015obtaining,guo2017calibration}. In the simplest binary-outcome setting, calibration requires that among units with predicted value $v$, the average outcome is also $v$; multicalibration extends this requirement to a rich family of subpopulations and prediction slices.

Since its introduction, the notion has broadened substantially. ~\cite{jung2021moment} extended multicalibration to real-valued outcomes and higher moments. ~\cite{gupta2021online} developed online and quantile-oriented variants for uncertainty prediction, while~\cite{bastani2022practical,jung2021moment} studied practical quantile calibration with connections to conditional coverage. A particularly useful unification is due to~\citet{deng2023happymap}, who replaced the residual $f(X)-Y$ by a general score function $s(Y,f)$ and thus encompassed a wide range of objectives, from fair uncertainty quantification to fairness constraints such as false-positive rate parity. More recently, the framework has been extended to the settings of vector-valued prediction~\citep{zhang2024fair}.

Beyond fairness, multicalibration has also emerged as a tool for robust learning under data heterogeneity. When training and test distributions differ but remain structurally related, post hoc multicalibration can yield predictors that adapt to new domains without retraining~\citep{kim2022universal}. This target-independent perspective connects multicalibration to domain adaptation and transfer learning. Empirically, boosting-based multicalibration has shown robust behavior across applications ranging from Gender Shades and Adult Income prediction~\citep{buolamwini2018gender,kim2019multiaccuracy} to epidemiological studies such as NHANES and NHIS~\citep{kim2022universal}; see also~\citet{ye2024multicalibration} for a survival-analysis extension based on pseudo-observations. 

The practical appeal of multicalibration boosting~(MCBoost) is clear: it can refine a pre-trained black box predictor through post-processing, regardless of whether it is well-trained or inadvertently chosen; it offers a route to target-independent adaptation when the relevant shifts lie near the learned calibration class, thus achieving universal adaptability~\citep{kim2022universal}; and it is especially attractive in multi-target settings where measuring demographic variables, estimating propensity scores are costly, or repeated access to raw data is difficult due to privacy constraints.

Despite rapid progress in both theory and applications, several core questions remain open. First, existing multicalibration algorithms are often developed for specific settings, and a unified algorithmic treatment remains relatively limited. While~\cite{deng2023happymap} provided a broad infinite-sample analysis, their framework implicitly assumes a black-box optimization oracle and complete knowledge of the underlying distribution, as also discussed by~\cite{gibbs2025conformal}. Their algorithm is presented as a high-level sketch where parameter choice, sample complexity, and explicit sample-based auditing procedures are not the main focus. 
Second, the role of weak learners is not yet well understood: the auditing class clearly shapes the updates, but its effect on the limiting predictor and on the ultimate multicalibrated class has not been characterized systematically. Third, finite-sample guarantees could be improved, especially for convergence rates under varying smoothness and for the stopping rules used in practice. Although some studies provide upper bounds on the number of boosting iterations~\citep{globus2023multicalibration}, a sharper characterization of convergence rates, sample complexity, and the effects of loss smoothness remains lacking. The theoretical justification for the practical stopping criterion, which underpins empirical success, has received limited attention.
Finally, although multicalibration has been linked to target-independent learning~\citep{kim2022universal}, the precise scope of its guarantees under covariate shift deserves a more general treatment.

\subsection{This paper}
Motivated by these questions, we develop a unified framework for understanding multicalibration boosting across objectives such as squared loss and pinball loss. Our main contributions are as follows. 
\begin{itemize}[leftmargin=*]
    \item[I] {\bf Unified formulation of MCBoost.} We propose a generalized multicalibration boosting procedure that subsumes the main variants in the literature, including standard calibration~\citep{roth2022uncertain}, multiaccuracy~\citep{kim2019multiaccuracy,kim2022universal}, BatchGCP~\citep{jung2022batch}, and BatchMVP~\citep{jung2022batch} for equalized coverage.
    \item[II] {\bf Limit characterization and the multicalibrated function class.} 
    Using convex analysis and Bregman geometry~\citep{bauschke1997legendre,bauschke2020correction}, we show that the MCBoost limit is the Bregman projection of the population-optimal predictor onto the cumulative boosting span generated by the audit class, extending the classical $L_2$-projection interpretation known for regression with squared loss. More importantly, this yields an explicit characterization of the function class on which the final predictor is multicalibrated, in contrast to prior analyzes that treated the resulting multicalibrated class as implicit or left it unspecific using an agnostic argument. In general, it is the closure of the cumulative boosting span, which could be substantially richer than the original audit class. We view this identification as a central contribution, because it clarifies what MCBoost can ultimately calibrate and how the choice of auditor shapes the learned predictor.
    \item[III] {\bf Convergence, finite-sample guarantees, and stopping.} 
     We derive explicit convergence rates for the excess risk under different smoothness regimes: sublinear under $c_L$-smoothness and geometric convergence under the Polyak-Łojasiewicz (PL) condition. We also establish finite-sample guarantees, including sample-complexity bounds and an upper bound on the number of iterations before termination.
    The stopping rule used in practical implementations is shown to guaranty multicalibration once the algorithm halts, providing a principled explanation for the validity of practical implementations. The analysis also covers optional projections under milder and more realistic conditions than assumed in prior work~\citep{deng2023happymap, zhang2024fair}, which are relevant for constrained predictors such as those taking values in $[0,1]$.
    \item[IV] {\bf Adaptation under covariate shift.} We extend the theory of universal adaptability~\citep{kim2022universal} under weaker assumptions and obtain more general guarantees. In particular, we provide refined bounds for target-domain error after multicalibration, characterize the transfer of multicalibration from source to nearby target distributions, and extend the framework to multi-source settings. These results clarify when multicalibration improves performance under distribution shift and highlight that its effectiveness depends on how well the relevant density ratios are represented within the learned learned function class. 
\end{itemize}

This work provides a comprehensive theoretical treatment of multicalibration, unifying existing variants under a single boosting-based framework and rigorously identifying their limit behavior, convergence properties, and transferability.

Our numerical study is designed to illustrate practical implications of these ideas. The motivating simulation in Figure~\ref{fig:compare-initial} shows that better global predictive accuracy does not necessarily imply smaller subgroup bias. The early-stopping experiments in Figures~\ref{fig:excess_vs_stepsize-tree} and~\ref{fig:excess_vs_ncalib-tree} provide qualitative evidence for the calibration-risk trade-off and the role of early stopping in controlling overfitting. The comparisons among various auditors give a practical sense of how auditor choice, group structure, and feature information affect MCBoost. The covariate-shift experiments likewise illustrate when performance can improve under structured shifts.

The paper is organized as follows. Section~\ref{sec:mcboost-framework} introduces the generalized MCBoost framework and presents a motivating simulation that illustrates why calibration and accuracy do not necessarily coincide. Section~\ref{sec:theory} develops the main theory, including the limit characterization, the explicit multicalibrated class, convergence and finite-sample guarantees, stopping, and target-independent learning under dsitribution shift. 
Section~\ref{sec:numerics} then studies auditor expressivity, group information, and structured shifts more in practice. Supplementary material contains proofs, additional examples, and additional numerical results.

\subsection{Related work}
Beyond the papers most directly aligned with ours, multicalibration has found use in several related areas. It has been used to promote fairness in ranking systems~\citep{dwork2019learning}, to give an indistinguishability-based interpretation of individual probabilities for non-repeatable events~\citep{dwork2021outcome}, and to extend calibration ideas to censored survival outcomes~\citep{ye2024multicalibration}. The concept is also closely tied to omniprediction~\citep{gopalan2021omnipredictors}, which shows that sufficiently multicalibrated predictors can be transformed to perform nearly optimally across a large collection of downstream losses.
Recent work has also begun to transport multicalibration ideas into modern applications. For example,~\cite{detommaso2404multicalibration} proposed multicalibration-based hallucination scores for large language models using groups induced by embedding clusters or LLM-generated annotations. Related work on conditional coverage in conformal prediction~\citep{romano2020malice,gibbs2025conformal} shares similar goals with multicalibration approaches for equalized quantile coverage~\citep{bastani2022practical,jung2022batch}, underscoring the broader connection between calibration, fairness, and reliable uncertainty quantification.
On the algorithmic side, our work is also connected to the classical boosting literature~\citep{freund1997decision}. Closely related are the function-space analysis of~\citet{buhlmann2003boosting} and~\citet{zhang2005boosting}, which highlights early stopping as a way to control overfitting.

\section{Multicalibration and the MCBoost framework}\label{sec:mcboost-framework} 

We study multicalibration as a general optimization-driven property of a predictor $f: \mathcal{X} \rightarrow \bbR$ relative to a score function $s(Y, u)$, which often arises as a gradient or subgradient of a loss $L(Y, u)$ with respect to its prediction argument $u$. Let $\sH$ be a class of measurable functions, possibly depending on both covariates and predictions.
\begin{defn}\label{def:mc}
    For tolerance $\alpha>0$, a predictor $f$ is $\alpha$-multicalibrated over $\sH$ if 
     \begin{equation}\label{eqn:mc}
           \sup_{h \in \mathcal{H}} 
\big| \, \mathbb{E}\!\left[ h(X, f(X)) \cdot s(Y, f(X))\right] \big| < \alpha.   
     \end{equation}   
\end{defn}
This more general notion can characterize functions obtained from various settings, such as:
(i) regression and classification with squared loss $L(Y, f) = (Y-f)^2/2$; (ii) classification with logistic loss $L(Y, f) = \log(1 + \exp(-fY))$ or exponential loss $L(Y, f) = \exp(-fY)$; and (iii) quantile regression with pinball loss. 
The function $h$ can encode the group structure (e.g., $h(X, f) =  \v1\{X \in G\}$), value buckets (e.g., $h(X, f) = \v1\{ f \in I\}$ for some value range $I$), or interactions $h(X, f) = \v1\{ X \in G, f \in I\} $. In all cases, multicalibration asks that $s(Y, u)$ have vanishing correlation, uniformly over $\sH$. 

A generic way to achieve multicalibration is the boosting scheme. We use boosting to reduce violations $\Expect[s(Y, f(X)) \cdot h ]$ over $h \in \sHa$ for some auditing class $\sHa$. We summarize a generic version in Algorithm~\ref{alg:MCboost-general}; it encompasses earlier variants for classification~\citep{kim2019multiaccuracy, kim2022universal}, as well as regression and quantile calibration~\citep{jung2022batch}. Basically, at each iteration, the algorithm (i) audits the current predictor $f^{(b)}$, (ii) selects the most violated direction $h^{(b)} \in \sHa$ via an auditing oracle $\mathcal{A}$, and (iii) updates $f^{(b)}$ in the opposite direction of violation. 
\begin{algorithm}
\hfuzz=100pt
\caption{Generalized MCBoost 
}
\label{alg:MCboost-general}
\KwData{Initial $f^{(0)}$; accuracy $\alpha$; steps $\{\eta^{(b)}\}$; optional projection $\mathcal P_{\mathcal O}$; 
auditing procedure $\mathcal A$ with associated auditing family $\sHa$; maximal iterations $B_0$; \\
Calibration set $\Xi = \{ ( X_i, Y_i ) \}_{i=1}^{N_1}$. \\
Validation set $V = \{ ( X_i, Y_i ) \}_{i=N_1+1}^{N_1+N_2}$. }
\For{$b=0,1,2,\dots, B_0$}{
    \textbf{Generate candidate directions (auditing):} 
    \begin{enumerate}[leftmargin=2em,label*=\arabic*.]
    \item[\textbf{(a)}] \emph{Dynamic (default):} 
    $\mathcal{A}$ returns one or more $h \in \sHa$ fit against the residual/score given the current predictor $f^{(b)}$:
    $$
        \sHc = \{h(\cdot, f^{(b)}(\cdot)) \} \leftarrow \mathcal{A}(\Xi; f^{(b)}) 
    $$
    \item[\textbf{(b)}] \emph{Static (optional, using initial predictor $f^{(0)}$):} 
    $$
    \sHc  = \{ h(\cdot,f^{(0)}(\cdot)) \} \leftarrow \mathcal{A}(\Xi; f^{(b)}, f^{(0)}),  
    $$
    \end{enumerate}
    which allows bucketization based on $f^{(0)}$. 
    
    (This step may return multiple candidates; the subsequent step selects the worst violation.)
    
    \textbf{Select worst-violated direction:}
    Select $h^{(b)}$ that maximizes correlation
    $$
    h^{(b)} \leftarrow \argmax_{h \in \sHc } \frac{1}{|V|} \big|\sum_{i\in V} h(X_i) \ s(Y_i,f^{(b)}(X_i))\big|.
    $$ 
    Here $h^{(b)}(X)$ is the realized measurable function $X \mapsto h^{(b)}(X, f^{(b)}(X))$. 
    \\
    Compute the normalized empirical violation:
    \[
    \Delta =\frac{\frac{1}{|V|} \big|\sum_{i\in V} h^{(b)} (X_i) \ s(Y_i,f^{(b)}(X_i))\big|}
                      {\sqrt{\frac{1}{|V|}\sum_{i\in V} (h^{(b)  }(X_i) )^2 }}.
    \]
    \If{
    $\Delta \le \alpha$
    }{\Return $\tilde f=f^{(b)}$ 
    }
    Update 
    $\displaystyle f^{(b+1)}=\mathcal P_{\mathcal O}\big(f^{(b)}-\eta^{(b)}\,h^{(b)}\big)$.
}
\end{algorithm} 
In Algorithm~\ref{alg:MCboost-general}, the auditing procedure $\mathcal{A}$ determines the admissible form of $h$. For instance, (i) a linear or Ridge regression could lead to $h = X^\top \beta \v1\{f^{(b)}(X) \in I\}$, representing a local linear correction within a funciton value range; (ii) Groupwise constant fit gives $h = \beta_G \v1\{X \in G\} $, corresponding to subgroup refinements; (iii) tree-based or kernel regressors that fit $s(Y, f^{(b)}(X_i))$ over partitioned regions of $(X, f^{(b)}(X))$ space. The coefficients (e.g., $\beta, \beta_G$) are learned by $\mathcal{A}$ using data $\Xi$. The optional projection operator $\mathcal{P}_O$ enforces a function within some closed space $O$, e.g. $O = \mathcal{Y} = [0, 1]$.

\subsection{A motivating simulation: accuracy is not enough}\label{subsec:motivation}
To motivate why MCBoost can still be useful even when a predictor already achieves good overall accuracy, we work with the numerical design used throughout the paper. We generate responses according to
$$
Y = f_1(X^{(c)}) + f_2(X^{(d)}) + \varepsilon, \quad \varepsilon \sim \mathcal{N}(0, \sigma^2(X^{(c)}, X^{(d)})),
$$
where the covariates are partitioned into continuous and categorical components. Specifically,
$X^{(c)} = (X_1, \dots, X_5) \sim \mathcal{N}(0, I_5)$ and $X^{(d)} = (X_6, X_7) \in \{0, 1\}^2$. Here, for illustrative purposes, let $X_6$ represent gender (male=1, female=0), and $X_7$ represent ethnicity (black=1, white=0). The heteroskedastic noise level is
$
\sigma(X^{(c)}, X^{(d)}) = \sigma \,(1 + 0.2X_1 + 0.3X_6 + 0.25X_7), \sigma = 0.5, 
$
and the regression functions are
$$
f_1(X^{(c)}) = X^{(c)\top} \beta + 0.8\sin(X_1) - 0.5(X_2^2 - 1) + 0.4X_1 X_2, \quad \beta = (1.2, 0.9, 0.6, 0.4, 0.2),
$$
and
$$
f_2(X^{(d)}) = -0.7X_6 + 0.8 X_7 + 0.4X_6 X_7.
$$
This design separates global predictive accuracy from subgroup bias. Standard criteria such as MSE or MAE quantify pointwise accuracy, whereas multicalibration asks whether residual bias remains after conditioning on subgroups and prediction levels. Consequently, a method looks strong under global loss may still exhibits substantial subgroup-level discrepancies. 

That is exactly what Figure~\ref{fig:compare-initial} shows. Flexible models such as random forests often achieve lower MSE, yet they need not exhibit smaller subgroup bias than a simpler linear model. By contrast, a linear model can satisfy global moment conditions such as
$$
X^\top (\vy - \wh{\vf}) = 0,
$$
which enforce calibration along the observed linear features, including group indicators when those features are included. But real groups are often nonlinear, intersected, or prediction-dependent. MCBoost is therefore best viewed as a post-processing layer that retains a good initial fit while explicitly targeting the remaining structured biases.
\begin{figure}[ht!]
    \centering
    \begin{subfigure}[b]{0.8\linewidth}
        \includegraphics[width=\linewidth]{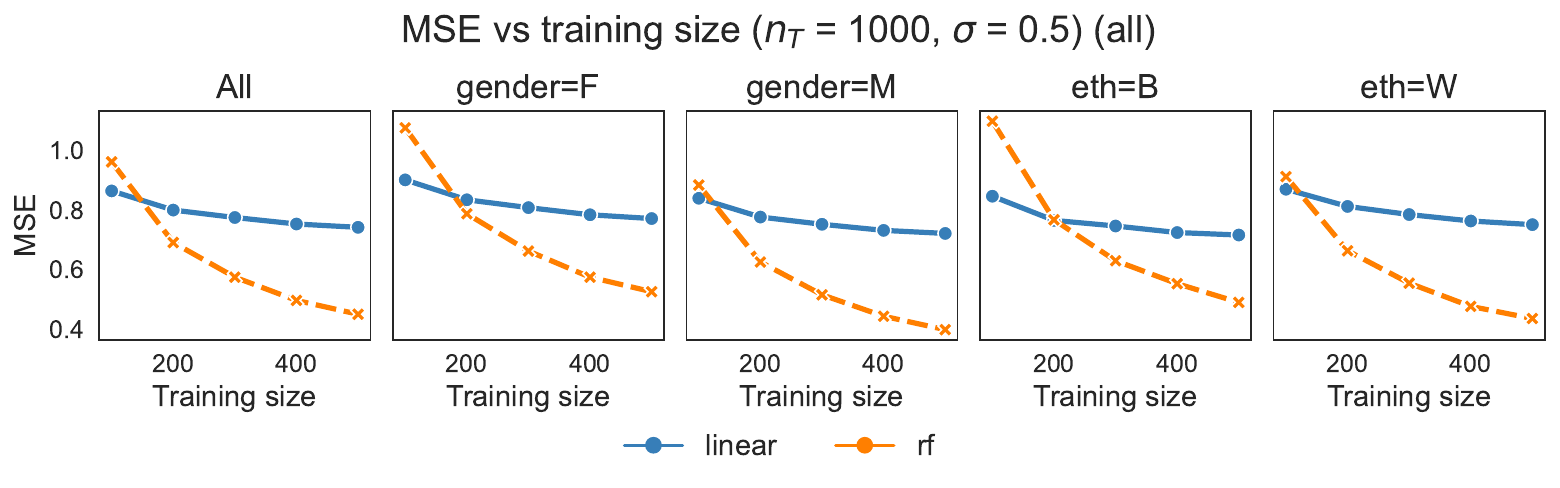}
    \end{subfigure}
    \begin{subfigure}[b]{0.8\linewidth}
        \includegraphics[width=\linewidth]{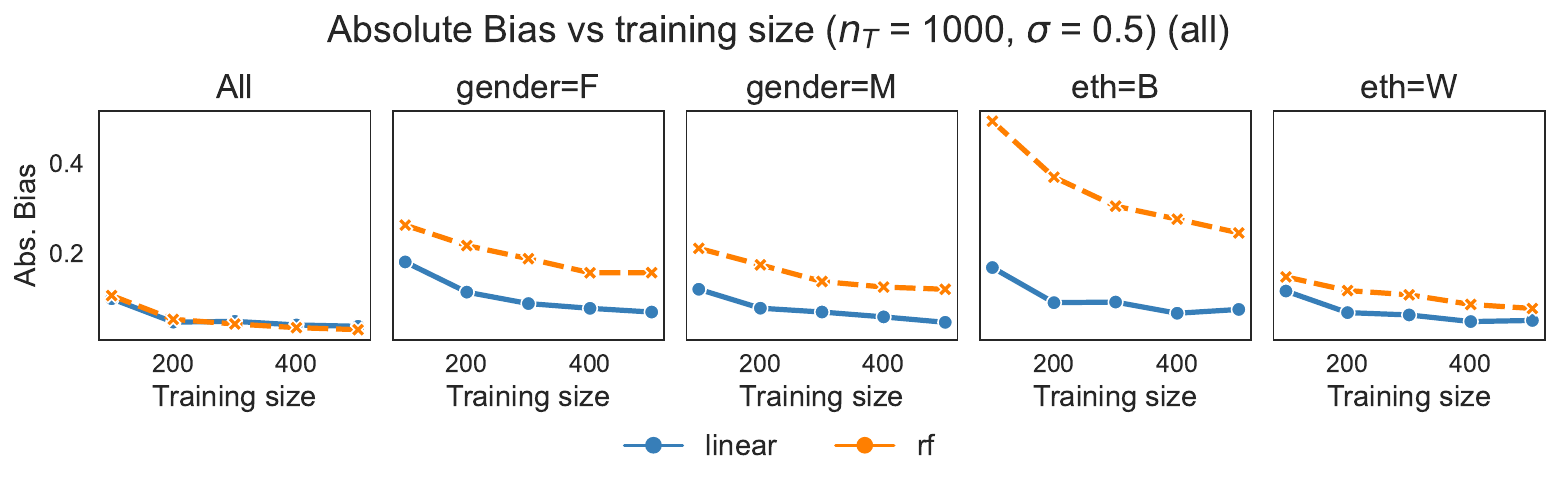}
    \end{subfigure}
    \begin{subfigure}[b]{0.8\linewidth}
        \includegraphics[width=\linewidth]{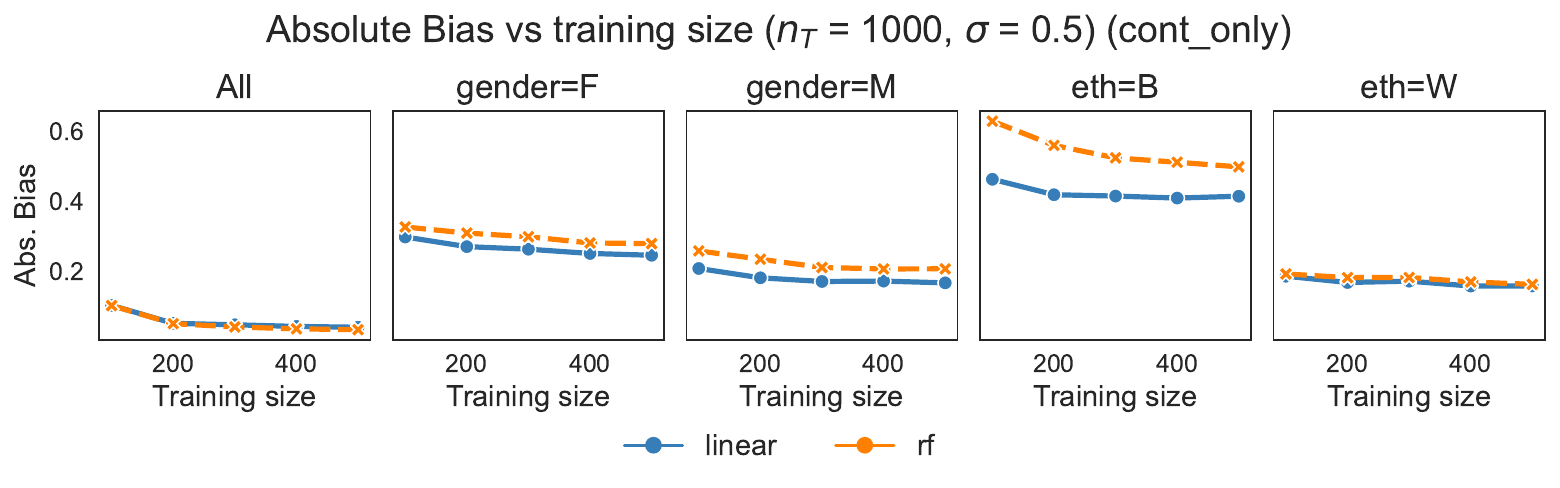}
    \end{subfigure}
    \caption{Test MSE and absolute groupwise bias as functions of the training sample size for two initial predictors, a random forest and a linear model. The random forest is often more accurate globally, but it can still exhibit larger subgroup bias. Results are averaged over $50$ replications.
    The first and second rows use all covariates, including categorical group indicators $(X_6,X_7)$, whereas the third row excludes them from model fitting, leading to larger groupwise biases for the linear model. 
}  \label{fig:compare-initial}
\end{figure} 

\subsection{Specialization and prior algorithms as instances}
To get better understanding of the abstract or formalized languaged Algorithm~\ref{alg:MCboost-general}, we highlight that it includes the procedure of earlier work by~\cite{kim2019multiaccuracy,kim2022universal}, which is modified slightly and is presented in Algorithm~\ref{alg:MCboost}. Their audited learner depends solely on covariates $h = h(X)$, but the algorithm effectively incorporates prediction buckets by conditioning on value ranges of $f^{(b)}(X)$. Each bucket $\{X \in S_l\}$ corresponding to multiplying $h(X)$ by an indicator $\v1\{X \in S_l\}$, thus implicitly creating functions of the form $h(X) \v1\{f^{(b)}(X) \in I_l\}$. More recent formulations~\citep{deng2023happymap,zhang2024fair} makes this structure explicit by allowing the auditor to depend jointly on covariates and predictions, making $h = h(X, f(X))$, as in~\eqref{eqn:mc}. This broader perspective requires the score function $s(Y, f(X))$ to be unbiased not only within subpopulations defined by $X$, but also conditional on the predicted value $f(X)$ itself, thereby extending the classical calibration.
\begin{algorithm}
	\hfuzz=100pt
	\caption{ 
    Modified bucketized MCBoost 
    }
	\label{alg:MCboost} 
	\KwData{Initial estimator: $f^{(0)}(X)$ \\
		Accuracy parameter $\alpha > 0$; step size $\eta$; maximal iterations $B_0$; \\
		Auditing algorithm $\mathcal{A}$. \\
		Calibration set $\Xi = \{ ( X_i, Y_i ) \}_{i=1}^{N_1}$. Validation set $V = \{ ( X_i, Y_i ) \}_{i=N_1+1}^{N_1+N_2}$.
	}
	\DontPrintSemicolon
	\For{$b = 0, \ldots, B_0$}{
        \textbf{Auditing:}
		Compute buckets $\{S_l\}_{l=1}^L$, where 
		$$
		S_l := \left\{ X \in \mathcal{X}: f^{(b)}(X) \in \Big[ C_1 + \tfrac{ (l-1)(C_2-C_1) }{L}, C_1 + \tfrac{ l (C_2 - C_1)}{L} \Big] \right\}, 
		$$
        or, optionally, by the initial predictor $f^{(0)}$:
        $$
        S_l = \left\{X \in \mathcal{X} : f^{(0)}(X) \in \Big[ C_1 + \tfrac{ (l-1)(C_2-C_1) }{L}, C_1 + \tfrac{ l (C_2 - C_1)}{L} \Big] \right\}. 
        $$ \;
        For each bucket $S_l$, fit 
		$ h_{b, S_l}(X) \leftarrow \mathcal{A}\!\left(\Xi, \{ (f^{(b)}(X_i), Y_i) : X_i \in S_\ell \} \right), $ with $h_{b,S_l}\in\sHa$.  \;
        \textbf{Selecting: }
        Let \begin{equation}\label{eqn:violation}
		    S^{*}  \leftarrow \argmax_{S_l} \frac{1}{|V|} \left| \sum_{i \in V} h_{b, S_l}(X_i)  \cdot s(Y_i, f^{(b)}(X_i)) \right|. 
		\end{equation}
        Optionally, define 
        \[
        S^{*}  \leftarrow \argmax_{S_l^{\le}, S_l^{\ge} } \frac{1}{|V|} \left| \sum_{i \in V} h_{b, S_l}(X_i) \cdot s(Y_i, f^{(b)}(X_i)) \right|,
        \]
        where groups are defined by quantile buckets such as $S_l^{\le} = \{ f^{(b)}(X) \le C_1 + \tfrac{ l (C_2 - C_1)}{L}\}$ or $S_l^{\ge} = \{ f^{(b)}(X) \ge C_1 + \tfrac{l (C_2 - C_1)}{L}\}$. 
	       \; 
        Set 
         \begin{enumerate}[leftmargin=2em,label*=\arabic*.]
        \item[\textbf{(a)}] \emph{Local update:}  $h^{(b)}(X) = h_{b, S^* }(X) \v1\{ X \in S_l\}$,   
        \item[\textbf{(b)}] \emph{Global update (optional):} $h^{(b)}(X) = h_{b, S^* }(X) $, 
        \end{enumerate}
         and consider   
        \begin{equation}\label{eqn:Delta}
            \Delta \leftarrow  \frac{1}{|V|} \Big| \sum_{i \in V} h^{(b)}(X_i)  s(Y_i, f^{(b)}(X_i) ) \Big|
            /
            \sqrt{ \frac{1}{|V|} \sum_{i \in V} \{ h^{(b)}(X_i) \}^2 }.
        \end{equation} 
		\If{$\Delta > \alpha$}{ 
			$f^{(b+1)}(X) \leftarrow  \mathcal{P}_{O} \left( f^{(b)}(X) - \eta^{(b)} \cdot h^{(b)}(X) \right).$ 
		}
		\Else{
			\Return{$\widetilde{f}(X) = f^{(b)}(X)$. }
		}
	}
\end{algorithm}
A convenient and widely used construction takes the tensor-product form 
\begin{equation}\label{eqn:product-form}
    \sHa = \{h(X ,f) = h_1(X) h_2 (f(X)): h_1 \in \sH_X, h_2 \in \sH_f \} = \sH_X \otimes \sH_f, 
\end{equation}
Here $\sH_X$ captures the partition of the covariate space, e.g., $\sH_X = \{\beta_G \v1\{X \in G\}: G \in \mathcal{G}\}$ or 
$ \sH_X = \{h(X) \v1\{X \in G\}: G \in \mathcal{G} \} $ for tree-based $h(X)$. The class $\sH_f$ stratifies the sample according to the predicted values, typically
$$
\sH_f = \{\v1\{f(X) = v\}; v \in \mathcal{V} \} = \v1\{X \in f^{-1}(v): v \in \mathcal{V} \}.  
$$
This form~\eqref{eqn:product-form} also appears in~\cite{gopalan2022low} in achieving low-degree multicalibration.
If $f(X)$ is continuous in practice, one replaces exact values by value intervals (buckets):
$$
\v1\{f(X) \in I_l \} = \v1\{X \in f^{-1}(I_l)\} = \v1\{X \in S_l\}, \quad l = 1, \ldots, L.
$$
This explains the bucketized boosting updates in Algorithm~\ref{alg:MCboost},  where the auditor inspects conditional calibration across both feature groups and value ranges of $f$.  

The earlier MCBoost version~\citep{kim2019multiaccuracy, kim2022universal} also offers an option that is based on the initial predictor $f^{(0)}$, corresponding to auditors of the form 
$$
h_1(X) \v1\{f^{(0)}(X) = v\} = h_1(X) \v1\{X \in f^{(0)-1}(v)\}.
$$
In quantile regression, as we shall see in the next subsection, $\sHa = \{ \sum_{G \in \mathcal{G}} \beta_G \v1\{ X \in G\}: G \in \mathcal{G} \}$ in BatchGCP algorithm of~\cite{jung2022batch}, or $\sHa = \{\beta_{G, v} \v1\{X \in G\} \v1\{ q_{\tau}(X) = v\} \}$ in BatchMVP~\citep{jung2022batch}. These examples illustrate the versatility of multicalibration, which covers the subgroup fariness and joint calibration over covariate-prediction slices.

\subsection{Quantile regression as an example}\label{subsubsec:quantile}
For level $\tau \in (0, 1)$, the pinball loss can be written as 
\begin{equation}\label{eqn:qr}
    L_{\tau}(Y, q) = (Y - q )(\tau - \v1\{Y \le q\}), \quad \partial_q L_{\tau} (Y, q) = \{ \v1\{Y \le q\} - \tau,  \v1\{Y < q\} - \tau \}
\end{equation}
We fix the representative subgradient $s(Y, q) = \v1\{Y \le q\} - \tau $, then 
\begin{equation}
\sup_{h \in \mathcal{H}} 
\big| \, \mathbb{E}\!\left[ h(X, q_{\tau}(X)) \cdot s(Y, q_{\tau}(X)) \right] \big| = 
\sup_{h \in \mathcal{H}} 
\big| \, \mathbb{E}\!\left[ h(X, q_{\tau}(X)) \cdot (\mathbf{1}\{Y \le q_{\tau}(X)\} - \tau ) \right] \big| < \alpha,    
\end{equation}
meaning the frequency of $Y \le q_{\tau}(X)$ is $\tau$ up to $\alpha$-error for every slice $h$. In other words, calibration at level $\tau$ holds not just marginally but uniformly over the audited subpopulations included in $\sH$.
\begin{example}[BatchGCP (Algorithm 1 of~\cite{jung2022batch})]\label{ex:batchGCP}
With a single global bucket and choose 
$$
\sHa = \left\{ \sum_{G} \beta_{G} \mathbf{1}\{X \in G \} \mid G \in \mathcal{G} \right\},
$$
a single update 
$$
h^{(0)} = \arg\min_{h \in \sHa } \; \mathbb{E}\!\left[ L_{\tau}(Y; q_{\tau}^{(0)}(X) - h ) \right],  
$$
produces $ q_{\tau}^{(0)}(X) - h^{(0)} $ that corrects groupwise coverage. This can be viewed as a one-step boosting refinement. Suppose the pre-trained function $q_{\tau}^{(0)}(\cdot) \in \mathcal{H}_0 $, the resulting function belongs to $\mathcal{H}_0 + \sHa$. See Section~\ref{apxsec:instances} for details.
\end{example}

\begin{example}[Multicalibration over $(G, v)$ bucket (MultiMVP)]\label{ex:multiMVP}
Discretize $[0, 1]$ and let 
$$
\sHa = \{ \beta_{G, v}\v1\{X \in G, q_{\tau}(X) = v\} \mid G \in \mathcal{G}, v \in [1/L] \}, 
$$
where $[1/L] = \{1/L, 2/L, \ldots, 1\}$. Auditing selects the worst $(G^*, v^*)$ bucket based on the gap $| \bbP(Y \le v | X \in G) - \tau|$ weighted by its mass, and updates that bucket by a constant shift. This is equivalent to MultiMVP  (Algorithm 2 in \citet{jung2022batch}; cf. \citet{roth2022uncertain}). Readers can find the details in Section~\ref{apxsec:instances} of Appendix. 
\end{example}  

\subsection{From audit class to calibrated class}
We are interested in characterizing the function class $\sH$ appearing in~\eqref{eqn:mc}. It is essential to distinguish between \emph{audit class} $\sHa$, the collection explored by the auditor during training, and the \emph{multicalibrated class} $\sH$. Earlier work often assumed $\sH = \sHa$, so eliminating empirical violations over $\sHa$ implied $\alpha$-multicalibration~\citep{deng2023happymap}. Other treatment sidestepped the issue by postulating an unspecified class $\sH$ that is “agnostically learned”~\citep{kim2019multiaccuracy} at finite samples $\#V=\Theta(N)$, i.e., 
$$
\sup_{h \in \sH} \left| \frac{1}{N} \sum_{i \in V} h (X_i) s(Y_i, \wt{f}(X_i)) \right| < \alpha, 
$$ 
leading to $\sup_{h \in \sH } | \Expect h (X) s(Y, \wt{f}(X)) | = O(\alpha)$ via concentration.
In our formulation, the additional normalization introduced in~\eqref{eqn:Delta} unlike~\cite{kim2019multiaccuracy, kim2022universal},  facilitates a shaper theoretical analysis. 
Their assumption is not entirely without basis, since stopping criterion implicitly selects a class $\sH$ for which the empirical multicalibration condition holds. At minimum, following the unnormalized stopping criterion, any $h$ belonging to the convex hull of final audit functions $\{h_{B, S_l}\}_{l=1}^L$ automatically satisfies~\eqref{eqn:mc}. Thus, one can at least conclude $\sH \supset \textrm{Conv}(\{h_{B, S_l}\}_{l = 1, \ldots, L} )$.

\begin{rem}
    {\rm 
    We observe that~\cite{deng2023happymap,zhang2024fair} implicitly assume access to a black-box optimization oracle together with exact knowledge of the underlying distribution, as also discussed by \cite{gibbs2025conformal}. Their procedure requires verifying whether all $h \in \sH$ are multicalibrated, which may only be feasible when $\sH$ is finite (for example, indicators of a fixed collection of sensitive groups). For general infinite classes, such an exhaustive check is generally not computationally tractable. Moreover, the updates presuppose a direct evaluation of expectations $\Expect h \cdot s$, which in practice would require knowing the distribution exactly. The algorithm looks more like sketches: the choice of parameters and sample sizes is not fully specified, and the auditing step is left abstract. If $h$ is unknown, it must be approximated through some auxiliary procedure $\mathcal{A}$, while in the simplest case where $h$ step functions on sensitive groups, the update reduces to constant adjustment within each group. 
    }
\end{rem} 

We make this relationship explicit. Formally, define the sequences of additive expansions 
$$
\mathcal F_k = \left\{\sum_{j=1}^k \eta_j h_j : h_j\in\sHa\right\}, \quad
\mathcal F=\overline{\bigcup_{k\ge1}\mathcal F_k}, 
$$
where the closure is taken in $L_2(\bbP_X)$.
We prove (Theorem~\ref{thm:stationarity}) that the limit predictor $f^{(\infty)}$ satisfies
$$
\langle g^{(\infty)},h\rangle=0, \quad \forall h\in\sF, \qquad g^{(\infty)} = \Expect[s(Y, f^{(\infty)}) | X].
$$

Despite that $|\langle g^{(\infty)}, h\rangle| = 0, \forall h \in \sF$, a practical algorithm halts after finitely many iterations $B$ on finite data. The resulting predictor $\wt{f} = f^{(B)}$ is multicalibrated over some $\sH$ with tolerance $\alpha$, but we cannot claim $\sH$ to be the entire $\sF$; the distinction is subtle but important. While the ideal limit $f^{(\infty)}$ is uncorrelated with all directions in the closure of the boosting span $\sF$, the finite-step predictor merely achieves small, but nonzero, correlation $| \Expect h \cdot s(Y, f^{(B)}) | $. 

Observe that scaling $h$ trivially rescales the correlation: 
\begin{align*}
    &| \Expect a \cdot h \cdot s(Y, f^{(B)}(X) ) | =   |a| | \Expect [ h \cdot s(Y, f^{(B)}(X) )  ] |.
\end{align*} 
Hence, without constraining the the magnitude of $h$, the notion of ``multicalibrated over $\sF$" becomes ill-posed: arbitrary large $L_2$-norms would yield unbounded violations. It is therefore meaningful to assess multicalibration within norm-bounded subset of $\sF$. Concretely, if the algorithm halts at step $B$ with empirical violation below $\alpha$, then 
$$
| \Expect h s(Y, f^{(B)}(X)) | = O(\|h\|_{L_2} \alpha ),  \forall h \in \sF, 
$$ 
and consequently, 
$$
\sup_{f \in \sF, \|h\|_{L_2} \le C_{\sF} }| \Expect h s(Y, f^{(B)}(X)) | = O( C_{\sF} \alpha ). 
$$
This establishes that $f^{(B)}$ is multicalibrated over the norm-bounded subspace of $\sF$, with violation scaling linearly in the $L_2$-radius $C_{\sF}$. Details and formal results appear in the following sections. 

\begin{example}\label{ex:class}
    Each audit direction $h_j$ may depend on $X$, on the current predictor $f^{(j)}$, or occasionally on the initial predictor. Typical examples include:
    \begin{enumerate}
        \item Groupwise refinements: 
        $$
        \sF_k = \left\{ \sum_{j=1}^k \eta_j \beta_{G_j, j} \v1\{ X \in G_j: G_j \in \mathcal{G} \} \right\} = \left\{ \sum_{G \in \mathcal{G}} (\sum_{j: G_j = G} \eta_j \beta_{G_j, j})\v1\{X \in G \}  \right\}.  
        $$
        \item Additive tree ensembles: $ \sF_k = \{ \sum_{j=1}^k \eta_j h_{1, j}(X) \}$, with tree-based $h_{1, j}$. 
        \item Hybrid refinements combining trees and groups: 
        $$
        \sF_k = \{ \sum_{j=1}^k \eta_j h_{1,j}(X) \v1\{X \in G_j: G_j \in \mathcal{G} \}\}.
        $$ 
        \item Fixed-bucket refinements based on $f^{(0)}$:
        $$
        \sF_k = \{ \sum_{j=1}^{k} \eta_j h_{1,j}(X) \v1\{ f^{(0)}(X) \in I_{l_j}; l_j = 1, \ldots, L\} \}.
        $$ 
        \item Dynamic bucket refinements using current predictions:
        \begin{align*}
            \sF_k & = \left\{ \sum_{j=1}^k \eta_j h_{1,j}(X) \v1\{f^{(j)}(X) \in I_l \} \right\} =  \left\{ \sum_{j=1}^k \eta_j h_{1,j}(X) \v1\{ f^{(0)} + \sum_{u<j}\eta_uh_u)(X)\in I_l \} \right\}.
        \end{align*}
    \end{enumerate}
\end{example}
These constructions illustrate that the expressive capacity of MCBoost depends crucially on the auditing algorithms and classes $\sHa$ that serve for different purposes. For example, group indicators mitigate group-wise disparities, bucketization enforces value-level calibration~\citep{roth2022uncertain}.

The preceding discussion clarifies that the true multicalibrated space need not coincide with the original class. Formally, we have $\sHa \subseteq \sF = \overline{\cup_{k \ge 1} \sF_k }$. The set $\sF$ thus represents the closure of all directions reachable by iterative boosting. In special cases, such as space spanned by finite-dimensional bases or partition indicators, do we have equality $\sF = \overline{\sHa} = \sHa$. In general, $\sF$ can be strictly larger: for instance, additive expansions of regression trees or weak classifiers yield rich, infinite-dimensional spaces. Boosting implicitly expands the calibration domain from $\sHa$ to $\sF$, providing a rigorous characterization of the learned multicalibrated class and formalizing what earlier work treated heuristically. 

\section{Theory of MCBoost and its empirical implications}\label{sec:theory} 
We now turn to the formal theoretical analysis of the multicalibration boosting procedure. Our goal is to articulate the statistical and algorithmic principles underlying MCBoost and to clarify several aspects that have remained partially understood in prior work. Specifically, we investigate: i) the limiting behavior of the boosting trajectory: to what function does the predictor sequence converge and at what rate?  ii) for which function family the predictor achieves multicalibration; and iii) the practical validity of finite-step stopping---does the halting rule guarantee multicalibration, and how does the iteration number relate to sample size and stopping threshold? To address these questions, we introduce some conditions and notation.
\begin{ass}\label{ass:l2}
    For every realized predictor $f$ along the boosting path, the map $X \mapsto h(X, f(X))$ is measurable and the corresponding auditing family satisfies
    $$
    \sup_{h \in \sHa} \|h\|_{L_2} \le C_{\sHa} < \infty, 
    $$
    so every admissible audit direction lies in $L_2(\bbP_X)$.  
    
\end{ass}
This holds for all indicator-bucket and bounded tree-based auditors. 

\paragraph{\bf Notation.}
Let 
$$
\mathcal L(f)\ :=\ \mathbb E\big[ L\big(Y,f(X)\big)\big],
$$
denote the population risk. For any $g,h\in L_2(\bbP_X)$, write 
$$
\langle g,h\rangle :=\mathbb E[g(X)h(X)], \quad \|h\|_{L_2} :=\sqrt{\langle h,h\rangle}
$$
If $s$ involves both $X$ and $Y$, we also use the shorthand 
$$
\langle s,h\rangle = \mathbb E[s(Y, f(X))h(X)]. 
$$
The conditional mean score at iteration $b$ is $g^{(b)}(X):=\mathbb E\!\left[s\!\left(Y,f^{(b)}(X)\right)\mid X\right]$, so that 
$ \langle s^{(b)},h\rangle=\langle g^{(b)},h\rangle, $ for all $h(X)$.

The boosting dynamics depend on an auditing class $\sHa \subset L_2(\bbP_X)$. Let 
\begin{equation}\label{eqn:chain}
    \mathcal F_k \ := \left\{\sum_{j=1}^k \eta_j h_j: h_j\in \sHa \right\},
\qquad
\mathcal F\ :=\ \overline{\bigcup_{k \ge 1} \mathcal F_k},
\end{equation}
For instance, if $\mathcal H$ is the class of fixed-depth regression trees, then $\mathcal F$ is the closure of all finite additive expansions of such trees, a rich additive function space. We then define the restricted dual norm
\[
\|g\|_* \ :=\ \sup\{\,\langle g,h\rangle:\ h\in\mathcal F,\ \|h\|_{L_2}\le 1\,\}.
\]

\subsection{The limit of boosting procedure}

We begin with a smoothness condition on the loss. 
\begin{defn}[Lipschitz derivative]\label{def:Lipschitz-deriv}
We say $L$ is $c_L$-smooth in its first argument if its score is $2c_L$-Lipschitz:
$|s(Y, u)-s(Y, u)|\le 2c_L|u-v|$ a.s. in $Y$.
\end{defn}  
This condition yields a standard quadratic lower bound, 
\begin{lemma}\label{lem:Lipshitz}
    Suppose $L$ has $c_L$-smooth,
    the objective function $L$ we consider satisfy 
    $$
    L(Y, f(X)) - L(Y, \wt{f}(X)) \ge \partial_{u} L(Y, u)|_{u = f} (f(X) - \wt{f}(X)) - c_{L} (f(X) - \wt{f}(X))^2,
    $$
    for some $c_L > 0$, where $\partial_{u} L(Y, u)$ is partial derivative/subgradient w.r.t the second argument $u$. 
\end{lemma} 
The $c_L$-smooth condition holds for $L = (Y - f(X))^2/2$ and the pinball loss. 
Next we impose a weak-learning condition. 
\begin{ass}[Weak learner edge]\label{ass:weak-edge}
At iteration $b$, the selected wost-violated direction $h^{(b)} \in \sHa$ satisfies
$$
\frac{\langle g^{(b)},h^{(b)}\rangle}{\|h^{(b)}\|_{L_2}}
\ \ge\
\kappa\,\sup_{h\in\mathcal F}\frac{\langle g^{(b)},h\rangle}{\|h\|_{L_2}}
\ =\ \kappa\,\|g^{(b)}\|_*, \quad g^{(b)}(X):=\mathbb E[s(Y,f^{(b)}(X))\mid X],
$$
for some $\kappa\in(0,1]$, where $\|g\|_*:=\sup_{\|h\|_2\le1,\,h\in\mathcal F}\langle g,h\rangle$.
\end{ass} 
This condition requires that the weak learner consistently identifies a direction that is non-trivially aligned with the gradient (or residual) $g^{(b)}$. In other words, although the optimal update would be the exact maximizer
$
h^* = \argmax_{\|h\|_{L_2}=1, h \in \sF }\langle g^{(b)}, h\rangle,
$
we only ask that the learner returns a hypothesis whose correlation with $g^{(b)}$ is a fixed fraction $\kappa$ of this optimum. This ``edge" assumption guarantees sufficient progress for refinement and convergence, even if the maximizer cannot be computed exactly in practice.
Similar weak-learning hypotheses can been seen in earlier work such as~\cite{zhang2005boosting}.~\cite{alon2021boosting} and~\cite{freund1997decision} assume that each weak learner performs slightly better than chance, and the cumulative effect of such small improvements yields a strong final predictor. Our condition plays the same role here: each audit step only needs to make nontrivial progress in the direction of the current score. 

To accommodate general losses beyond the squared loss, we measure the discrepancy via the generalized Bregman divergence  
$$
D_{\sL} (f_1, f_2; \zeta) = \sL(f_1) - \sL(f_2) - \langle \zeta, f_1 - f_2 \rangle, \quad \zeta \in \partial \sL(f_2), 
$$ 
where $\partial \sL(f_2)$ is the subderivatives at $f_2$. Alongside, we define the associated (possibly set-valued) Bregman projection
$$
\mathcal{P}_{f^{(0)} + \sF}^{\sL} (f_1; \zeta) := \argmin_{f \in f^{(0)} + \sF} D_{\sL}(f, f_1; \zeta), 
$$ 
that project the function $f_1$ onto $f^{(0)} + \sF$. 
\begin{rem}
    {\rm 
    A typical Bregman divergence 
    $$
        D_{\sL} (f_1, f_2) = \sL(f_1) - \sL(f_2) - \langle \nabla \sL(f_2), f_1 - f_2 \rangle, 
    $$
    requires the distance generating function ($\sL$ in our context) to be smooth and strictly (or strongly convex), Under which $D(f_1 , f_2) = 0$ implies $f_1 \equiv f_2 $. For non-smooth convex losses (e.g., pinball loss), we work with subgradients. 
    To speak of projection,
    it often requires  $\mathcal{P}_{f^{(0)} + \sF}^{\sL} (f_1; \zeta)$ to be unique; we extend this notion to allow it to be a set that accommodates multiple projected functions. Check Section~\ref{sec:general-Bregman} in the Appendix for more details.  
    }
\end{rem}

\begin{theorem}[Stationary and projection]\label{thm:stationarity}
Suppose $L$ is $c_L$-smooth and Conditions~\ref{ass:l2}, ~\ref{ass:weak-edge} hold. 
Let the step size at iteration $b$ be 
$\eta^{(b)} = \tfrac{\langle h^{(b)}, s(Y, f^{(b)}) \rangle }{ 2c_L \|h^{(b)}\|_{L_2}^2}$ such that $\{ \| \sum_{b=1}^{k} \eta^{(b)} h^{(b)} \|_{L_2} \}_{k \ge 1} $ forms a Cauchy sequence in $L_2(\bbP_X)$. Then the boosting iterates $\{f^{(b)}\}_{b \ge 0}$ lie in $f^{(0)}+\mathcal F$ and satisfy
$$
\sL(f^{(0)}) \;\ge\; \mathcal L(f^{(1)}) \;\ge\; \cdots \;\downarrow\; \mathcal L(f^{(\infty)}). 
$$
The limit $f^{(\infty)}$ is stationary in the affine span $f^{(0)}+\mathcal F$: 
$$
\ \langle g^{(\infty)},h\rangle\ =\ 0\quad\text{for all }h\in\mathcal F,\qquad g^{(\infty)}(X):=\mathbb E[s(Y,f^{(\infty)}(X))\mid X].\ 
$$
If the global minimizer $f^* \in \arg\min_{f}\mathcal L(f)$ exists, then 
$$
f^{(\infty)} \in \mathcal{P}_{f^{(0)}+\mathcal F}^{\sL} (f^*; 0) = \argmin_{f \in f^{(0)} + \sF} D_{\sL}(f, f^*; 0),  
$$
i.e., the Bregman projection of $f^*$ onto the affine span $f^{(0)} + \sF$. Moreover,
$$
\argmin_{f \in f^{(0)} + \sF} D_{\sL}(f, f^*; 0) = \argmin_{f \in f^{(0)} + \sF} \sL(f)  
$$ 
is invariant of the choice of $f^*$. 
\end{theorem}
Theorem~\ref{thm:stationarity} uncovers that $f^{(\infty)}$ can be viewed as projection onto the set $f^{(0)} + \sF$. The initial predictor $f^{(0)}$ need not lie in $\sF$; it may arise from a different training procedure
(e.g., linear regression or neural network) before refinement through tree or group-based boosting, and it may be well-trained or just inadvertently chosen. Nonetheless, the resulting limit is with respect to $\sF$, 
$$
\sup_{h \in \sF} | \Expect  s(Y, f^{(\infty)}) h(X)  | =\sup_{h \in \sF} | \langle g^{(\infty)}, h\rangle |  = 0.
$$
Thus, the learned multicalibrated class for $f^{(\infty)}$ is exactly  the closed convex symmetric $\sF$.

It would be instructive to investigate the gain from boosting under various form of $\sF$ and examine how the structure of the auditing family $\sHa$ influences the expansion of $\sF$.

\begin{enumerate}
    \item \textbf{Finite basis.} If $\sHa$ is a finite-dimensional linear subspace, e.g., $\sHa=\mathrm{span}\{\phi_1,\ldots,\phi_K\}$ or a fixed set of group indicators $\{ \sum_G \beta_G \v1\{X\in G\}:G\in\mathcal G\}$, then $\sF=\overline{\sHa}=\sHa$.   
To expand calibration beyond this,
the audit class itself has to be enlarged. 
\item \textbf{Rich dictionary.}  
If $\sHa$ is an infinite dictionary of simple atoms, such as all depth-$d$ trees with arbitrary thresholds or all decision stumps $\mathrm{sgn}(X-t),\,t\in\mathbb R$~\citep{bickel2006some}, then $\sF_k$ grows genuinely with $k$.  
For example, $k$ decision stumps on $\mathbb R$ yield piecewise-constant functions with up to $k+1$ intervals; sums of $T$ depth-$d$ trees are constant on the common refinement of $T$ partitions, with at most $(2^d)^T$ cells.  
In this regime $\sF$ is infinite-dimensional, and its growth with $k$ precisely captures the ``weak-to-strong'' expressivity gain attributed to boosting.  
As $k$ increases, the multicalibration guarantee over $\sF$ quantifies this expressive power rigorously. 
\end{enumerate}

\subsection{Projection operator: existence and limit characterization}

When an optional projection onto a closed convex set $O \subset L_2(\bbP_X)$ is used each round (e.g., $O := \{f: \mathcal{X} \to [0,1] \}$ ), the update becomes
$$
f^{(b+1)} = \mathcal{P}_O(f^{(b)} - \eta^{(b)} h^{(b)}). 
$$
Since $O$ need not be symmetric or affine, we adopt a mild operational hypothesis to control the projection step. 
\begin{ass}[Working Hypothesis]\label{ass:projection}
    Let $O \in L_2(\bbP_X)$ be a closed convex set. 
    \begin{itemize} 
        \item \emph{Inactive projection.}
        The projection becomes inactive beyond some finite iteration $B'$, i.e., 
        $$
        f^{(b+1)} = \mathcal{P}_O(f^{(b)} - \eta^{(b)} h^{(b)}) = f^{(b)} - \eta^{(b)} h^{(b)}, \text{ for all } b > B'. 
        $$
        \item \emph{Interior-at-limit.} The limit $f^{(\infty)} \in \text{Int}(O)$, where $\text{Int}(O)$ is the interior of $O$: there is a positive margin to the boundary. Meanwhile, $f^{(\infty)} \pm th \in O$ for all $h \in \sF$ with sufficiently small $t$. 
    \end{itemize}
\end{ass}
This condition is weaker and more realistic than requiring $\sL(\mathcal{P}_O(\cdot) \le \sL( \cdot ). $ at every step~\citep{deng2023happymap,zhang2024fair}, which generally fails for arbitrary convex $\sL$ and $O$. It is safer and standard to propose Condition~\ref{ass:projection} to handle projection steps in convex optimization. Then, we can reach the conclusion similar to~\ref{thm:stationarity}.
\begin{theorem}\label{thm:stationarity-proj}
    Under Conditions~\ref{ass:l2}-\ref{ass:projection}, with $c_L$-smooth $L$ and step size in Theorem~\ref{thm:stationarity}, the sequence $\sL(f^{(b)})$ is eventually nonincreasing (for all $b > B'$), and $f^{(b)} \to f^{(\infty)} \in \text{Int}(O)$ in $L_2$. Moreover, 
    $$
    \langle g^{(\infty)},h\rangle\ =\ 0\quad\text{for all }h\in\mathcal F,\qquad g^{(\infty)}(X):=\mathbb E[s(Y,f^{(\infty)}(X))\mid X],  
    $$ 
    and $f^{(\infty)}$ is a $\sF$-stationary in the sense that $\sL(f^{(\infty)})$ is optimal along directions in $\sF$, i.e., 
    $$
    \sL(f) \ge \sL(f^{(\infty)}), \quad \forall f \in O, \, f - f^{(\infty)} \in \sF.
    $$
\end{theorem}
The Bregman projection identity generally does not hold because our search space is $\sF$ rather than the entire space. Therefore, $f^{(\infty)}$ may be differ from the global Bregman projection $ \mathcal{P}_{O}^{\sL} (f^*; 0) = \argmin_{f \in O} D_{\sL}(f, f^*; 0)$; however, the two coincides when no projection is presented. 
Indeed, the limit satisfies a \emph{restricted variational inequality}~\citep{rockafellar1998variational}
\begin{equation*}
    \langle g^{(\infty)}, f - f^{(\infty)} \rangle \ge 0, \quad \forall f \in O \text{ with } f - f^{(\infty)} \in \sF.     
\end{equation*}
This can be expressed in a general subdifferential form  as 
$
0 \in \mathcal{P}_{\sF}\left( \partial \sL (f^{(\infty)} )  \right), 
$
which ensures
$$
\sL(f^{(\infty)}) \le \sL(\wt{f}), \quad \forall \wt{f} \in O, \wt{f} -f^{(\infty)} \in \sF.   
$$
This directional stationary is weaker than the first-order optimality $\mathcal{P}_O^{\sL} (f^*; 0)$, and $f^{(\infty)}$ have a higher objective value if $\sF$ is not rich enough. More detailed discussion can be found in Section~\ref{apxsec:projection}. 
In other words, the stagewise limit $f^{(\infty)}$ need only be directionally stationary along $\sF$, in contrast to $\mathcal{P}_O^{\sL} f^*$ which satisfies 
$$ 
\langle - \partial \sL(\mathcal{P}_O^{\sL} f^*), f -\mathcal{P}_O^{\sL} f^* \rangle \le 0, \quad \forall f \in O, 
$$
which is a stronger condition than that required by $f^{(\infty)}$, reflecting the possible difference between these two -- $f^{(\infty)}$ may have a higher objective value if $\sF$ is not rich enough.

\subsection{Convergence rates}\label{apxsub:convergence}
Having established convergence of the boosting sequence, we next quantify the rate at which the excess risk $\Delta_b = \sL(f^{(b)}) - \sL(f^{(\infty)})$ vanishes. To this end, we posit an mild boundedness condition on the boosting path: 
\begin{ass}[Finite radius]\label{ass:finite-radius}
There exists $R<\infty$ such that $\|f^{(b)}-f^{(\infty)}\|_{L_2}\le R$ for all $b\ge0$.
\end{ass}
This condition says that all iterates remain in a $L_2$-ball around the limit, which is slightly stronger than the Cauchy condition  $\|\sum_{j = m}^{n} \eta^{(j)} h^{(j)}\|_{L_2} \to 0, m, n \to \infty$; indeed,  
$\|f^{(b)}-f^{(\infty)}\|_{L_2} \le \left\| \sum_{j = b}^{B'} \eta^{(j)} h^{(j)} \right\| + \epsilon $ for any $\epsilon > 0$ and sufficient large $B'$.

\begin{theorem}\label{thm:smooth-rate}
Suppose $L$ is $c_L$-smooth and use the step
\(
\eta^{(b)}=\dfrac{\langle g^{(b)},h^{(b)}\rangle}{2c_L\|h^{(b)}\|_{L_2}^2}.
\)
Under Conditions~\ref{ass:l2}--
\ref{ass:finite-radius}, the excess risk satisfies
$$
\ \Delta_B\ \le\ \frac{4c_L R^2}{\kappa^2}\cdot \frac{1}{B}.
$$
\end{theorem}  
Theorem~\ref{thm:smooth-rate} elucidates an $O(1/B)$ convergence rate for excess risk under $c_L$-smooth condition. Stronger curvature assumptions can further accelerate convergence, such as $L_2$ loss. In particular, if the loss satisfies the Polyak–Łojasiewicz (PL) condition~\citep{karimi2016linear}, which is weaker than the strong convexity, the excess risk decays geometrically. Notably, our analysis only requires the PL condition with respect to the restricted dual norm on $f^{(0)} + \sF$, which is weaker than the usual global PL condition and aligns naturally with the structure induced by the boosting updates.
For more general convex losses for which the $c_L$-smooth condition fails, the faster $O(1/B)$ and linear rates need not hold. In that case, the usual subgradient argument gives an $O(B^{-1/2})$ guarantee for the averaged iterate; see Appendix~\ref{apxsubsec:nonsmooth-convex}.

To show that the PL condition yields linear rate, we define the $\mu$-PL condition as  follows. 
\begin{defn}[Polyak-Łojasiewicz on restricted dual norm]\label{def:PL}
We say $\mathcal L$ satisfies a $\mu$-PL condition (w.r.t.\ $\|\cdot\|_*$) on $f^{(0)}+\mathcal F$ if
\[
\frac12\|g\|_*^2\ \ge\ \mu\,\big(\mathcal L(f)-\mathcal L(f^{(\infty)})\big)
\quad\text{for all }f\in f^{(0)}+\mathcal F,\ g=\mathbb E[s(Y,f(X))\mid X].
\]
\end{defn} 

\begin{theorem}
\label{thm:PL-linear}
Under $c_L$-smoothness,  Condition~\ref{ass:weak-edge}, and the $\mu$-PL condition~\ref{def:PL},
the same step as in Theorem~\ref{thm:smooth-rate} yields
$$
\ \Delta_{b+1}\ \le\ \Big(1-\frac{\mu\kappa^2}{2c_L}\Big)\Delta_b \quad \Longrightarrow\quad
\Delta_B\ \le\ \Big(1-\frac{\mu\kappa^2}{2c_L}\Big)^B\Delta_0.
$$
\end{theorem} 
For $\sL = \Expect(Y-f)^2/2$, the PL conditon holds automatically. Indeed, $\sL$ is $1$-strongly convex: 
$$
\sL(f_1) \ge \sL(f_2) + \langle \nabla \sL(f_2), f_1 - f_2 \rangle + \frac{1}{2} \|f_1 - f_2 \|_{L_2}^2. 
$$
Rearranging the above yields 
$
\sL(f^{(\infty)}) -  \sL(f)  \ge \langle \nabla \sL(f), f^{(\infty)} - f \rangle + \frac{1}{2}  \|f^{(\infty)} - f \|_{L_2}^2, 
$
implying 
$$
\|g\|_* \| f - f^{\infty} \| \ge  \langle \nabla \sL(f),  f  -  f^{(\infty)} \rangle \ge \frac{\mu_1}{2}  \|f^{(\infty)} - f \|_{L_2}^2, \forall f \in f^{(0)} + \sF, 
$$
where $\mu_1 = 1$ for squared loss. This gives the Error Bound~(EB) condition 
$$
\|g\|_* \ge \frac{\mu_1}{2}  \|f^{(\infty)} - f \|_{L_2}.  
$$
Combing EB and $c_L$-smoothness ($c_L = 1/2$ here)
$$
\sL(f) \le \sL(f^{(\infty)}) + \langle \nabla \sL(f^{(\infty)}), f - f^{(\infty)} \rangle + c_L \|f^{(\infty)} - f\|_{L_2}^2, \forall f \in f^{(0)} + \sF,   
$$ 
yields 
$$
\sL(f) - \sL(f^{(\infty)}) \le  c_L \|f^{(\infty)} - f\|_{L_2}^2 \le c_L \frac{4}{\mu_1^2} \|g\|_{*}^2,
$$
Thus, 
$$
\frac{1}{2} \|g\|_{*}^2 \ge \frac{\mu_1^2}{ 8 c_L } (\sL(f) - \sL(f^{(\infty)})). 
$$ 
So the squared loss satisfies the $\mu$-PL condition with $\mu=1/4$.
\begin{corollary}
\label{cor:PL-B}
Under $c_L$-smoothness, Conditions~\ref{ass:l2}-\ref{ass:weak-edge}, and the $\mu$-PL condition~\ref{def:PL}, $\Delta_B = O(\alpha^2)$ is guaranteed by 
\begin{align*}
    B \gtrsim \frac{1}{\log\!\big((1-\mu\kappa^2/(2c_L))^{-1}\big)}\,
\log\!\Big(\frac{C_{\mathcal F}\sqrt{2 \Delta_0}}{\alpha^2} \Big),  
\end{align*}
provided $\mu\kappa/(2c_L) < 1$.
\end{corollary} 

\begin{rem}\label{rem:square-loss}
    {\rm 
    The PL framework connects excess risk directly to calibration error. Indeed, 
    $$
    \mathrm{MC}_{\mathcal F}(f^{(B)}):= \sup_{h \in \sF, \|h\|_{L_2} \le C_{\sF} } \left|\Expect h \cdot s(Y, f^{(B)}(X)) \right| \le C_{\sF} \|g^{(B)}\|_*,  
    $$ 
    then $\Delta_B$ is small if $\alpha > C_{\sF} \|g^{(B)}\|_* > C_{\sF} \sqrt{ 2\mu \Delta_B } $. 
    For squared loss, more directly,  
    \begin{align*}
    & \left|\Expect h \cdot s(Y, f^{(B)}(X)) \right| =   \left|\Expect h \cdot (f^{(\infty)}(X) - f^{(B)}(X)) \right| \\ 
    & \le C_{\sF} \|f^{(B)} - f^{(\infty)}\|_{L_2} \le   C_{\sF} \sqrt{\frac{2}{\mu_1} \Delta_B } =  C_{\sF} \sqrt{2 \Delta_B }.    
    \end{align*} 
    Small excess risk $\Delta_B$ immediately yields small multicalibration error. Compared with the $O(1/B)$ decay under smoothness alone, the $\mu$-PL condition provides a linear rate, showing that significantly fewer steps suffice to achieve a target excess risk, $L_2$-risk, and multicalibration level.
    } 
\end{rem}

\begin{rem}
    {\rm 
     With a projection, the inactive projection condition (Condition~\ref{ass:projection}) allows the algorithm to maintain the same convergence rates for sufficiently large $B$ (see Appendix~\ref{apxsec:convergence-rate})  and to retain the same multicalibrated class $\sH$. 
    }
\end{rem}

\subsection{Finite-step guarantees for practical multicalibration} 

The preceding results established the asymptotic convergence of MCBoost. In practice, however, the algorithm operates with finite samples and halts after finitely many iterations according to the stopping criterion. This raises two natural questions: (I) Does the procedure necessarily terminate in finite time~(given sufficient data)? (II) Does the stopping rule ensure that the resulting predictor is multicalibrated? We answer both affirmatively under the weak-learner edge condition. 

\subsubsection{Guaranteed finite termination.}

In what follows, we regard both the calibration and validation sets as non-negligible fractions of the total sample $N = N_1 + N_2$ in Algorithm~\ref{alg:MCboost-general}, i.e., $\# \Xi = \Theta(N), \# V = \Theta(N)$, so that both the auditing and the violation evaluation steps are effective. 
\begin{ass}\label{ass:VC}  
Along the boosting trajectory, the condition mean score 
$\|g^{(b)}\|_{\infty} \le \wt{C}$, and there exists an envelope function $H(\cdot)$ such that $| h(X) | \le H(X) $ with $\|H^2\|_{L_2} < \infty$.   Both classes $\sHa$ and $\sHa^2$ are VC classes with VC dimensions $V(\sHa) = O(V(\sHa^2))$. Moreover, a \emph{variance floor} holds uniformly:  $\sum_{i=1}^N (h^{(b)}(X_i))^2/N \ge c_{\sHa}^2 > 0$.   
\end{ass}
Classes satisfying the VC condition include: (1) Two-layer neural networks $$\sHa = \{ \text{Acitvate}(w^\top X + w_0), w \in \bbR^d, w_0 \in \bbR\}.$$ (2) Tree-basis functions $\sHa = \{\v1\{ (-\infty, a_1] \times \cdots \times (-\infty, a_d]: a_1, \dots,a_d \in \bbR \} \}$. (3) Finite-dimensional linear classes: $\sHa = \{ w^\top \phi(X): \|w\|_2 \le C_w, X \in \bbR^d \} $ with $\phi: \mathcal{X} \to \bbR^d$, $\lambda_{\max}(\Sigma) < \infty, \Sigma = \Expect \phi(X)\phi^\top (X)$. Both the tree-based and neural-network classes have finite VC-dimensions, and their linear spans are dense in $C(U)$ for any compact subset $U \in \bbR^d$~\citep{zhang2005boosting}. 

The above condition is weaker than the uniform boundedness condition $\|h\|_{\infty} \le C_{\sHa}, h \in \sHa$ used in earlier work. 
For instance, one may have $h = w^\top X, \|w\|_{L_2} \le 1$ and $H = \|X\|$, where $X$ itself need not be bounded. Classical concentration inequalities such as McDiarmid's or Hoeffding's apply in the bounded case; here, we rely on Dudley's entropy integral with VC complexity to handle more general settings.

\begin{proposition}\label{prop:finite}
    Suppose Conditions~\ref{ass:l2},~\ref{ass:VC} hold and 
    $L$ is $c_L$-smooth. Then,  with probability at least $1 - \delta$,  
    the MCBoost algorithm stops in $B = O\left( \frac{\sL(f^{(0)}) - \sL(f^{(\infty)})}{\alpha^2}\right)$ with the step size in Theorem~\ref{thm:smooth-rate} and $N = \Omega\left( \frac{ \log(B/\delta) + V(\sHa)}{\alpha^2} \right)$ samples. 
    If, in addition, the loss satisfies $\mu$-PL condition, then the number of iteration reduces to $B = O\left(\log(\tfrac{\sL(f^{(0)}) - \sL(f^{(\infty)} }{\alpha^2} ) \right)$ under the same sample complexity.
\end{proposition} 
Earlier studies~\citep{kim2019multiaccuracy, zhang2024fair} considered 
settings with abundant data, drawing a new validation batch at each iteration; in that regime, the total sample requirement scales as $N = \Omega\left(B \cdot \tfrac{\log(1/\delta) + V(\sHa)}{\alpha^2}\right)$ 
in Proposition~\ref{prop:finite}. 
In contrast, when data are limited, one may take $V = \Xi$, a common practice in boosting and conformal prediction. This result complements Theorem~\ref{thm:smooth-rate} and~\ref{thm:PL-linear}, the earlier results quantify the asymptotic rate of risk reduction, whereas Proposition~\ref{prop:finite} provides a finite-sample halting guarantee, linking the number of steps, the tolerance level $\alpha$, and the sample size. 

\subsubsection{Validity of stopping rule.}
\begin{theorem}\label{thm:stopping}
Under Conditions~\ref{ass:l2}, \ref{ass:weak-edge} and~\ref{ass:VC}, the algorithm stops at iteration $B$ with sample size $N$, where $B$ and $N$ satisfy Proposition~\ref{prop:finite}; then we have 
\[
\frac{|\langle g^{(B)},h^{(B)}\rangle|}{\|h^{(B)}\|_{L_2}}\ < 2 \alpha 
\]
Then $\|g^{(B)}\|_* < 2 \alpha/\kappa $ and
$$
\left|\Expect h \cdot s(Y, f^{(B)}(X)) \right| \le \|h\|_{L_2} \frac{2\alpha}{\kappa}.
$$
Hence, 
\[
\sup_{h \in \sF, \|h\|_{L_2} \le C_{\sF} } \left|\Expect h \cdot s(Y, f^{(B)}(X)) \right| \le C_{\sF} \frac{2 \alpha}{\kappa}. 
\]
\end{theorem}
Theorem~\ref{thm:stopping} shows that once the empirical violation statistic falls below the threshold $\alpha$, the correlation between the score $s(Y, f^{(B)})$ based on $f^{(B)}$ and any $h \in \sF$ is uniformly bounded by a constant proportional to its norm, the tolerance level $\alpha$, and the weak-learner constant $\kappa$. The bound generally applies to any function $h$ representable within $\sF$. In particular, when $\sF$ arises from an auditing family $\sHa$ consisting of indicator-type functions, such as group indicators, value-based partitions, or their finite combinations, their inherent boundedness ensures that the scaling factor remains controlled, so the multicalibration guarantee holds without concern for unbounded magnitudes of $h$. 

Taken together, Propositions~\ref{prop:finite} and Theorem~\ref{thm:stopping} 
establish the practical soundness of the MCBoost procedure: it halts after finitely many steps under realistic sample sizes, and the stopping criterion itself certifies multicalibration of the learned predictor.

\subsection{Empirical illustration: early stopping and overfitting} 
Theorem~\ref{thm:stationarity} relies on an ideal, adaptive (line-search-like) step size 
$\eta^{(b)} = \tfrac{\langle h^{(b)}, s(Y, f^{(b)}) \rangle }{ 2c_L \|h^{(b)}\|_{L_2}^2}$ and assumes a Cauchy-boosting path,
which is guaranteed when step sizes decay fast such that $ \sum_{b=1}^{\infty} \eta^{(b)} < \infty$ (Appendix~\ref{apxsubsec:stationarity}), i.e., when the total step budget is bounded. Conversely, the Cauchy-increment condition on $\{\eta^{(b)} h^{(b)}\}$ along does not imply $\sum_{b=1}^{\infty} \eta^{(b)} < \infty$. These assumptions differ from those in~\cite{zhang2005boosting}, which require $\sum_{b=1}^{\infty} \eta^{(b)} = \infty$ and $\sum_{b=1}^{\infty} \eta^{(b)2} < \infty$. 
Nevertheless, these considerations highlight the practical necessity of  early-stopping as a regularization device. This is particularly relevant when $c_L$ is unknown and one resorts to a small fixed step size for implementation. 
For least-squares boosting,~\citet{buhlmann2003boosting} shows that using a fix step with a powerful base learner can overtrack noise as $b \to \infty$, which eventually interpolates the data and causes overfitting.\citet{zhang2005boosting} further formalizes early stopping through the cumulative step-size budget and emphasizes its role as a key regularization mechanism. Our experiments demonstrate that the same phenomenon carries over to MCBoost. 

We revisit the simulation setup in Section~\ref{subsec:motivation}. Throughout, MCBoost partitions the calibration data into $|\mathcal{G}| \times L$ bins based on categorical group membership and prediction levels to promote groupwise calibration. We consider $L \in \{1,2,4\}$ and either group by $X^{(d)} = (X_6, X_7) \in \{0,1\}^2$ or not, which yields $|\mathcal{G}|=1$ (no grouping) or $|\mathcal{G}|=4$ (four demographic groups).
\begin{figure}[ht!]
    \centering
    \begin{subfigure}[b]{0.48\textwidth}
        \includegraphics[width=\linewidth]{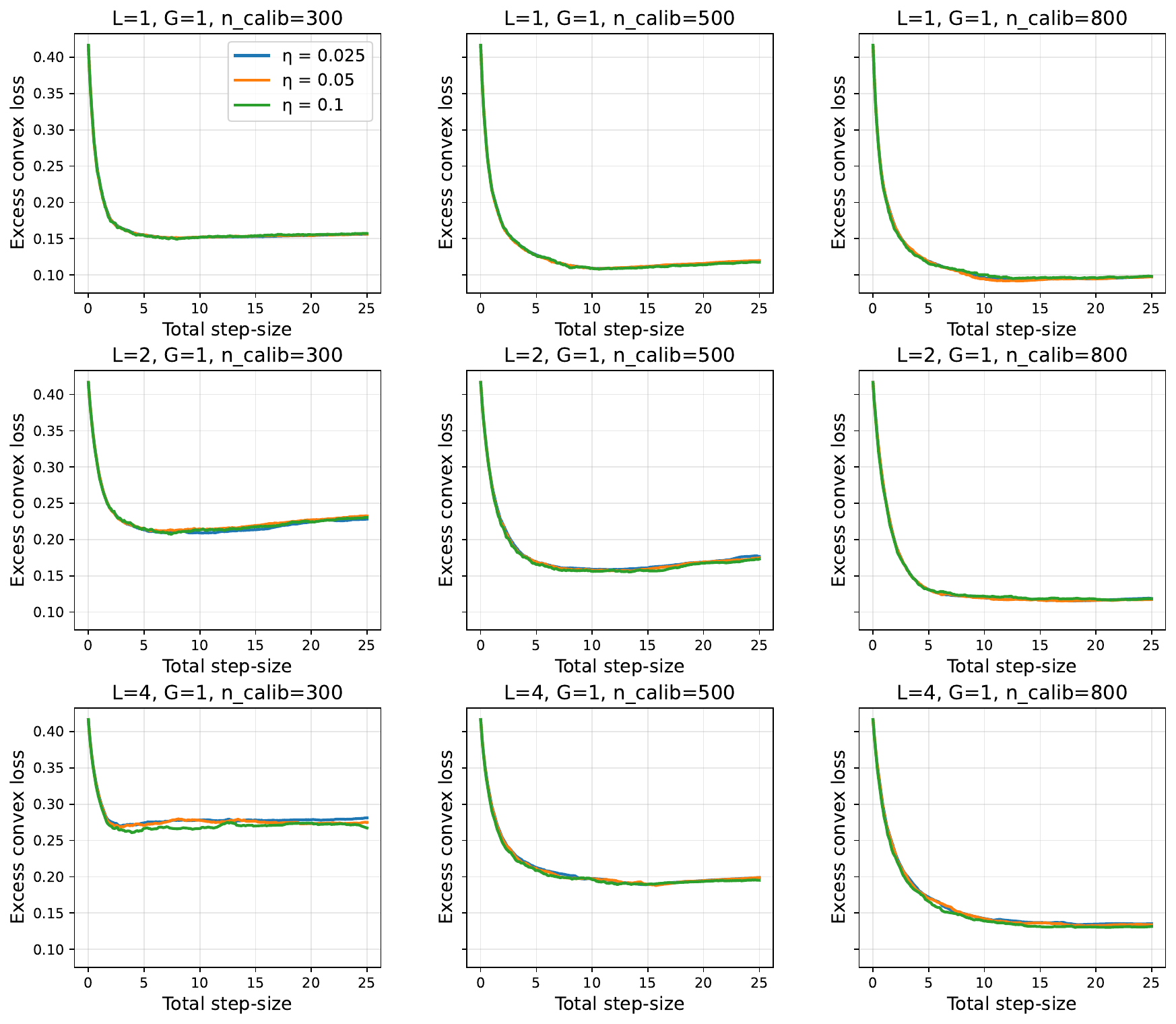}
        \caption*{(a)}
    \end{subfigure}
    \begin{subfigure}[b]{0.48\textwidth}
         \includegraphics[width=\linewidth]{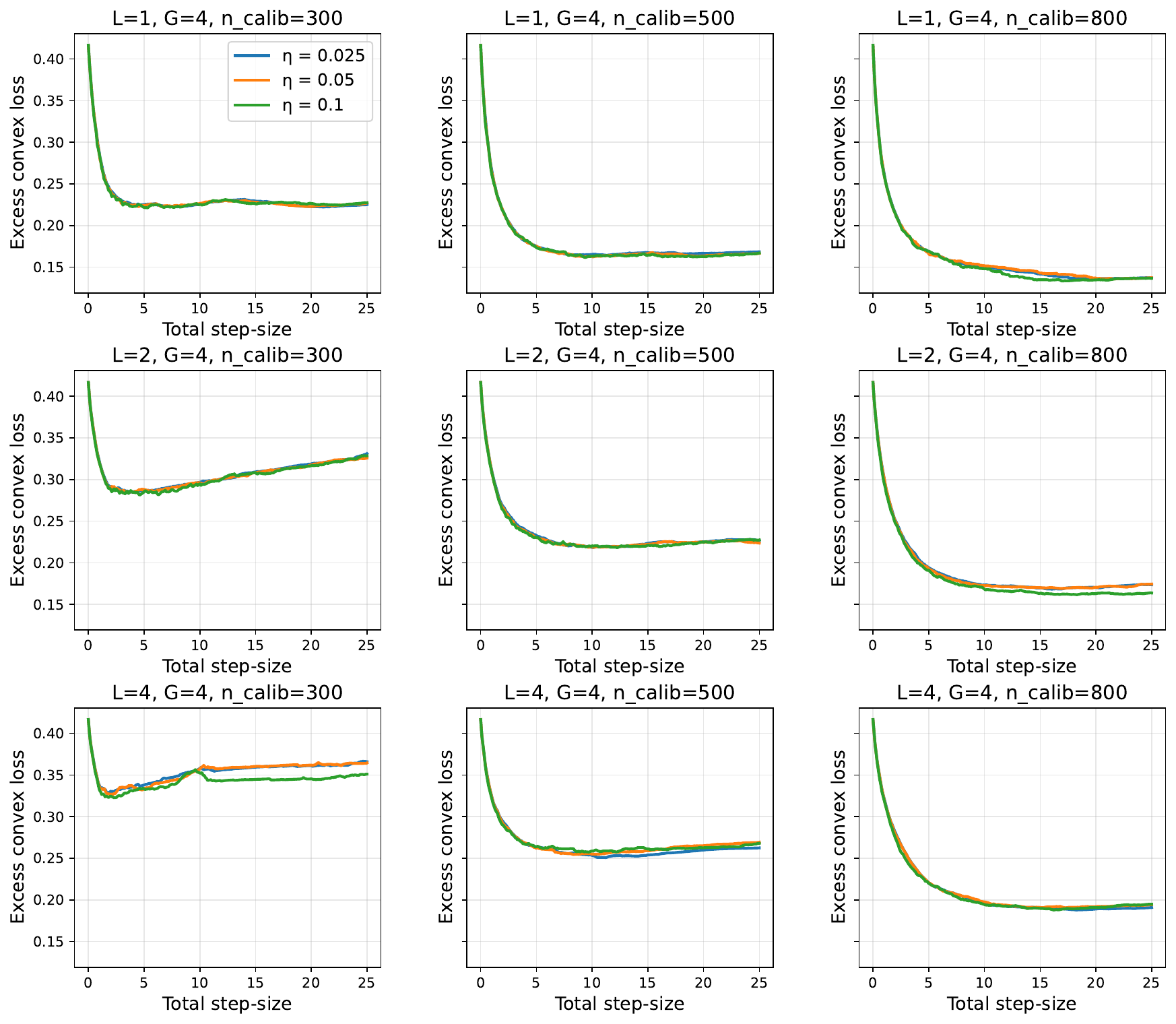}
         \caption*{(b)}
    \end{subfigure}
    \caption{Excess convex loss versus cumulative step size. The initial predictor is linear and the auditor is tree-based. The panels vary the number of groups, the number of buckets, and the calibration sample size.}
    \label{fig:excess_vs_stepsize-tree}
\end{figure}
\begin{figure}[ht!]
    \centering
     \begin{subfigure}[b]{0.48\textwidth}
         \includegraphics[width=\linewidth]{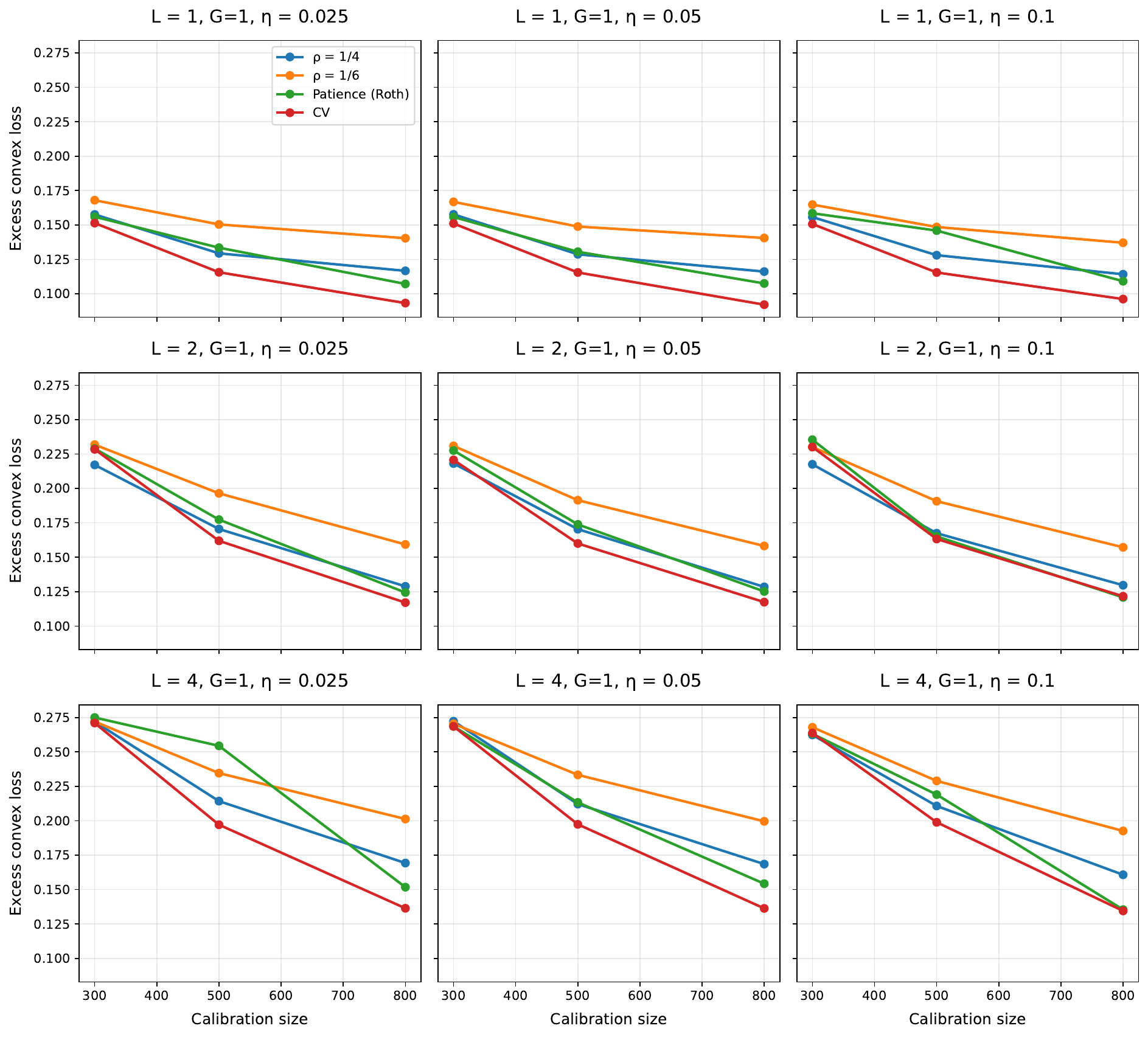}
     \end{subfigure}
     \begin{subfigure}[b]{0.48\textwidth}
          \includegraphics[width=\linewidth]{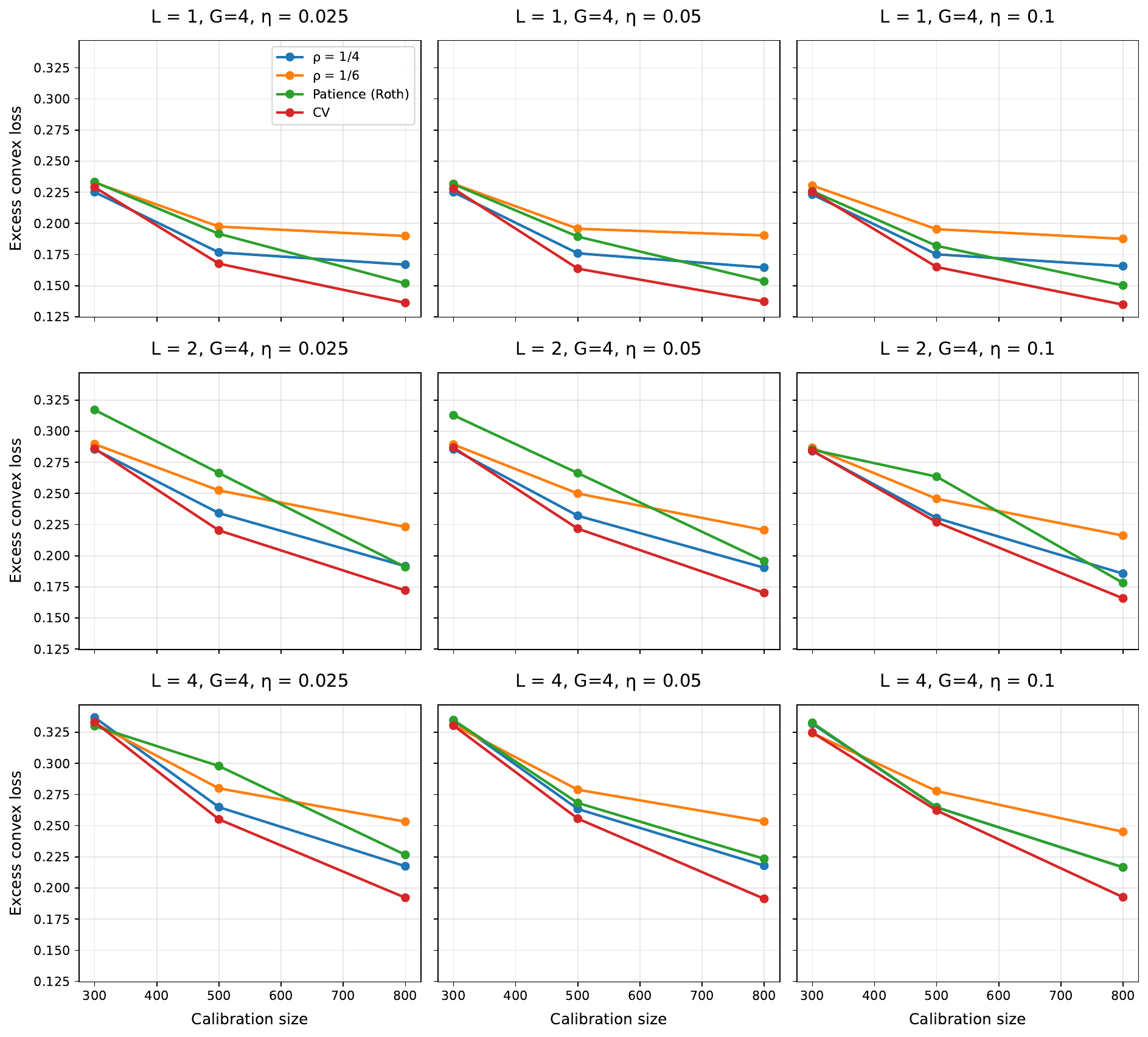}
     \end{subfigure}
    \caption{Excess convex loss versus calibration sample size. The initial predictor is linear and the auditor is tree-based. The panels vary the number of groups, buckets, and stopping rules.}
    \label{fig:excess_vs_ncalib-tree}
\end{figure}
We evaluate the global excess convex risk
$
\big[ \Expect(Y - f^{(b)}(X))^2 - \Expect(Y - f^{*}(X))^2 \big]/2
$
as a function of the cumulative step size. Across a range of nominal step sizes, the optimal stopping point occurs in a nearly identical total step budget (Figure~\ref{fig:excess_vs_stepsize-tree}). Moreover, as the sample size increases, the optimal stopping point shifts later and the admissible stopping range widens. This pattern is consistent with the classical boosting view of early stopping as regularization~\citep{zhang2005boosting}.

Notably, the same figure also reveal a trade-off in MCBoost between calibration and predictive accuracy. Incorporating more groups and finer bucket structures improves calibration (Section~\ref{sec:numerics}), but can slightly increase excess risk, particularly when the data are partitioned into many small cells. While traditional boosting directly optimizes global MSE, multicalibration imposes additional structural constraints. Consequently, the empirical ordering
\[
\textrm{Naive} \;\ge\; \textrm{MCBoost} \;\ge\; \textrm{Traditional boosting}
\]
in optimal MSE (Figure~\ref{fig:excess_vs_stepsize-tree}) reflects the cost of enforcing calibration.

Figure~\ref{fig:excess_vs_ncalib-tree} compares several stopping strategies: fixed total-step budgets of order $N^{\rho}$ with $\rho \in \{1/4,1/6\}$, three-fold cross-validation, and patience-based rules~\citep{detommaso2404multicalibration}. Cross-validation performs robustly across settings, while patience-based stopping tends to slightly underperform, as it allocates additional data for monitoring.

\subsection{Target-independent learning under covariate shift}
The multicalibration framework can be extended beyond the source distribution to ensure valid guarantees under covariate shift. Let $\mathcal D_\sS$ and $\mathcal D_\sT$ denote the joint distributions of $(X, Y)$ in the source and target domains, respectively, with density ratio 
$$
w(X) = \frac{ p(X|D = \sT) }{p(X| D = \sS)} = \frac{\bbP(D = \sT | X)}{\bbP(D = \sS | X)} \frac{\bbP(D = \sS)}{\bbP(D = \sT)}. 
$$
Here $D$ denotes domain membership, $\sS$ and $\sT$ represent the source and target domains, respectively.
The key question is if a predictor $\wt{f}$ that is multicalibrated on the source domain could implicitly indicate 
universal adaptability such that 
$
\big| \mathbb E_{\mathcal D_\sT} s(Y, \wt{f}(X) )] \big| 
$
is small, and remain multicalibrated to some extent, i.e., 
$$
\big| \mathbb E_{\mathcal D_\sT}\big[ c \cdot s(Y, \wt{f}(X)) \big] \big|  < \alpha, \quad \forall c(\cdot) \in \mathcal{C}.  
$$
for some function class. 
\begin{theorem}[Target-independent multicalibration under covariate shift]\label{thm:universal} 
Let $(X,Y)\sim\mathcal D_\sS$ denote the source distribution and $(X,Y)\sim\mathcal D_\sT$ the target distribution, with density ratio $w(X)$. 
Suppose the score function admits a decomposition: 
$$
s(Y, f(X)) = s_1(Y, f(X)) + s_2(Y), 
$$
and define the error of a propensity-score weighted estimator using weights $\wt{w}$ as 
$$
\Err(\tau^{(ps)}(\wt{w})) = \left| \Expect_{\sD_\sS} [ (\wt{w}(X) - w(X) )  s_2(Y)] \right|,
$$
Assume further that $\wt{f}$ is  $(\mathcal H, \alpha)$-multicalibrated on $\mathcal D_\sS$ and  
$$
\| \Expect[s_1(Y, f(X)) | X] \|_{L_2} \le \wt{C}, \quad \| \Expect[s_2(Y) | X] \|_{L_2} \le \wt{C}. 
$$
Define
\[ \mathcal H(\Sigma) := \Big\{
\wt\sigma_\sT(X) = \bbP(D=\sT\mid X)
\Big\}, 
\]
the class of shifts induced by the propensity scores. Then the following hold. 
\begin{enumerate}[label=(\roman*), leftmargin = *]
\item \textbf{Universal adaptability.} 
The target error satisfies
\begin{align*}
  &  \Err_{\sT}(\wt f)
:= \Big|\mathbb E_{\mathcal D_\sT} s(Y, \wt{f}) \Big|  \le\ \wt{C} \inf_{h\in\mathcal H}\{ \| w -  h/C \|_{L_2} + \| h/C - \wt{w} \|_{L_2} \}+ \Err(\tau^{(ps)}(\wt{w})) + \tfrac{\alpha}{C}, 
\end{align*}
for any constant $C>0$.
In particular, if $\wt w\in\mathcal H$, 
$$
\Err_{\sT}(\wt f)\le \Err(\tau^{ps}(\wt{w})) + \wt{C} \, \| w - \wt w\|_{L_2} +\alpha.
$$
Furthermore, if $ w^* := \inf_{\wt{w}\in \mathcal H(\Sigma)} \|w-\wt{w}\|_{L_2} $ and $w^* = h^* / C^* $ for some $h^* \in \sH$, the predictor $\wt{f}$ achieves \textit{Universal Adaptability}:
\begin{align*}
\Err_{\sT}(\wt f) \le \Err(\tau^{(ps)}( w^* )) +
\wt C\ \| w - w^*\|_{L_2}
+ \alpha/C^*.
\end{align*}  
\item \textbf{Transfer of multicalibration.} Let $\mathcal{C}$ be a reference function class with $\|c(\cdot)\|_{L_2} \le \wt{C}$,   
if $\wt f$ is $(\mathcal H(\Sigma)\otimes\mathcal C,\alpha)$-multicalibrated on $\mathcal D_\sS$,
then it is $(\mathcal C,\alpha+ \wt{C}^2 \inf_{\wt w\in\mathcal H(\Sigma)} \|w - \wt{w}\|_{L_2})$-multicalibrated on $\mathcal D_\sT$.
\item \textbf{Multiple source domains.} 
Suppose $\sS$ aggregates $M$ source domains 
$\sS_1, \ldots, \sS_M$ with weights $p_m = \bbP(D = \sS_m | D \in \sS )$. If 
$$
\sup_{h \in \sH} \big| \Expect_{\sD_\sS} h s(Y, \wt{f}) \big| < \alpha, 
$$
then 
\begin{align*}
    &\Err_{\sT}(\wt{f}) \le \sum_{m=1}^M p_m \wt{C} \inf_{h \in \sH} \left\{ \| w_m -  h/C_m \|_{L_2} + \| h/C_m - \wt{w}_m \|_{L_2} \right\} \\
    & \quad + (\max_m \tfrac{1}{C_m} ) \alpha + \Err(\tau^{(ps)}( \wt{\vw} )),  
\end{align*}
where $ w_m(X): = p(X|\sD = \sT)/p(X|\sD = \sS_m)$ and 
$$
\Err(\tau^{(ps)}(\wt{\vw})) = \left| \sum_{m=1}^M p_m \Expect_{\sU_{\sS_m}} (\wt{w}_m(X) - w_m(X) ) g_2  \right|, \quad  g_2 = \Expect[s_2(Y)|X].   
$$ 
\end{enumerate}  
\end{theorem}
Theorem~\ref{thm:universal} establishes that  multicalibration can be extended naturally to covariate shift settings.  If the density ratio $w$ can be well-approximated within $\sH$, the same boosting scheme yields predictors that adapt universally to new target distributions while preserving conditional calibration guarantees. 

For concreteness, in regression with squared loss, $s_1(Y, f(X)) = f(X)$, $s_2 = -Y$; in quantile regression ($\wt{C} = 1$), 
$s_1(Y, q_{\tau}(X)) = \v1\{Y \le q_{\tau}(X) \}, s_2 = -\tau$, 
the target coverage error at level $\tau$ obeys
$$
\left| \mathbb E_{\sD_\sT} \left[ \big(\v1\{Y \le \wt q_\tau(X)\}-\tau\big)\right] \right|
    <  \inf_{h\in\mathcal H}\{\| w - h/C \|_{L_2} + \| h/C-  \wt{w})\|_{L_2} \}  + \tfrac{\alpha}{C}.
$$ 
In particular, if $\| \Expect[s_1(Y, f(X)) | X] \|_{\infty}, \| \Expect[s_2(Y) | X] \|_{\infty} \le \wt{C}$, then the $L_2$-norms in the above bounds can be replaced by their $L_1$ counterparts. 

Empirical studies further corroborate these theoretical insights. Recent work~\citep{kim2022universal, ye2024multicalibration} demonstrates that multicalibrated predictors under covariate shift perform on par with, and often surpass, standard propensity-score reweighting estimators, both in overall and in subgroup-level performance. These results underscore the strength of multicalibration as a unified principle for achieving reliable generalization and fairness under distributional heterogeneity. 

\begin{rem}
    {\rm 
    Viewed from another angle, Theorem~\ref{thm:universal} show that multicalibration on the source domain automatically transfer to a broad family of covariate shift. Specifically, if we learn a class $\sH$, then for any nonnegative function $h(X) \in \sH$, the induced shift 
    $$
        p(X | \sT) = \frac{h(X)}{\Expect_{\sU_\sS}h(X)}p(X|\sS). 
    $$
    is covered by the guarantee. In this case, the theorem implies 
    $$
    \left|\; \mathbb E_{\sD_\sS}\!\left[ \frac{h(X)}{\mathbb E_{\sU_\sS}[h(X)]}\,\big(\v1\{Y \le \wt q_\tau(X)\}-\tau\big)\right] \right|
    < \frac{\alpha}{\mathbb E_{\sU_\sS}[h(X)]},
    $$
    by choosing $w=\wt w=h(X)/\mathbb E[h(X)]$ and $C=\mathbb E[h(X)] > 0$. 
    In multi-source settings, the same reasoning implies
    $$
    \left| \mathbb E_{\sD_\sT} \left[ \big(\v1\{Y \le \wt q_\tau(X)\}-\tau\big)\right] \right| < \frac{\alpha}{\min \Expect_{\sU_{\sS_m}} h_m(X)}, 
    $$
    whenever the $m$-th source-to-target shift is $h_m \in \mathcal{H}$. This is obtained by by choosing $w_m = \wt{w}_m = h_m(X)/\Expect_{\sU_{\sS_m}} h_m(X)$.  
    The perspective underscores the universality of multicalibration: it extends to a family of target distribution once established on the source.
    }
\end{rem}

\section{Numerical studies of auditor mechanisms and covariate shift}\label{sec:numerics} 
We now return to the simulation design introduced in Section~\ref{subsec:motivation} and use it to study the empirical mechanisms suggested by the theory. The focus is not on literal empirical verification of the exact rates or constants in the theory, but on understanding how auditor richness, group information, and shift structure shape the practical behavior of MCBoost. 

\subsection{Auditor expressivity, group information, and calibration error}
We study how different auditing strategies, grouping structures, and prediction partitions modify the practical gains from MCBoost.
Cross-validation is used for early stopping so that calibration gains are not driven by uncontrolled overfitting.
\begin{figure}[ht!]
    \begin{subfigure}[b]{\textwidth}
        \includegraphics[width=\textwidth]{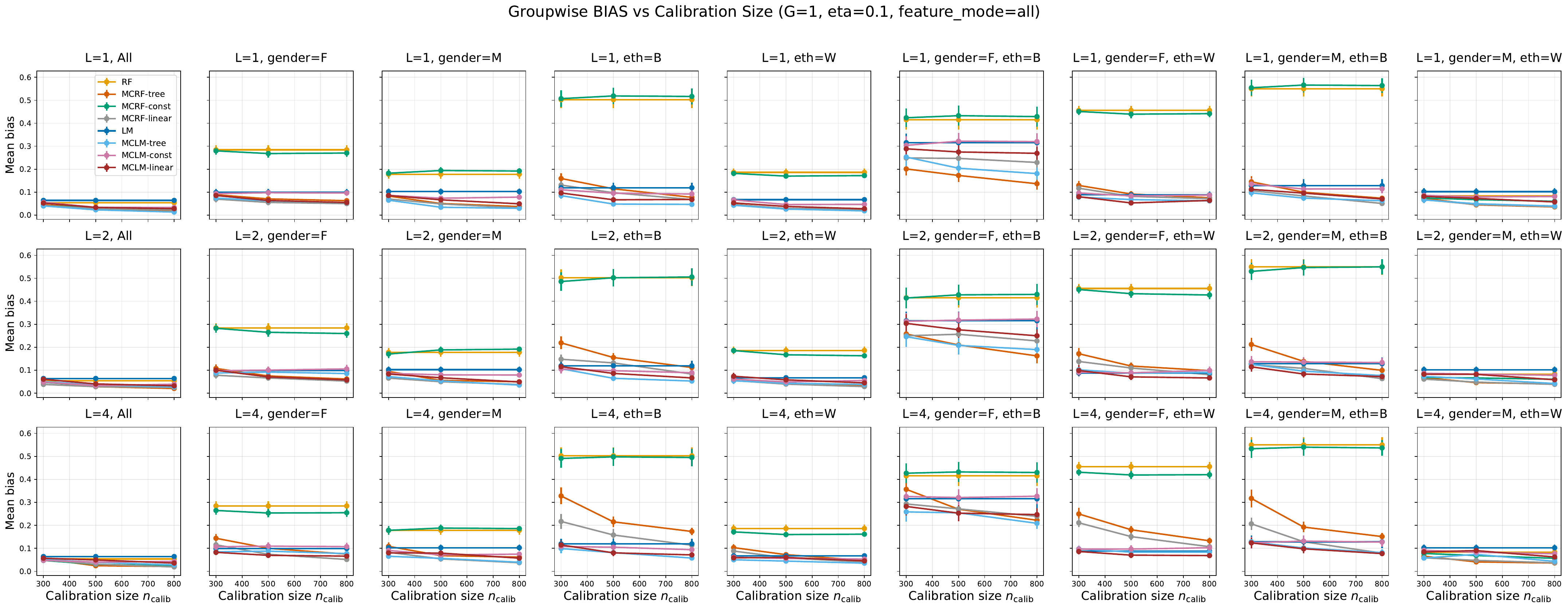}
    \end{subfigure}
    \begin{subfigure}[b]{\textwidth}
        \includegraphics[width=\textwidth]{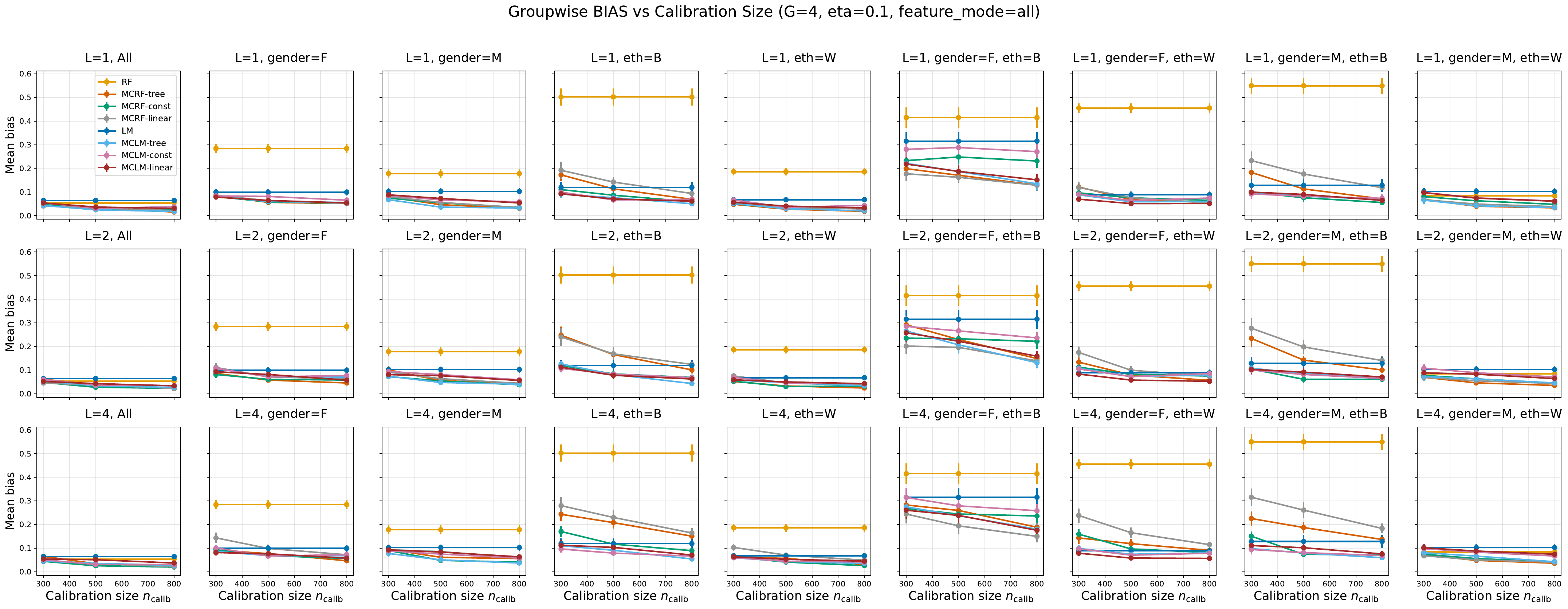}
    \end{subfigure}
    \begin{subfigure}[b]{\textwidth}
        \includegraphics[width=\textwidth]{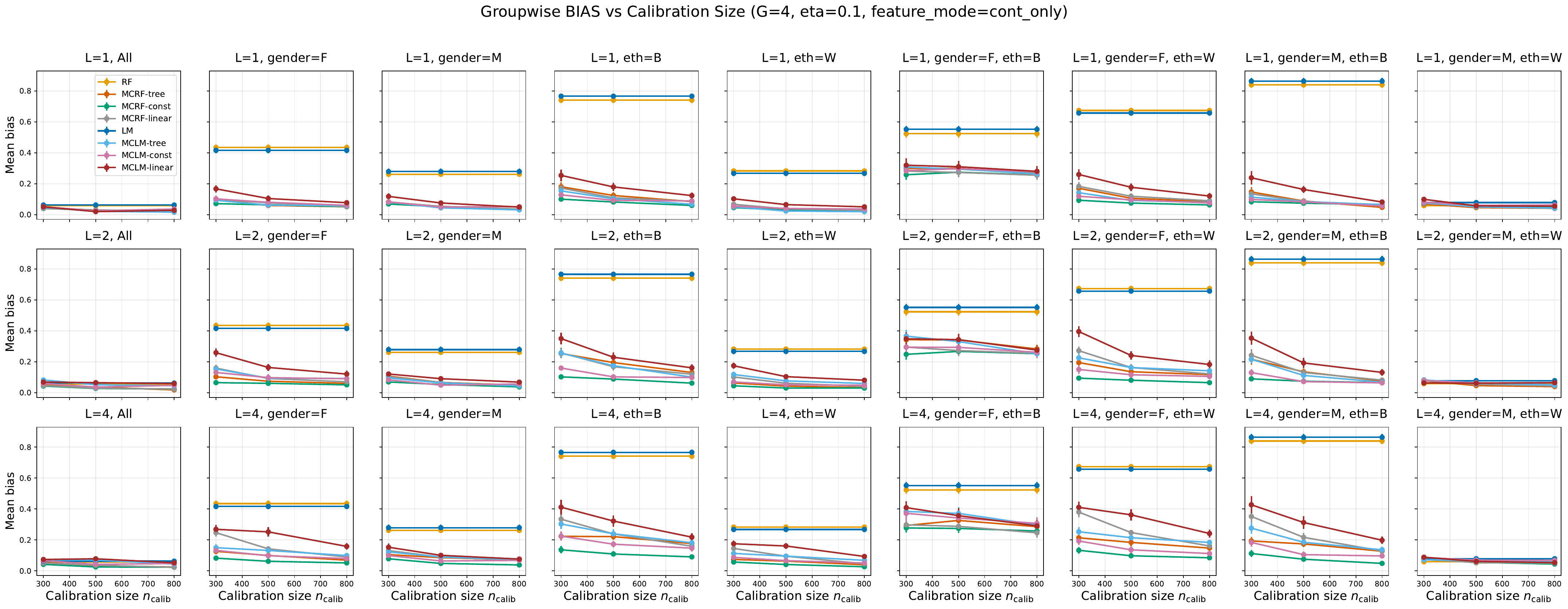}
    \end{subfigure}
    \caption{Mean groupwise biases across subpopulations  
    vs. calibration size for different initial predictors (linear model and random forest) and auditors (decision tree, constant, and linear). Top:
    bucket-based calibration without group partition ($|\mathcal{G}| = 1$). 
    Middle: calibration with both bucket and group partition ($|\mathcal{G}| = 4$).  
    Bottom: no group partition and categorical features excluded from model fitting. 
    }
    \label{fig:bias_vs_calib}
\end{figure}
\begin{figure}[ht!]
    \begin{subfigure}[b]{\textwidth}
        \includegraphics[width=\textwidth]{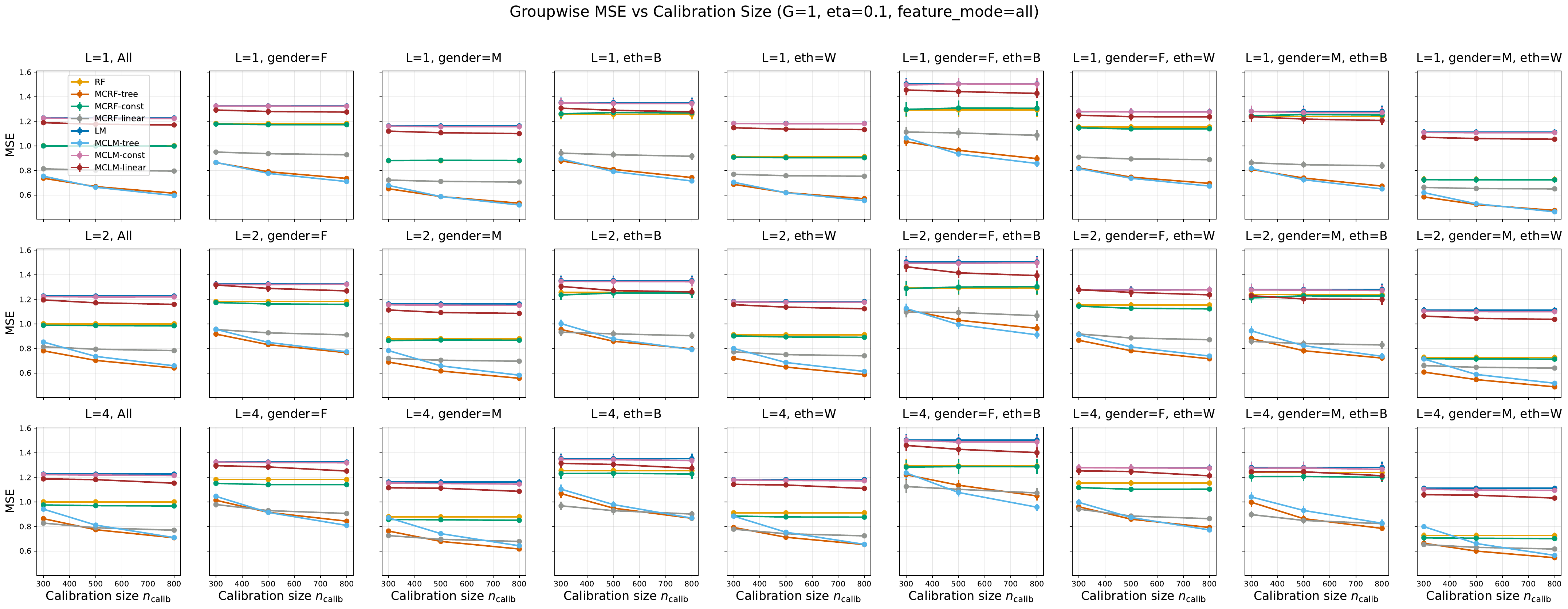}
    \end{subfigure}
    \begin{subfigure}[b]{\textwidth}
        \includegraphics[width=\textwidth]{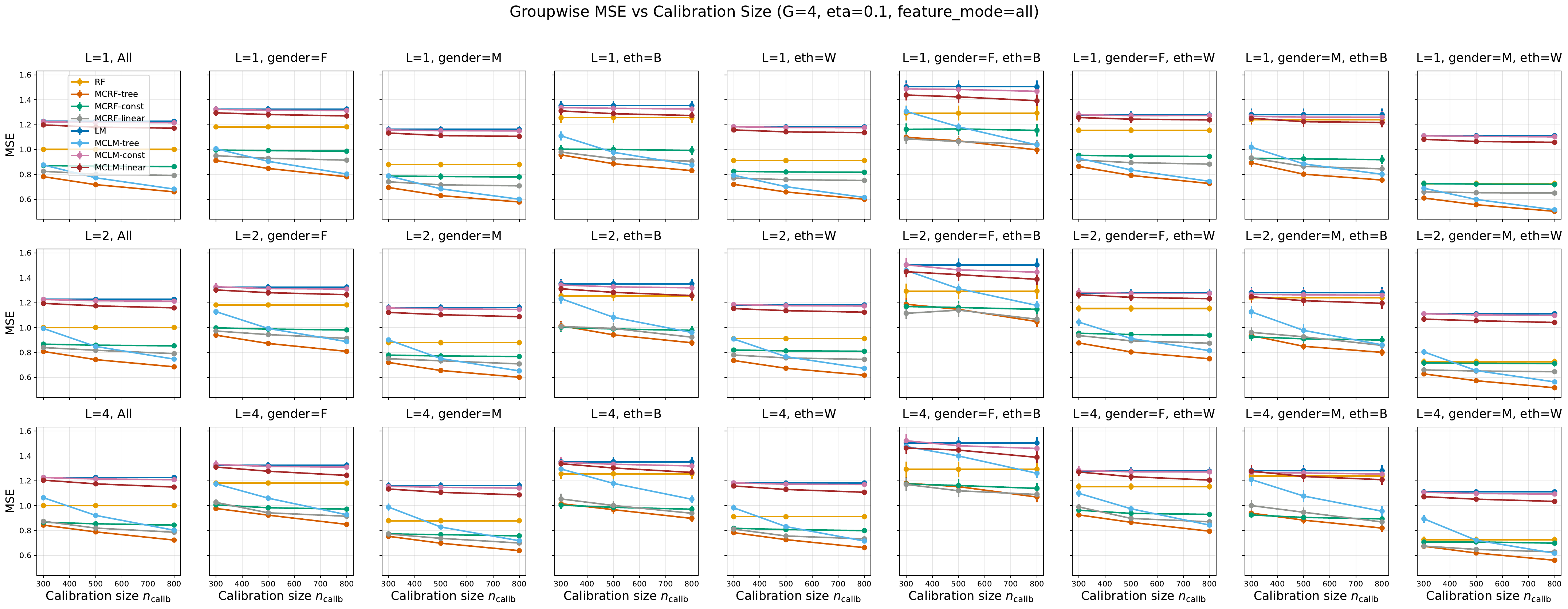}
    \end{subfigure}
    \begin{subfigure}[b]{\textwidth}
        \includegraphics[width=\textwidth]{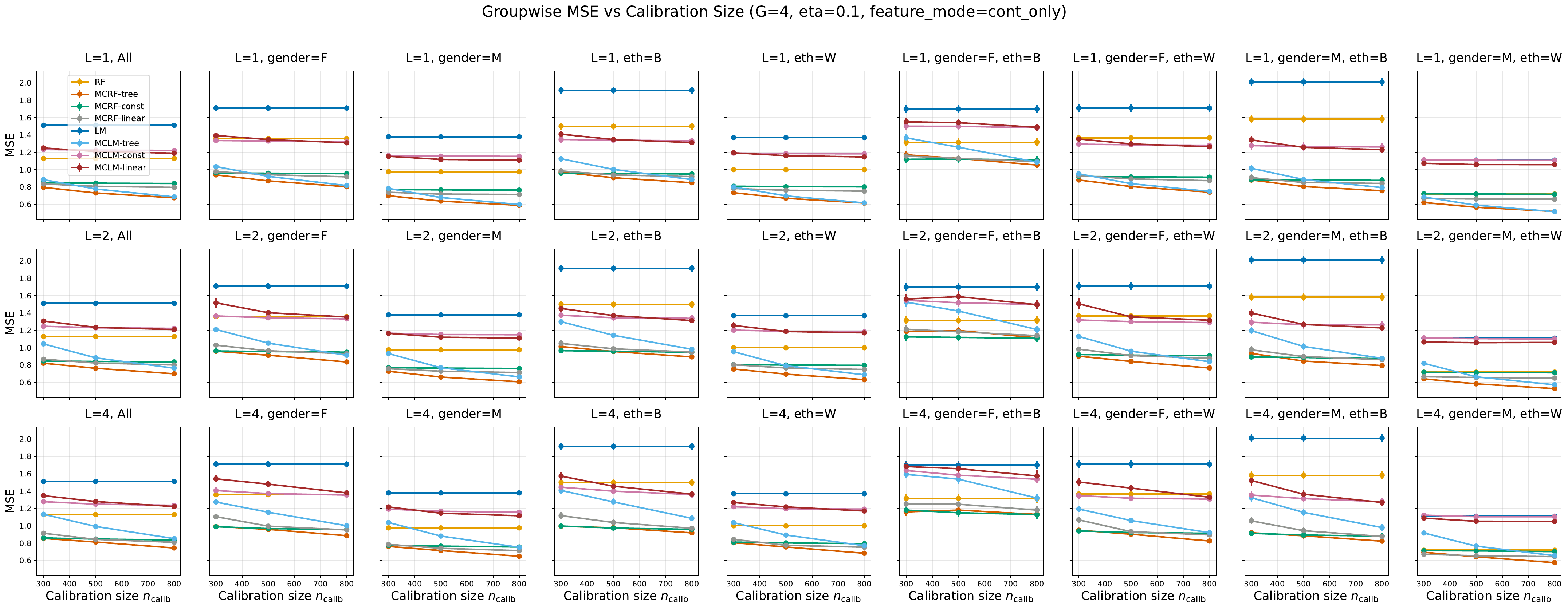}
    \end{subfigure}
    \caption{Mean squared errors across subpopulations  
    vs. calibration size for different initial predictors (linear model and random forest) and auditors (decision tree, constant, and linear). 
    Top:
    bucket-based calibration without group partition ($|\mathcal{G}| = 1$). 
    Middle: calibration with both bucket and group partition ($|\mathcal{G}| = 4$).  
    Bottom: no group partition and categorical features excluded from model fitting. 
    }
    \label{fig:mse_vs_calib}
\end{figure}
The first message is that the auditor's expressivity is decisive. Figure~\ref{fig:bias_vs_calib}, top picture, shows that without explicit group splitting ($|\mathcal{G}|=1$), constant auditors only deliver piecewise-constant bucket corrections. Such updates merely shift the predictor within each bucket and cannot capture residual structure in the feature space, so the gains in both calibration and MSE are small. By contrast, tree-based and linear auditors exploit covariate information to approximate conditional residuals, leading to much larger reductions in both bias and prediction error. 
This pattern is broadly consistent with the theory: more expressive auditors provide richer correction directions, although the practical gains still depend on how well those directions align with the residual structure. In particular, linear refinement, while still finite-dimensional, can be substantially more effective than constant correction.

Figure~\ref{fig:bias_vs_calib}, middle picture, shows that introducing group partition ($|\mathcal{G}|=4$) makes even simple constant auditors more effective because group--bucket interactions directly target subgroup offsets. The resulting gains in calibration are noticeable, but MSE improvements remain limited (middle picture of Figure~\ref{fig:mse_vs_calib}) because constant adjustments can only perform coarse mean shifts. More flexible auditors, such as trees, continue to provide great improvements by capturing heterogeneous structures beyond group-level offsets.
These findings are further supported by the excess-risk trajectories in Appendix Figures~\ref{fig:excess_vs_stepsize-linear-constant},~\ref{fig:excess_vs_stepsize-rf-constant},~\ref{fig:excess_vs_stepsize-lm-tree}, and~\ref{fig:excess_vs_stepsize-rf-tree}. When $|\mathcal{G}|=1$, constant auditors tend to increase the global excess risk after only a few iterations for both linear and random forest initial learners, whereas tree-based auditors or explicit group partitioning ($|\mathcal{G}|=4$) generally lead to stable or decreasing excess risk.

Increasing the number of prediction buckets $L$ alone has a relatively small effect on groupwise performance. The mean bias and MSE curves are nearly unchanged across bucket choices, which suggests that finer prediction partitions are not by themselves the main source of improvement. The dominant factors are the expressive power of the auditor and whether the relevant group information is made available to it.

Another phenomenon concerns access to group information during audit. When there is no explicit group split ($|\mathcal{G}| = 1$) and categorical variables are excluded from the residual fit (Figure~\ref{fig:bias_mse_vs_calib_nogroup}), all methods show negligible improvement in mean calibration bias and only limited MSE gains. Auditors simply cannot distinguish latent subpopulations; without group-related features, the correction term is effectively homogeneous and cannot repair systematic group-specific discrepancies. Even flexible tree auditors fail in this regime because their partitions are built only from the observed covariates.  In terms of MSE, the linear model shows worst error for all settings. Tree-based auditors likely reduce MSE the most, and constant correction does not improve MSE under $|\mathcal{G}| = 1$. 

By contrast, when samples are explicitly partitioned by group ($|\mathcal{G}|=4$) and $X^{(d)}$ are included in the audit step (middle pictures of Figures~\ref{fig:bias_vs_calib} and~\ref{fig:mse_vs_calib}), most methods reduce both MSE and bias. The main exception is constant refinement applied to an already well-specified linear model: if the initial predictor already satisfies the relevant linear moment conditions with the group features included, there is little left for a constant correction to do. However, when $X^{(d)}$ is omitted from the initial fit, the same groupwise constant corrections become effective, which highlights the practical importance of injecting group structure either into the initial predictor or into the audit stage.

If group partitions are imposed but $X^{(d)}$ is still excluded during residual fitting (Figure~\ref{fig:bias_vs_calib}(c)), performance deteriorates again, especially for the calibration bias of the linear initial predictor. In that case, the correction can only enforce calibration along the observed feature space, for example through relations such as $\Expect[X^{(c)\top}(Y-\wh f)] = 0$, and cannot fully resolve hidden group-level misspecification.
\begin{figure}[ht!]
    \centering
    \begin{subfigure}[b]{0.49\textwidth}
    \includegraphics[width=\linewidth]{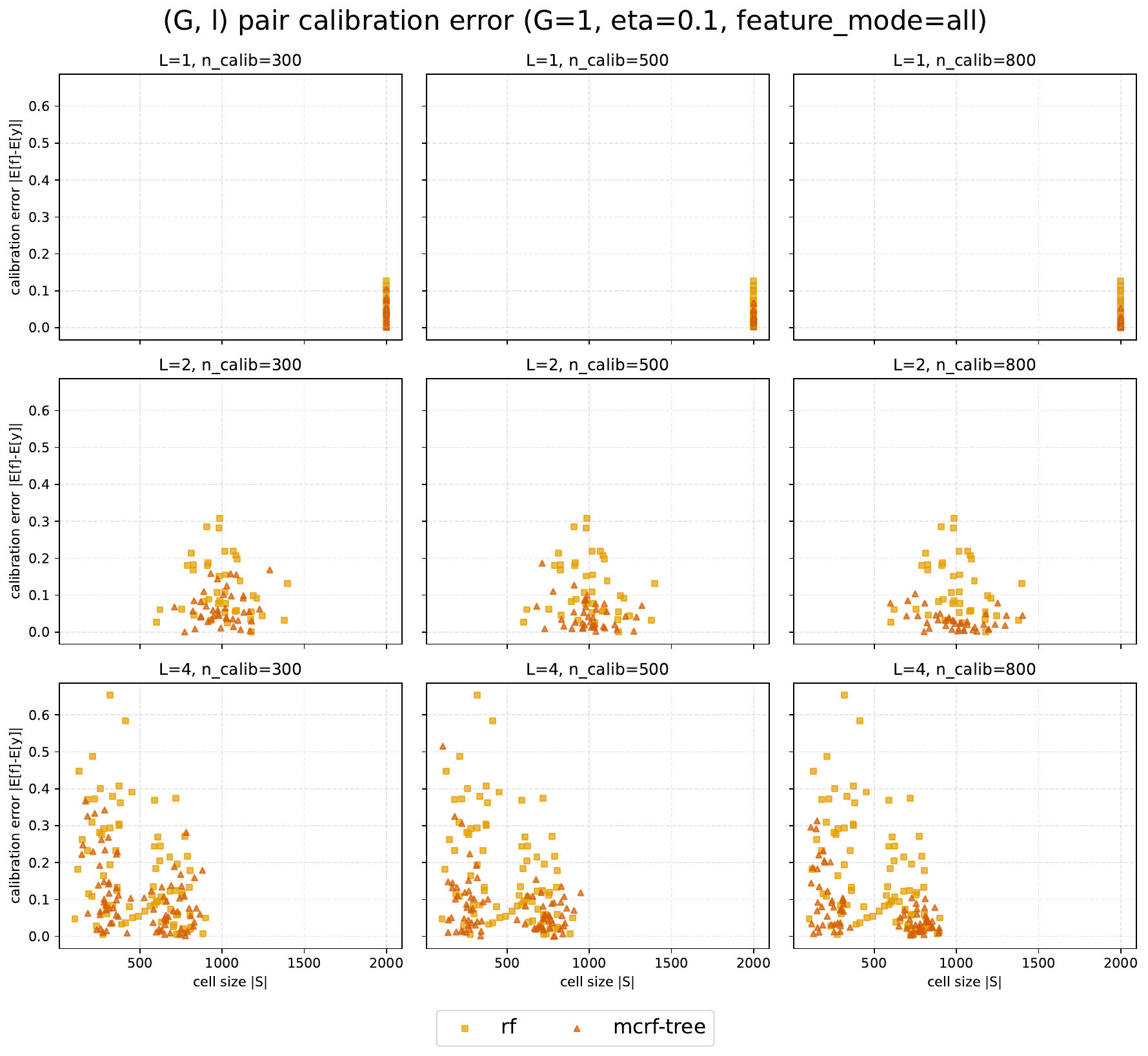}  
    \caption*{(a)}
    \end{subfigure}
    \begin{subfigure}[b]{0.49\textwidth}
    \includegraphics[width = \linewidth]{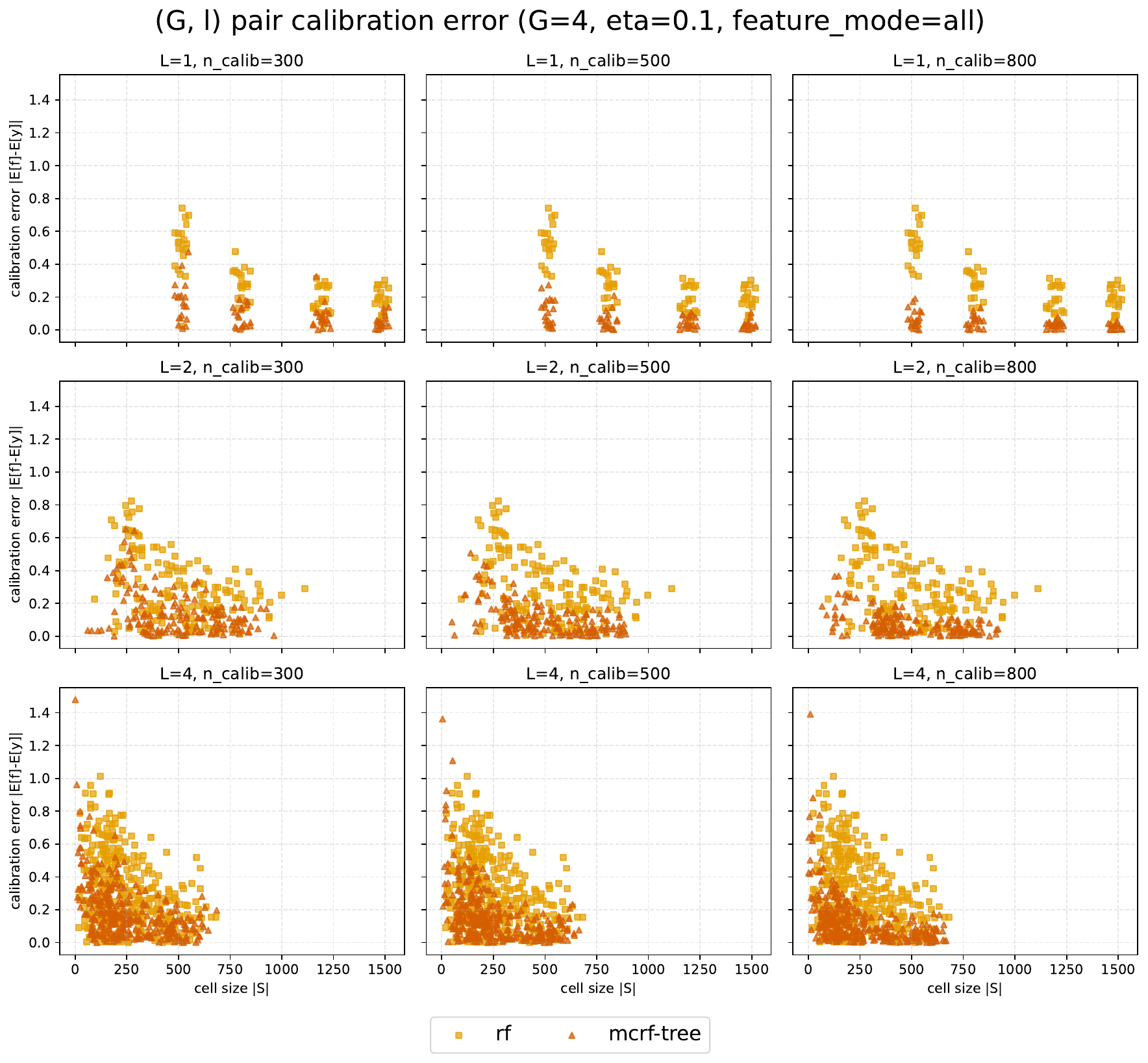}    
    \caption*{(b)}
    \end{subfigure}
    \caption{Scatter plot for empirical calibration error of $(G, l)$ pairs. Left: MCBoost without group partition. Right: MCboost with group partition. }
    \label{fig:calib-error-scatter}
\end{figure}
Beyond the group level, we also examine calibration at a finer resolution by considering group--prediction-level pairs. We apply MCBoost to an initial random forest predictor in Figure~\ref{fig:calib-error-scatter}, using a decision tree regressor as the auditor for iterative refinement. For each 
$(G, l)$, we evaluate the conditional calibration error  
$ \Expect[Y - f(X) | f(X)\in I_l , X \in G],  $
estimated empirically within each bin,
where the prediction bins $I_l$ partition the range of $f(X)$ into $L$ equal-length intervals. Because bucket boundaries depend on the fitted predictor and therefore vary across replications, the calibration errors are aggregated over repeated runs. In both settings ($|\mathcal{G}|=1$ and $|\mathcal{G}|=4$), MCBoost substantially reduces these conditional calibration errors, confirming that the method effectively corrects the systematic bias within the slices defined jointly by the group membership and the prediction level.

We also evaluate MCBoost for quantile regression under pinball loss, again with cross-validation used to determine the effective step-size budget. The qualitative picture is similar to the mean-regression case; Appendix Figures~\ref{fig:qr-excess_vs_stepsize-rf-tree} and~\ref{fig:qr-excess_vs_stepsize-rf-constant} plot the excess pinball los against the cumulative step size. In particular, improvements are more pronounced in underrepresented subpopulations (Figure~\ref{fig:qr-coverage_vs_calib}), and the coverage of multicalibrated predictors is more tightly concentrated near the nominal level $0.9$ (Figure~\ref{fig:qr-coverage-scatter}). This suggests that the calibration mechanism extends beyond the squared loss and remains effective for distributional targets such as conditional quantiles.

\subsection{MCBoost under distribution shift}\label{subsec:simu-shift} 
\begin{figure}[ht!]
    \centering
    \includegraphics[width=\linewidth]{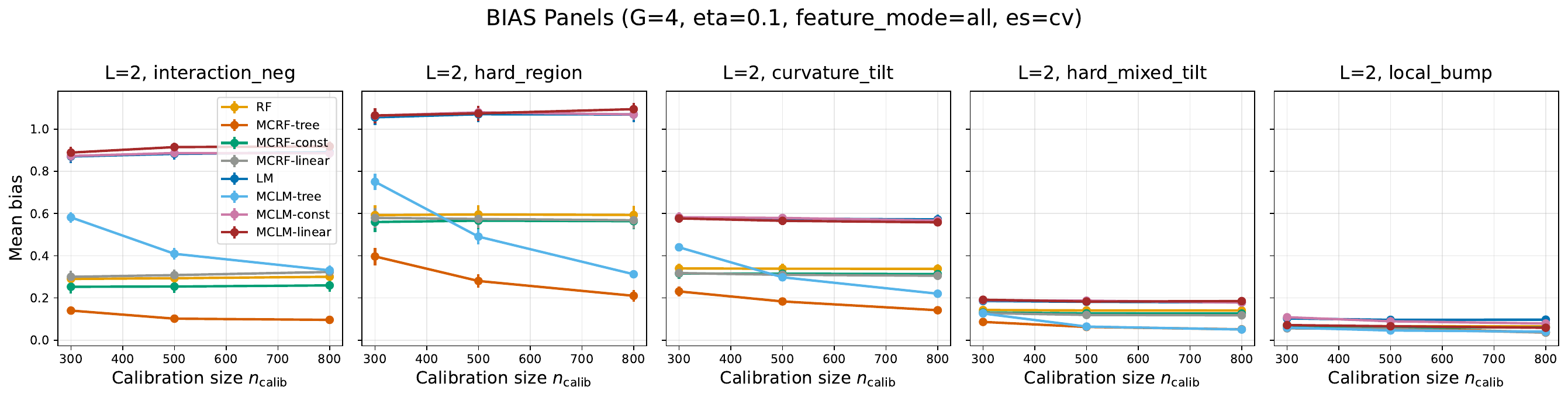}
    \includegraphics[width=\linewidth]{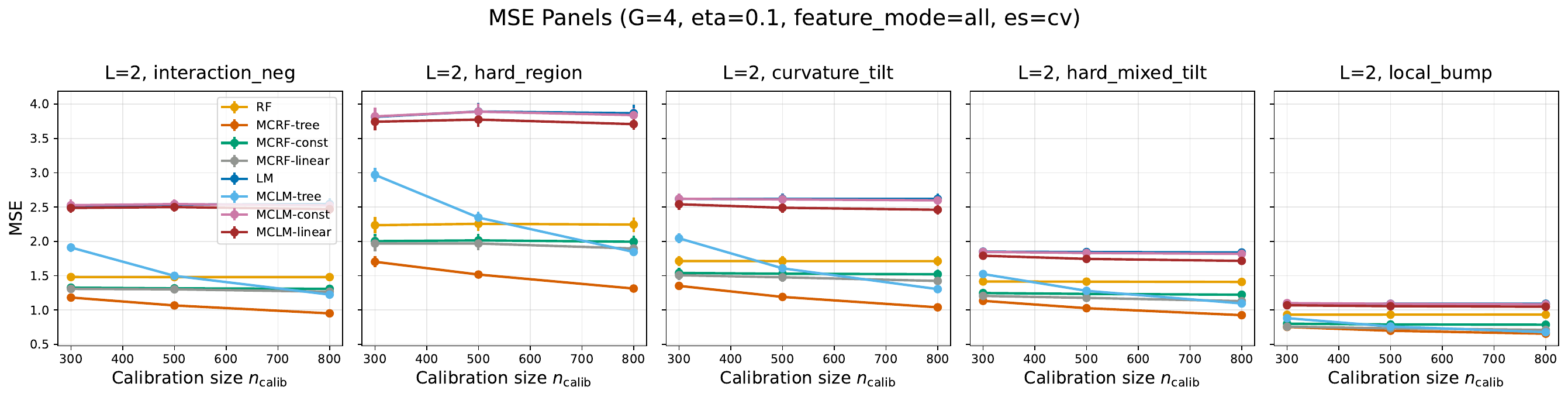}
    \caption{ Mean bias (top) and mean squared error (bottom) on selected structural subgroups and weighted shifts versus calibration sample size, for different initial predictors (linear model and random forest) and auditors (decision tree, constant, and linear).
    }
    \label{fig:calib-structure}
\end{figure}
We further investigate MCBoost under covariate shift. This experiment complements Theorem~\ref{thm:universal}: the theory predicts gains when the target shift is well approximated by the learned multicalibration class. The demographic groups considered above can already be interpreted as special cases of covariate shift through localized weight functions of the form
$ w(X)=\v1\{X\in G\}/\bbP(G). $ Our goal here is to move beyond prespecified demographic groups and examine whether MCBoost remains effective on nonlinear structured subgroups and weighted shifts that are not explicitly hard-coded in advance. We evaluate mean bias and MSE in a collection of nonlinear structural subgroups and weighted shifts constructed from standardized covariates $Z=(Z_1,Z_2)$, corresponding to $(X^{(c)}_1, X^{(c)}_2)$; full constructions are given in the Supplementary Section~\ref{apxsubsec:shift-construct}.

These targets are intentionally more structured than the demographic groups used earlier. They are not aligned with simple linear partitions and therefore provide a sharper test of the effectiveness of the chosen auditor. Figure~\ref{fig:calib-structure} shows that the benefits of MCBoost persist on these structured targets but depend on auditor's expressivity. Constant and linear auditors show limited improvement in mean bias and MSE, yet tree-based auditors improve both metrics much more substantially. 
This is broadly consistent with the theoretical picture: more expressive auditors induce a richer cumulative boosting span, enabling better approximation of nonlinear subgroup structure and density shifts.

The practical implication is similar to the work on conditional guarantee and helps to answer the question of how to choose the auditor in practice. When target groups or probable forms of shift are known in advance, the audit class $\sHa$ should reflect that structure as closely as possible. However, when little is known \emph{a priori}, more expressive auditors, such as trees, may provide a safer default because they are better able to adapt to nonlinear and interaction-driven discrepancies.

\section{Discussion}
This work provides a thorough investigation of the mechanisms through which MCBoost achieves multicalibration. Viewed as an extension of classical boosting, MCBoost inherits many familiar principles: the weak learner determines the approximation power, while early stopping helps control overfitting. 
Our theory clarifies the functional class that can ultimately be calibrated, establishes convergence rates under different smoothness conditions, and gives a finite-sample guarantee together with a valid termination rule, under relatively mild assumptions.  

Our numerical results illustrate the motivation for multicalibration and the trade-off between optimal global predictive accuracy and improved calibration. We further evaluate the performance of MCBoost across different auditors, group partitions, and feature sets, considering hard-coded demographic groups, nonlinear data-driven subpopulations, and structured weighted shifts. 
Taken together, these results provide a clearer understanding  of the MCBoost framework. 
We hope that this article can serve as a useful guide for practitioners seeking predictive models that are both reliable and transferable in practice.

More broadly, Algorithm~\ref{alg:MCboost-general} is not limited to the $|\mathcal{G}| \times L$ partitions studied here. The auditing step can be based on richer, data-adaptive partitions of the covariate-prediction space, potentially targeting fairness across more complex regions. We leave an investigation of such extensions to future work. 
Our framework currently treats univariate predictors $f: \mathcal{X} \to \bbR$. Extending the analysis to multivariate prediction problems with $f(X) \in \bbR^d$, or to scores that depend on additional objects such as the initial predictor $f^{(0)}$ or a threshold function $\lambda(X)$~\citep{zhang2024fair}, is a natural next step. These extensions are conceptually straightforward and can be handled with minor modifications.



\newpage

\appendix

\section{Extended Bregman divergence with subgradient}\label{sec:general-Bregman}
The standard distance generating function~(expected loss function $\sL$ in our article), which induces the Bregman divergence, needs to be differentiable and strictly (or strongly) convex functions, such as the squared loss. However, many objectives of interest, such as the pinball loss, are merely convex, not strictly convex, and are non-differentiable globally. To extend our analysis, we introduce a generalized Bregman divergence based on subgradients. 

\medskip 
Recall that the inner product $\langle \cdot, \cdot \rangle$ denotes the expectation over $L_2(\bbP_X)$, sometimes equivalently viewed as the expectation under the joint distribution of $(Y, X)$. 

The pinball loss $L_{\tau}(Y, q_{\tau}(X))$ admits a Lipschitz subgradient, so by Lemma~\ref{lem:Lipshitz} it satisfies 
\begin{equation}\label{eqn:qt-sub}
      L_{\tau}(Y, q_{\tau}(X)  ) - L_{\tau}(Y, \wt{q}_{\tau}(X) )
     \ge \zeta(Y, X) (q_{\tau}(X) - \wt{q}_{\tau}(X)),  
\end{equation}  
where 
$$
  \zeta(Y, X)  = \begin{cases}
   1 - \tau,  & Y < \wt{q}_{\tau}(X),  \\
   [-\tau, 1-\tau], & Y = \wt{q}(X), \\
   -\tau, &  Y > \wt{q}(X).
    \end{cases}
$$
In other words, 
$$
\sL(q_{\tau}) - \sL(\wt{q}_{\tau}) \ge \langle \zeta(Y, X; \wt{q}_{\tau}), q_{\tau}(X) - \wt{q}_{\tau}(X)\rangle. 
$$ 

Since $\zeta$ depends on both $X$ and $Y$, we formally define the subderivatives of $\sL$ as 
$$
\partial \sL(f) := \{\zeta(f): \mathcal{X} \rightarrow \bbR\}, 
$$
such that for any pairs of functions $f, \wt{f}$, 
$$
\sL(f) - \sL(\wt{f}) \ge \langle \zeta(\wt{f}), f - \wt{f}\rangle, \quad \zeta(\wt{f}) \in \partial \sL(\wt{f}).  
$$
For the pinball loss, the conditional subgradient is explicitly 
$$
\zeta(\wt{q}_{\tau}) \in \left\{ (1-\tau)\bbP(Y < \wt{q}_{\tau}(X) | X) - \tau \bbP(Y > \wt{q}_{\tau}(X) | X) + [-\tau, 1-\tau] \cdot \bbP(Y = \wt{q}_{\tau}(X) | X) \right\}.
$$
When $Y|X$ is continuous almost surely, the last term vanishes and $\partial \sL(f)$ reduces to a singleton, recovering the usual differentiable case. 

If $\sL$ is smooth, then $\partial \sL(f) = \{\nabla \sL(f)\}$, where $\nabla \sL(f)$ corresponds to the Gateaux derivative~\citep{zhang2024fair} satisfying 
$$
\langle \nabla \sL(f), h \rangle = \frac{\partial}{\partial \epsilon } \sL(f + \epsilon h) |_{\epsilon = 0},  \quad \forall h.
$$ 
For quantile regression, $s_{\tau}(Y, q_{\tau} ) = \v1\{Y \le q_{\tau}\} - \tau $ is indeed one of the subderivatives we choose for $L_{\tau}$, it could also be $\v1\{Y < q_{\tau}\} - \tau $. These two do not make distinction when $Y | X$ is continuous and $\partial \sL$ reduces to a singleton. 

\paragraph{Generalized Bregman divergence.} 
Regarding all these, we define the generalized Bregman divergence as 
$$
D_{\sL}(f_1, f_2; \zeta) = \sL(f_1) - \sL(f_2) - \langle \zeta(f_2), f_1 - f_2\rangle, \quad \zeta(f_2) \in \partial \sL(f_2). 
$$ 
The divergence depends on the specific choice of subgradient $\zeta$. 

Algorithmically, given Algorithmically,  given the form of $s_{\tau}(Y, q_{\tau})$ in Subsection~\ref{subsubsec:quantile}, we actually fit the ``residual" $s_{\tau}(Y, q_{\tau}^{(b)})$ to approximate the subgradient
\begin{align*}
    g^{(b)} & := \Expect[s_{\tau}(Y, q_{\tau}^{(b)}(X))|X] \\
    & =  (1-\tau)\bbP(Y \le \wt{q}_{\tau}(X) | X) - \tau \bbP(Y > \wt{q}_{\tau}(X) | X)  \\
    & =: \zeta(q_{\tau}^{(b)})(X) 
\end{align*} 
at iteration $b$. In the boosting procedure, we usually fix a consistent subgradient form during updates and do not very the subgradient form across iterations. When $Y|X$ is continuous a.s., both the population minimizer $q_{\tau}^*(X)$ and the boosting limit $q_{\tau}(X)$ are unique; otherwise, multiple minimizers (e.g., any median between two discrete points) may exist. 

The generalized Bregman divergence retains key properties of its classical counterpart: 
\begin{itemize}
    \item Non-negativity. $D_{\sL}(f_1, f_2; \zeta(f_2)) \ge 0$ for any $\zeta \in \partial \sL$, a direct consequence of convexity and definition of subgradient. 
    \item Generalized Pythagorean inequality. If 
    $$
    p \in \argmin_{f \in C} D_{\sL} (f, f_1; \zeta(f_1) ), \quad \zeta(f_1) \in \partial \sL(f_1) 
    $$
    for a closed convex $C$, then there exists $\chi(p) \in \partial \sL(p) $ such that 
    $$
    D_{\sL} (f, f_1; \zeta(f_1) ) \ge D_{\sL} (p, f_1; \zeta(f_1) ) + D_{\sL}(f, p; \chi(p) ), \quad \forall f \in C. 
    $$
    This follows from the optimality condition
    \begin{align*}
        & 0 \in \partial (\sL + \iota_C) (p) - \zeta(f_1), \quad  \iota_{C}(f) = \begin{cases}
        0, & \text{ if } f \in C, \\
        \infty, & \text{ if } f \notin C, 
    \end{cases}  
    \end{align*} 
    and the relation 
    $$
     \partial ( \sL + \iota_C) (p) = \partial \sL(p) + N_C(p), 
    $$
    where 
    $ N_C (p) = \{ \chi(p): \langle \chi(p), f - p \rangle \le 0, \forall f \in C \} $ denotes the normal cone of $C$. 
    Then, there exists $\chi_p$ such that  
    $ \zeta_1 - \chi_p \in \partial \iota_C (p)$.
    Therefore, 
    \begin{align*}
        D_{\sL} (f, f_1; \zeta(f_1) ) - D_{\sL} (p, f_1; \zeta(f_1) ) & = \sL(f) - \sL(p) - \langle \zeta(f_1), f - p\rangle \\
        & \ge  \sL(f) - \sL(p) - \langle \chi(p), f - p\rangle, \forall f \in C, 
    \end{align*}
    since $\langle \zeta(f_1) - \chi(p), f - p\rangle \le 0, \forall f \in C $. 
\end{itemize}
Note that $D_{\sL}(f_1, f_2; \zeta(f_2)) = 0$ does not necessarily imply $f_1 = f_2$. Thus, the minimizer of the generalized divergence may not be unique unless $\sL$ is strictly convex. 

Following the same reasoning, we can show that the boosting limit satisfies 
$$
f^{(\infty)} \in \argmin_{f \in C} \sL(f), \quad C = f^{(0)} + \sF,  
$$
and $\|g^{(\infty)}\|_* = 0$, where 
$$
g^{(\infty)} = \Expect[s(Y, f^{(\infty)}(X)) | X]. 
$$
For the pinball loss,
$$
g^{(\infty)} =  (1-\tau)\bbP(Y \le q_{\tau}^{(\infty)}(X) | X) - \tau \bbP(Y > q_{\tau}^{(\infty)}(X) | X).
$$ 
Equivalently, the minimization problem can be written as 
\begin{align*}
    f^{(\infty)} \in \argmin_{f \in C} \sL(f) = \{ \argmin_{f} L(f) + \iota_{C}(f)\}. 
\end{align*} 
From the perspective of generalized Bregman divergence, this is the same as 
$$
f^{(\infty)} \in \argmin_{f \in C} D_{\sL}(f, f^*; \zeta ) |_{\zeta = 0} = \argmin_{f \in C}  \{ \sL(f) - \langle \zeta, f\rangle \} |_{\zeta = 0}.  
$$ 
For any minimizer $\wt{f}^{(\infty)} \in \argmin_{f \in C} \sL(f)$, the first-order optimality condition for convex functions implies 
$$
0 \in \partial \sL( \wt{f}^{(\infty)}) + N_C ( \wt{f}^{(\infty)}), 
$$
or equivalently, there exits $ \zeta(\wt{f}^{(\infty)}) \in \partial \sL(\wt{f}^{(\infty)})$ such that 
$$
\langle -\zeta(\wt{f}^{(\infty)}) , f - \wt{f}^{(\infty)} \rangle \le 0, f \in C. 
$$ 
In particular, for our solution $f^{(\infty)}$, 
$$
\langle -g^{(\infty)}, f - f^{(\infty)} \rangle \le 0, f \in C.  
$$ 
When $C = f^{(0)} + \sF$, we have $\langle g^{(\infty)}, h\rangle = 0, \forall h \in \sF$.   
Having this is no stronger than the general optimality condition. The uniqueness is by strictly convexity, no-atom at the $\tau$-th quantile for pinball loss, or the use of regularizing like minimizing $\sL(f) + \lambda \|f\|_{L_2}^2, \lambda > 0$. Consequently, it will select the one with minimal norm as $\lambda \to 0$.

\section{Proof for multicalibration} 

\subsection{Proof of Lemma~\ref{lem:Lipshitz} }
\begin{proof}
\begin{align*}
     L(f, Y) - L(\wt{f}, Y) & = \int_{\wt{f}}^f s(u, Y) du = s(f, Y )(f - \wt{f}) + \int_{\wt{f}}^f  s( u, Y) - s(f, Y) du  \\
    & \ge s(f, Y )(f - \wt{f}) - \int_{\wt{f}}^f 2c_{L} (f - u ) du \\  
    & = s(f, Y )(f - \wt{f}) - c_{L}(f - \wt{f})^2, 
\end{align*}
Both the squared loss and pinball loss satisfies this property. 
\end{proof} 

\subsection{Proof of Theorem~\ref{thm:stationarity}} \label{apxsubsec:stationarity}

\begin{proof}
We have
\begin{equation}\label{eqn:boost-obj}
    \begin{aligned}
    & \mathcal{L}^{(b)} - \mathcal{L}^{(b+1)}  \\
    & \stackrel{(i)}{\ge} \Expect[L(Y, f^{(b)}(X))] - \Expect[L(Y, f^{(b)}(X) - \eta^{(b)} h^{(b)}(X)) ] \\
    &  \ge  \Expect[s(Y, f^{(b)}(X)) \eta^{(b)} h^{(b)}(X)  ] - c_{L} \Expect \eta^{(b) 2} h^{(b) 2}(X). 
    \end{aligned}
\end{equation} 
Here, we consider adaptive step size and let 
$$
\eta^{(b)} = \frac{\Expect[s(Y, f^{(b)}(X)) h^{(b)}(X) ] }{2c_{L} \|h^{(b)}\|_{L_2}^2 }.
$$
We then have
$$
\eqref{eqn:boost-obj}  \ge \frac{\Expect[s(Y, f^{(b)}(X)) h^{(b)}(X) ]^2 }{4 c_{L} \Expect [h^{(b)2}(X)] } \ge \frac{\kappa^2}{4 c_L}\|g^{(b)}\|_*^2.  
$$
Therefore, $\{\mathcal{L}^{(b)}\}$ is a monotone bounded sequence that converges to a limit $\mathcal{L}^{(\infty)}$. Under regularity condition (Lemma~\ref{lem:limit-fun} below), 
$$
\mathcal{L}^{(\infty)} = \lim_{b \rightarrow \infty} \Expect[L(Y, f^{(b)}(X))] { = }  \Expect[ \lim_{b \rightarrow \infty} L(Y, f^{(b)}(X))] = \Expect[  L(Y, f^{(\infty)}(X))],
$$
for some $f^{(\infty)}$ such that 
$\sL^{(\infty)} = \Expect[L(Y, f^{(\infty)}(X))]$, we claim  
$$
\Expect [s(Y, f^{(\infty)}(X)) h(X)] =  \langle s(Y, f^{(\infty)}), h \rangle = 0
$$
for all $h \in \sHa$ (thus $h \in \sF$). If not, we have 
$\wt{h} \neq 0 $ such that $\langle s(Y, f^{(\infty)}), \wt{h} \rangle > \epsilon > 0$ and
\begin{align*}
    & \Expect[L(Y, f^{(\infty)}(X))] - \Expect[L(Y, f^{(\infty)}(X) - \wt{\eta} \wt{h}(X)) ] \\
    & \ge \frac{\Expect[s(Y, f^{(\infty)}(X)) \wt{h} (X) ]^2 }{2c_{L} \Expect [ \wt{h}^2 (X)] } \\
    & > \frac{\epsilon^2 }{2c_{L} \Expect [ \wt{h}^2 (X)] } > 0,  
\end{align*}
contradicting minimality within the space $ f^{(0)} + \sF $ 
Therefore, the limit satisfies an ``orthogonality'' condition:
$$
\langle s(Y,f^{(\infty)}),h\rangle = 0, \forall  h\in\mathcal F 
\quad\Longleftrightarrow\quad
\|g^{(\infty)}\|_{*} := \sup_{h \in \sF, \|h\|_{L_2} \le 1}\langle g^{(\infty)},h\rangle = 0
\quad . 
$$
This identity also justifies the basic decomposition
$$
\mathbb E[h(X)\,s(Y,f^{(B)})]
=\mathbb E\!\big[h\{s(Y,f^{(B)})-s(Y,f^{(\infty)})\}\big]
=\langle g^{(B)}-g^{(\infty)},\,h\rangle.
$$

Now suppose $f^{*} = \argmin_{f} \Expect L$, which satisfies global optimality condition
$$
\langle s(Y, f^* ), h \rangle = 0, \quad \forall h \in L_2(\bbP_X).
$$
In the special case of the squared loss, the orthogonality $f^{(\infty)}\in  f^{(0)}+\mathcal F$ implies that $f^{(\infty)}$ is the $L_2$-projection of $f^*$ onto the space $ f^{(0)}+\mathcal F$. In particular, if $f^{(0)}\in\mathcal F$, then 
$$
f^{(\infty)} = \mathcal P_{\mathcal F  } f^*, \qquad 
\mathcal L(f^{(\infty)}) = \inf_{f\in\mathcal F  } \mathcal L(f).
$$
Indeed, by Lemma~\ref{lem:bauschke}, 
$$
f^{(\infty)} - f^{(0)} = \mathcal P_{\mathcal F}(f^*-f^{(0)})
\Rightarrow
f^{(\infty)} = f^{(0)} + \mathcal P_{\mathcal F}(f^*-f^{(0)})
= \mathcal P_{f^{(0)}+\mathcal F} f^*.
$$
For general convex losses, the relevant notion of projection need not be Euclidean. Instead, it is more natural to express it in terms of Bregman divergence. Recall the generalized Bregman projection we mentioned above
\begin{align*}
  &   \mathcal{P}_{  f^{(0)} + \sF }^{\sL} (f_1; \zeta ) = \argmin_{ f \in  f^{(0)} + \sF } D_{\sL}(f, f_1; \zeta),   
\end{align*}  
where 
$$
D_{\mathcal L}(f,f_1; \zeta )
= \mathcal L(f)-\mathcal L(f_1)-\langle \zeta, f -f_1 \rangle, \zeta \in \partial \sL(f_1)
$$
Following~\citep{bauschke1997legendre} and our derivations, the Bregman projection satisfies the variational inequality 
$$
\langle \zeta - \chi(p),
\, f - p \rangle \le 0,
\qquad \forall f\in  f^{(0)}+\mathcal F,
$$ 
for some $\chi(p) \in \partial \sL(p)$ and any projection point $p \in \mathcal{P}_{  f^{(0)} + \sF }^{\sL} (f_1; \zeta )$.
Combined with the Bregman Pythagorean theorem, the inequality ensures the minimal distance: 
$$
D_{\mathcal L}( p, f_1; \zeta )
\;\le\; D_{\mathcal L}(f,f_1; \zeta) - D_{\mathcal L}(f, p; \chi(p)),
\quad \forall f\in  f^{(0)}+\mathcal F,
$$ 
for some $\chi(p) \in \partial \sL(p)$. Applying this with $f_1 = f^*$ and $\zeta = 0$ implies that the limit $f^{(\infty)}$ minimizes $\sL(f)$ within $f^{(0)} + \sF$ and satisfies 
\begin{equation}\label{eqn:variational-ineq}
    \langle -\chi(f^{(\infty)}),\, f-f^{(\infty)}\rangle \le 0,
\qquad \forall f\in f^{(0)}+\mathcal F,
\end{equation}
for some $\chi(f^{(\infty)})$.  Our setting ensures that $\chi(f^{(\infty)}) := g^{(\infty)}$ satisfies~\eqref{eqn:variational-ineq}. 
Since $f^{(0)}+\mathcal F$ is an affine space, we may take $f=f^{(\infty)}+th$ with $h\in\mathcal F$ and $t\in\mathbb R$. Varying $t$ in both directions forces 
$
\langle g^{(\infty)},\, h\rangle = 0, \forall h\in\mathcal F,
$
which recovers the stationary condition in the general convex (possibly non-smooth) case.
\end{proof}

\begin{lemma}[Proposition 3.19 of~\cite{bauschke2020correction}]\label{lem:bauschke}
        Let $\sF$ be a nonempty closed convex subset of $L_2(\bbP_X)$, then $\mathcal{P}_{f^{(0)} + \sF}(f^*) = f^{(0)} + \mathcal{P}_{\sF}(f^* - f^{(0)}) $. 
\end{lemma}   

To show the existence of the limit $f^{(\infty)}$ and the dominated/uniform integrability condition to swap the limit and expectation, we assume 
\begin{enumerate}
    \item[A1] $L(Y, \cdot)$ is continuous for a.e $Y$, and there exist $p\ge1$, $\epsilon>0$, $c>0$, and $a(Y)\in L_{1+\epsilon}$ such that for all $u \in \mathbb R$, 
    \begin{equation}\label{eqn:poly-growth}
        | L(Y, u) | \le a(Y) + c|u|^p, a(Y) \in L_{1+\epsilon}(\bbP_X), \,\, \sup_b \|f^{(b)}\|_{p(1+\epsilon)} < \infty.   
    \end{equation} 
\end{enumerate}
The polynomial growth and continuity conditions are satisfied by common convex losses. For squared loss, 
$ (Y - u)^2/2 \le Y^2 + u^2, $ corresponding to  $c = 1, p =2 $. 
For pinball loss, 
$
|(Y-u)(\tau - \v1\{Y \le u\})| \le |Y| + |u|,  
$
corresponding to $c = 1, p = 1$. 

In both cases, to guarantee the $L_2$-bounded and $L_2$-convergence of the boosting path, it suffices that the increments $\{ \eta^{(b)} h^{(b)}\}_{b \ge 1}$ 
form a Cauchy series 
\begin{equation}
    \Big\| \sum_{b=m}^n \eta^{(b)} h^{(b)} \Big\|_{L_2} \rightarrow 0, 
\quad \text{as } m, n \to \infty.     
\end{equation}
This condition holds whenever $\sum_{j=1}^{\infty} \eta_j  < \infty$ if $\sup_{h \in \sHa}\|h\|_{L_2} \le C_{\sHa}$, which results in 
$$
\Big\| \sum_{b=m}^n \eta^{(b)} h^{(b)} \Big\|_{L_2} \le  \sum_{b=m}^n \eta^{(b)} \|h^{(b)}\|_{L_2} \le C_{\sHa} \sum_{b=m}^n \eta^{(b)} \rightarrow 0, 
$$   
Particularly, under the step size of Theorem~\ref{thm:smooth-rate},
$$
\frac{\kappa}{2c_L C_{\sH}} \sum_{b=1}^{\infty} \|g^{(b)}\|_* \le  \kappa \sum_{b=1}^{\infty} \frac{\|g^{(b)}\|_* }{2c_L \|h^{(b)}\|_{L_2}} < \sum_{b = 1}^{\infty} \frac{\langle g^{(b)}, h^{(b)} \rangle }{2c_L \|h^{(b)}\|_{L_2}^2} < \infty, 
$$ 
which forces the sequence $\{g^{(b)}\}$ of most-violated directions also forms a Cauchy series. Intuitively, the dual norm $\|g^{(b)}\|_*$ must diminish--either because the predictor itself approaches the optimum, or because no single $h \in \sHa$ can remain well aligned with $g^{(b)}$ as the algorithm progresses.
Moreover, we have 
$$
\sup_b \|f^{(b)}\|_{L_2} \le \|f^{(0)}\|_{L_2} + \sup_b \|\sum_{j=0}^b \eta^{(j)} h^{(j)}\|_{L_2} < \infty,
$$
providing the uniform integrability required in Lemma~\ref{lem:limit-fun}. 

\begin{lemma}\label{lem:limit-fun} 
    Under A1, there exists $f^{(\infty)} \in f^{(0)} + \mathcal{F}$ with $f^{(b)} \to f^{(\infty)}$ in $L_2$ (hence in  probability), and 
    $$
      \lim_{b \rightarrow \infty} \Expect[L(Y, f^{(b)}(X))]  =   \Expect[ L(Y, f^{(\infty)}(X))],   
    $$    
\end{lemma}

\begin{proof}
    Given the step size, we have 
    $f^{(b)} {\to} f^{(\infty)}$ in $L_2$, hence in probability. The continuity condition leads to $L(Y,f^{(b)}(X))\to L(Y,f^{(\infty)}(X))$ in probability. Uniform integrability of $\{L^{(b)}\}$ follows from the de la Vallée–Poussin criterion using \eqref{eqn:poly-growth} and $\sup_b \|f^{(b)}\|_{L_{p(1+\epsilon)}} < \infty$ (implied by $L_2$-boundedness when $p\le 2$, or ensured by construction when $p>2$): with a super linear map $\phi(t) = t^{1+\epsilon}$ such that $$
    \lim_{x \to +\infty} \frac{\phi(t)}{t} = + \infty, $$
    we have 
    $$
    \sup_b \Expect [ \phi(|L(Y, f^{(b)}(X))| ) ] = \sup_b C' \Expect[ |a(Y)|^{1+\epsilon} + |f^{(b)}(X)|^{p(1+\epsilon)}] < \infty. 
    $$
    Vitali's theorem~\ref{lem:Vitali} then gives the claimed exchange of limit and expectation. 
\end{proof} 

\begin{lemma}[Vitali convergence theorem~\citep{bogachev2007measure}]\label{lem:Vitali}
    For a finite measurable space, the following are equivalent: 
    \begin{enumerate}
        \item The sequence $\{L^{(b)}\} \in L_p$ convergence to $L^{(\infty)}$ in measure and $\{L^{(b)}\}$ is uniformly integrable. 
        \item $L \in L_p$, and $L^{(b)}$ converges to $L$ in $L_p$.
    \end{enumerate}
\end{lemma} 

\subsection{Proof of Theorem~\ref{thm:stationarity-proj}}\label{apxsec:projection}

We now discuss the  scenario in which the optional projection $\mathcal{P}_O$ exists.
For instance, $O$ represents the space of bounded functions taking values in $[0, 1]$: 
$$
O := \{f: \mathcal{X} \to [0,1] \}, 
$$ 
which is closed and convex but not symmetric. With respect to the $L_2$ norm, the projection takes the explicit form $\mathcal{P}_O f: X \mapsto \max\{ \min\{f(X) , 1\}, 0 \}$.  
\begin{proof}[Proof of Theorem~\ref{thm:stationarity-proj}]
    Under Condition~\ref{ass:projection} and the weak learner condition~\ref{ass:weak-edge}, the sequence $\{f^{(b)}\}$ remains in $O$, and $\{\sL(f^{(b)})\}$ forms monotonically nonincreasing for $b > B'$, forsufficiently large $B'$, since $\mathcal{P}_O(f^{(b)} - \eta^{(b)} h^{(b)}) = f^{(b)} - \eta^{(b)} h^{(b)} $ by the inactive projection condition; and updating with small steps stay within $O$. Consequently, by the same argument as before, we obtain 
$$
\langle g^{(\infty)}, h\rangle = 0, \quad \forall h \in \sF. 
$$

\paragraph{Restricted optimality.}
A subtle but important observation is that no feasible $\sF$-direction can reduce the loss since. Indeed,
$$
\sL(f^{(\infty)}) - \sL(f) \le \langle g^{(\infty)}, f - f^{(\infty)}\rangle \le 0, \quad \forall f = f^{(\infty)} + th \in O, h \in \sF.
$$
If $\|g^{(\infty)}\|_* \neq 0$, there would be some $h$ and $t > 0$ such that $\sL(f^{(\infty)} + t h )  \le \sL(f^{(\infty)} )$, contradicting stationary. 

Note, however, that $f^{(\infty)}$, though lies in $O$, may differ from 
$$
\mathcal{P}_{O}^{\sL}(f^*) := \argmin_{f \in O} D_{\sL}(f, f^*; 0),
$$
because our search space of our algorithm is $f^{(0)} + \sF$ rather than the entire space. 
Suppose $p \in \mathcal{P}_{O}^{\sL}(f^*)$, then by convex optimality,  
$$
    \langle -\zeta(p), f - p \rangle \le 0, \quad \exists \zeta(p) \in \partial \sL(p), \forall f \in O, 
    $$ 
which represents a different variational inequality from that satisfied by $f^{(\infty)}$.
The two coincide when no projection is present: $O = f^{(0)} + \sF$. 

\paragraph{Restricted variational inequality.}
The limit $f^{(\infty)}$ satisfies a \emph{restricted variational inequality}: 
\begin{equation}\label{eqn:VI}
    \langle g^{(\infty)}, f - f^{(\infty)} \rangle \ge 0, \forall f \in O \text{ with } f - f^{(\infty)} \in \sF,      
\end{equation}
where $g^{(\infty)} = \Expect[s(Y, f(X)) | X]$. 

More generally,~\eqref{eqn:VI} can be expressed in subdifferential form as 
$$
0 \in \mathcal{P}_{\sF}\left( \partial \sL (f^{(\infty)} )  \right), 
$$ 
which ensures
$$
\sL(f) \le \sL(\wt{f}), \quad \forall \wt{f} \in O, \wt{f} -f \in \sF.   
$$
Our predictor guarantees $\langle g^{(\infty)}, h \rangle = 0$ for specific $g^{(\infty)}$. In words, the stagewise limit $f^{(\infty)}$ need only be directionally stationary along $\sF$, in contrast to $\mathcal{P}_O f^*$ which satisfies $ \langle - \partial \sL(\mathcal{P}_O f^*), f -\mathcal{P}_O f^* \rangle \le 0 $. 
This condition is stronger than the variational inequality~(VI) required by $f^{(\infty)}$, reflecting the possible difference between these two -- $f^{(\infty)}$ may have a higher objective value if $\sF$ is not rich enough. 
\end{proof}

\begin{example}
    A concrete example to separate the two estimator would be: 
    $
    \sL(f) = \tfrac{1}{2}\|f - f^* \|_{2}^2
    $
    with $f^* = (1, 0)$, $\sF = \text{span}\{e_2\}$, and $O = [-10, 10]^2$. The minimizer is $\mathcal{P}_{O} f^* = (1, 0)$. For the restricted problem: $g^{(\infty)} = f^{(\infty)} - f^*$; projecting onto $\sF$ yields $(f^{(\infty)})_2 = 0$. This shows any $(f_1, 0)\in O$ is $\sF$-stationary, whereas only $f_1 = 1$ matches the global minimum. 
\end{example}

\begin{rem}
    MCBoost with optional projection can be interpretedas a proximal gradient descent (mirror descent) method with the $L_2$-distance: 
    $$
     f^{(b+1)} =  \mathcal{P}_{O}(f^{(b)} - \eta^{(b)} h^{(b)}) = \argmin_{f \in O} \left\{ \langle \eta^{(b)}h^{(b)}, f\rangle + \frac{1}{2} \|f - f^{(b)}\|_{L_2}^2\right\}.
    $$
    The three-point inequality ensures that 
    $$
    \langle -\eta^{(b)} h^{(b)}, f - f^{(b+1)} \rangle \le -D_{L_2}(f, f^{(b+1)}) + D_{L_2}(f, f^{(b)}) - D_{L_2}(f^{(b+1)}, f^{(b)})
    $$
    Unlike standard projected gradient descent, boosting uses $h^{(b)} \in \sHa$ as an approximation of the true gradient or subgradient $g^{(b)}$. Consequently,  
    \begin{align*}
        \sL(f^{(b+1)}) - \sL(f^{(b)}) & \le  \langle g^{(b)}, f^{(b+1)} - f^{(b)} \rangle  + c_{L} \|f^{(b+1)} - f^{(b)}\|_{L_2}^2   \\
        & \le \langle h^{(b)}, f^{(b+1)} - f^{(b)} \rangle + \langle g^{(b)} - h^{(b)}, f^{(b+1)} - f^{(b)} \rangle + c_{L} \|f^{(b+1)} - f^{(b)}\|_{L_2}^2  \\
        & \le \left(c_L - \frac{1}{\eta^{(b)}} \right) \|f^{(b+1)} - f^{(b)}\|_{L_2}^2 + \langle g^{(b)} - h^{(b)}, f^{(b+1)} - f^{(b)} \rangle. 
    \end{align*}
    We can obtain the same results presented in the proof of the main theorem if there is no $\mathcal{P}_O$ -- each iteration ensures sufficient descent. 
\end{rem}

\section{Proof of Convergence rates}\label{apxsec:convergence-rate} 

\subsection{Proof of Theorem~\ref{thm:smooth-rate}}

\begin{proof}
Define $\Delta_b = \sL(f^{(b)}) - \sL(f^{(\infty)})$, if the score is \(2c_L\)-Lipschitz in its first argument, i.e.
\(|s(Y, u)-s(Y, v)|\le 2c_L|u-v|\).
Then for all $f, \tilde f$ (Definition~\ref{def:Lipschitz-deriv}),
\[
L(Y,f)-L(Y,\tilde f)\ \ge\ s(Y,f)\,(f-\tilde f)\ -\ c_L\,(f-\tilde f)^2.
\]
Consequently, with update \(f^{(b+1)}=f^{(b)}-\eta^{(b)}h^{(b)}\),
\begin{equation}\label{eq:smooth-progress}
\Delta_b-\Delta_{b+1}\ \ge\ \eta^{(b)}\langle g^{(b)},h^{(b)}\rangle\ -\ c_L(\eta^{(b)})^2\|h^{(b)}\|_2^2.
\end{equation}
Choosing the “quadratic-optimal” step
\(
\displaystyle
\eta^{(b)}=\frac{\langle g^{(b)},h^{(b)}\rangle}{2c_L\|h^{(b)}\|_2^2}
\)
gives
\[
\Delta_b-\Delta_{b+1}
\ \ge\ \frac{\langle g^{(b)},h^{(b)}\rangle^2}{4c_L\|h^{(b)}\|_2^2}
\ \stackrel{\text{Ass.~\ref{ass:weak-edge}}}{\ge}\ \frac{\kappa^2}{4c_L}\,\|g^{(b)}\|_*^2.
\]
On the other hand, by convexity and Assumption~\ref{ass:finite-radius},
\[
\Delta_b
\ \le\ \langle g^{(b)},\,f^{(b)}-f^{(\infty)}\rangle
\ \le\ \|g^{(b)}\|_*\,\|f^{(b)}-f^{(\infty)}\|_{L_2}
\ \le\ R\,\|g^{(b)}\|_*.
\]
Therefore,
\[
\Delta_b-\Delta_{b+1}\ \ge\ \frac{\kappa^2}{4c_LR^2}\,\Delta_b^2,
\]
By harmonic descent lemma: 
$$
\Delta_{b+1} \le \Delta_{b} - c \Delta_b^2 
$$ indicates 
$
1/\Delta_{b+1} \ge 1/\Delta_{b} + c.
$
We have 
$$
1/\Delta_{b+1} \ge 1/\Delta_{b} + \tfrac{\kappa^2}{4 c_L R^2}, 
$$
indicating that 
$$
1/\Delta_{B} \ge 1/\Delta_{0} + \tfrac{\kappa^2 B}{4 c_L R^2}. 
$$
which yields the standard harmonic descent:
$$
\Delta_{B}  \le \left[ \frac{1}{\Delta_{0}} + \frac{\kappa B}{4 c_L R^2} \right]^{-1} \le \frac{4c_L R^2}{\kappa^2 } \cdot \frac{1}{B}. 
$$ 

Under the presence of a projection $O$ and the inactive-projection condition~\ref{ass:projection}, the convergence analysis proceeds almost identically -- 
we have 
$$
\frac{1}{\Delta_B} \ge \frac{1}{\Delta_B'} + \frac{\kappa^2 (B - B')}{4 c_L R^2}. 
$$
For large enough $B > B'/2$, 
$$
\Delta_{B}  \le \left[ \frac{1}{\Delta_{B'}} + \frac{\kappa (B - B')}{4 c_L R^2} \right]^{-1} \le \frac{8 c_L R^2}{\kappa^2 } \cdot \frac{1}{B}.  
$$
Hence, the same $O(1/B)$ sublinear convergence rate holds asymptotically.
\end{proof} 

\begin{example}[Square loss specialization]
    Consider the squared loss $ \sL(f)=\tfrac12\,\mathbb E(Y-f(X))^2$. Here $c_L = 1/2$ and the step choice coincides with the exact line search. The population minimizer is $f^*(x)=\Expect[Y\mid X=x]$, and the boosting limit satisfies $f^{(\infty)}=\mathcal P_{f^{(0)}+\mathcal F}f^*$. A direct calculation shows
    \begin{align*}
    \Delta_b & = \frac{1}{2}\langle f^{(\infty)} - f^{(b)}, 2Y - f^{(b)} -  f^{(\infty)}\rangle  \\
    & = \frac{1}{2} \langle f^{(\infty)} - f^{(b)}, 2 \mathcal{P}_{f_0+\sF}Y - f^{(b)} -  f^{(\infty)}\rangle  \\
    & = \frac{1}{2}\| f^{(\infty)} - f^{(b)}\|_{L_2}^2. 
    \end{align*}
    Multicalibration violations can be bounded in terms of $\Delta_B$: 
    \begin{align*}
    \sup_{h\in\mathcal F}\Big|\Expect\!\big[h\{f^{(B)}-f^{(\infty)}\}\big]\Big| 
    & \; \le \; C_{\sF} \|f^{(B)} - f^{(\infty)}\|_* \\
    &\;\le\; C_{\mathcal F}\,\|f^{(B)}-f^{(\infty)}\|_{L_2} \\
    &\;=\; C_{\mathcal F}\sqrt{2\Delta_B} \le C_{\mathcal F}\sqrt{\tfrac{8c_LR^2}{\kappa^2B}}.
    \end{align*}
    Thus choosing 
    $$ B\ \ge\ \dfrac{8 c_L C_{\mathcal F}^2 R^2}{\kappa^2\,\alpha^2} $$
    guarantees regression multicalibration at tolerence $\alpha$. In Remark~\ref{rem:square-loss} of the main context, we show that the excess risk admits faster convergence rate. 
\end{example}

\subsection{Proof of Theorem~\ref{thm:PL-linear}}
\begin{proof}
Under 
\[
\frac12\|g^{(b)}\|_*^2\ \ge\ \mu\,\Delta_b
\qquad(\text{PL on the restricted dual norm}), 
\]
combining with \eqref{eq:smooth-progress} (and the optimal step) gives
\[
\Delta_b-\Delta_{b+1}\ \ge\ \frac{\kappa^2}{4c_L}\|g^{(b)}\|_*^2\ \ge\ \frac{\mu\kappa^2}{2c_L}\,\Delta_b,
\]
hence
\[
\ \Delta_{b+1}\ \le\ \Big(1-\frac{\mu\kappa^2}{2c_L}\Big)\Delta_b,
\]
indicating 
$$
\Delta_{B} \le \left( 1 - \frac{\mu \kappa }{2 c_L} \right)^B \Delta_0, 
$$
i.e., the loss has exponential decay.

With projection $\mathcal{P}_O$, an analogous argument yields linear convergence after projection becomes inactive: 
$$
\Delta_B \le \left( 1 - \frac{\mu \kappa}{2 c_L}\right)^{B - B'} \Delta_{B'} 
$$ 
for some $B' > 0$.
\end{proof} 

\subsection{Finite sample concentration}
Before present finite sample and step results of the algorithm in practice, 
we provide a concentration lemma that controls the deviation of the empirical mean from its population counterpart. Throughout, write $P_N f = N^{-1} \sum_{i=1}^N f(X_i)$ and $Pf = \Expect f(X)$. 
\begin{lemma}\label{lem:concentration}
    Let $\sHa \subset L_2(\bbP_X)$ a function class and assume $\|g^{(b)}\|_{\infty} \le \wt{C}$.
    and there exists an envelope function $H$ such that 
    $|h(\cdot)| \le H(\cdot)$ with $H^2 \in L_2(\bbP_X)$. Then, with probability at least $1 - \delta$,
    \begin{align*}
        & \sup_{h \in \sHa}  |  (P_N - P )h(X) s(Y, f^{(b)}(X)) | \le  C\left(  \frac{ \sup_Q J_2(\sHa; Q) }{\sqrt{N}}  +   \sqrt{ \frac{ \log(1/\delta)}{N}   }\right),  \\
        & \sup_{h \in \sHa}  | (P_N - P )h^2 | \le C \left( \frac{ \sup_Q J_2(\sHa^2; Q)}{\sqrt{N}} +  \sqrt{\frac{ \log(1/\delta)}{N}}\right),   
    \end{align*}
    where $J_2(\sHa; Q) = \int_0^{\infty} \sqrt{\log \mathcal{N}(\epsilon\|H\|_{L_2(Q)}, \sHa, L_2(Q))} d\epsilon $ denotes the Dudley entropy integral with some measure $Q$.  
\end{lemma} 

\begin{proof}[Proof of Lemma~\ref{lem:concentration}]
    Since $s(Y, f^{(b)})$ is measurable and $g^{(b)} = \Expect[s(Y, f^{(b)}(X)) | X]$, 
    \begin{align*}
      \sup_{h \in \sHa}  | \Expect (P_N - P )h(X) s(Y, f^{(b)}(X)) | & = \sup_{h \in \sHa}  | \Expect (P_N - P) h(X) g^{(b)}(X) |  \\
    & \le \Expect \sup_{h \in \sHa} | (P_N - P) h(X) g^{(b)}(X) | \\
    & \le \wt{C} \Expect \sup_{h \in \sHa} | (P_N - P) h(X) | 
    \end{align*} 
    Applying the symmetrization inequality~\citep{wellner2013weak,vershynin2018high, wainwright2019high} with Radamacher variables $\{\epsilon_i \}_{i=1}^N$, 
    $$
    \Expect \sup_{h \in \sHa} |(P_N - P) h| \le 2 \Expect \sup_{h \in \sHa} \left|\frac{1}{N} \sum_{i=1}^N \epsilon_i h(X_i) \right|. 
    $$
    Let $Z_h = N^{-1/2} \sum_{i=1}^N \epsilon_i h(X_i)$ and conditional on $X_1, \ldots, X_N$, $Z_h$ is a sub-Gaussian process with 
    $
    \|Z_{h_1} - Z_{h_2}\|_{\psi_2} \lesssim \|h_1 - h_2 \|_{L_2(\bbP_N)}, 
    $
    where $\bbP_N$ refers to empirical measure. By Dudley's integral bound for sub-Gaussian processes,
    $$
    \Expect \sup_{h \in \sHa} |Z_h| \lesssim \int_0^{\text{diam}(\sHa; \bbP_N)} \sqrt{\log \mathcal{N}(\epsilon, \sHa, L_2(\bbP_N )) } d\epsilon , 
    $$
    which leads to 
    $$
    \Expect \sup_{h \in \sHa} |(P_N - P) h| \lesssim \frac{1}{\sqrt{N}} \int_0^{\text{diam}(\sHa; \bbP_N)} \sqrt{\log \mathcal{N}(\epsilon, \sHa, L_2(\bbP_N )) } d\epsilon .
    $$
    For high-probability control, we have 
    $$
    \sup_{h \in \sHa} |(P_N - P)h| \le C\left( \frac{ \int_0^{\text{diam}(\sHa; \bbP_N)} \sqrt{\log \mathcal{N}(\epsilon, \sHa, L_2(\bbP_N )) } d\epsilon }{\sqrt{N}} + \frac{\text{diam}(\sHa; \bbP_N)u}{\sqrt{N}}\right)
    $$
    with probability at least $1 - 2\exp(-u^2)$, according to~\cite{vershynin2018high}. Substituting $u = \sqrt{\log(2/\delta)}$ yields 
    \begin{equation}
     \begin{aligned}
    & \sup_{h \in \sHa} |(P_N - P )h  |  \\ 
    & \le \frac{C}{\sqrt{N}} \left( \int_0^{\text{diam}(\sHa; \bbP_N)} \sqrt{\log \mathcal{N}(\epsilon, \sHa, L_2(\bbP_N )) } d\epsilon   +    \text{diam} (\sHa; \bbP_N ) \sqrt{ \log(\frac{2}{\delta}) } \right). 
    \end{aligned}  
    \end{equation}
To express the bound in population terms, let $H$ be an envelope satisfying $|h(\cdot)| \le H(\cdot)$ for all $h \in \sHa$ and $H \in L_4(\bbP_X)$, then 
\begin{align*}
    & \sup_{h \in \sHa} |(P_N - P )h  | \\
    &\lesssim \frac{1}{\sqrt{N}} \left( 
     \|H\|_{L_2(\bbP_N)} \int_{0}^{\text{diam}(\sHa; \bbP_N )/\|H\|_{L_2(\bbP_N)}} \sqrt{ \log N(\epsilon \|H\|_{L_2(\bbP_N)}, \sHa, L_2(\bbP_N ))} d\epsilon  \right. \\
    & \quad + \left. \text{diam}(\sHa; \bbP_N) \sqrt{\log(2/\delta)}
    \right) \\
    & = \frac{1}{\sqrt{N}} \left(   \|H\|_{L_2(\bbP_N)}  J_2(\sHa; \bbP_X) + \text{diam}(\sHa; \bbP_N) \sqrt{\log(2/\delta)} \right),
\end{align*} 
provided $\|H\|_{L_2(\bbP_N)}$. In above, 
$$
J_2(\sHa; \bbP_N ) := \int_{0}^{\text{diam}(\sHa; \bbP_N )/\|H\|_{L_2(\bbP_N)}} \sqrt{ \log N(\epsilon \|H\|_{L_2(\bbP_N)}, \sHa, L_2(\bbP_N ))} d\epsilon,  
$$
which represents Dudley's entropy integral. 
Using Chebyshev's inequality, 
$$
P_N H^2 \le 2 PH^2, 
$$
with probability at least $ 1 - (PH^4)/( N (PH^2)^2 )$. 
On this event, the empirical and population $L_2$-metrics are comparable. This implies 
$$
\sup_{h \in \sHa } |(P_N - P) h | \lesssim \frac{\|H\|_{L_2(\bbP_X)}}{\sqrt{N}} \left( \sup_Q J_2(\sHa; Q) + \sqrt{ \log \tfrac{2}{\delta}}\right), 
$$
with probability at least $ 1- \delta - O(N^{-1})$. Repeating the same arguments for the class $\sHa^2 = \{h^2: h \in \sHa \}$ yields 
$$
\sup_{h \in \sHa} |(P_N - P)h^2 | \lesssim \frac{\|H^2\|_{L_2(\bbP_X)}}{\sqrt{N}} \left( \sup_Q J_2(\sHa; Q) + \sqrt{ \log(2/\delta)}\right). 
$$
This completes the proof. 
\end{proof}

\paragraph{Consequences for VC-type classes.}
A lemma from~\cite{wellner2013weak} connects the entropy with VC dimensions:
\begin{lemma}[Theorem 2.6.7 of~\cite{wellner2013weak}]
    For a VC-class of functions $\sHa$ with finite VC dimension
    $V(\sHa)$, envelope function $H$ and $p \ge 1$, one has for any probability measure $Q$ with $\|H\|_{L_p(Q)} > 0$, 
    $$
    \mathcal{N}(\epsilon \|H\|_{L_p(Q)}, \sHa, L_p(Q) ) \le C \left(\frac{1}{\epsilon} \right)^{p \cdot V(\sHa)}, \quad 0 < \epsilon < 1. 
    $$
\end{lemma}
Consequently, if both $\sHa$ and $\sHa^2$ are VC classes with finite VC dimensions $V(\sHa)$ and $V(\sHa^2)$, then 
$$
J_2(\sHa; Q) \lesssim \sqrt{V(\sHa)}, \quad J_2(\sHa; Q) \lesssim \sqrt{V(\sHa^2)}, 
$$
and Lemma~\ref{lem:concentration} yields the simplified rates
\begin{align*}
    & \sup_{h \in \sHa}  | (P_N - P )h(X) g^{(b)}(X) | \lesssim \frac{\sqrt{V(\sHa)} + \sqrt{\log(1/\delta)}}{\sqrt{N}}, \\
    & \sup_{h \in \sHa}  | (P_N - P )h^2 | \lesssim \frac{\sqrt{V(\sHa^2)} + \sqrt{\log(1/\delta)}}{\sqrt{N}}. 
\end{align*}
Some examples of classes satisfying the condition include: (1) Two-layer neural networks $\sHa = \{ \text{Acitvate}(w^\top X + w_0), w \in \bbR^d, w_0 \in \bbR\}$. (2) Tree-basis functions $\sHa = \{\v1\{ (-\infty, a_1] \times \cdots \times (-\infty, a_d]: a_1, \dots,a_d \in \bbR \} \}$. The tree-basis and neural network classes all have finite VC-dimensions, and $\text{span}(\sHa)$ is dense in $C(U)$ for any compact subset $U \in \bbR^d$~\citep{zhang2005boosting}. (3) Finite-dimensional linear classes: $\sHa = \{ w^\top \phi(X): \|w\|_2 \le C_w, X \in \bbR^d \} $ with $\phi: \mathcal{X} \to \bbR^d$, $\Sigma = \Expect \phi(X)\phi^\top (X)$. It guarantees  
$\|h_w\|_{L_2} \le C_w \sqrt{\lambda_{\max}(\Sigma)}$. Covering a $d$-dimensional ball of radius $R$ requires at most $(CR/\epsilon)^{d}$ points, thus $J_2 \lesssim \sqrt{d} \sqrt{\lambda_{\max}(\Sigma)}$. 

Thus the uniform deviation decays as $O_p(N^{-1/2})$. To ensure the empirical weak-learner edge approximates its population counterpart, we assume a variance floor, i.e., $P_N (h^{(b)})^2 > c_{\sHa}^2 > 0$, to make passage from empirical to population more clear.  Then we have 
\begin{lemma}\label{lem:emp-approx-pop}
    Suppose $\sHa$ and $\sHa^2$ are VC-classes with finite VC dimensions $V(\sHa)$ and $V(\sHa^2)$, and there exists a variance floor for $h^{(b)}$ along the boosting path, i.e., $\|h^{(b)}\|_{L_2(\bbP_N)} > c_{\sHa}^2 > 0$, then with probability at least $1-\delta$, if the sample size satisfies 
    $$ N = \Omega\left(\frac{V(\sHa) + \log(1/\delta)}{ \alpha^2} \right), 
    $$ 
    the empirical and population weak-learner edges are comparable:
    $$
    \frac{| \langle h^{(b)}, g^{(b)} \rangle |  }{ \|h^{(b)}\|_{L_2}  } \ge \frac{\alpha}{2}, \text{ if }  \frac{|P_N h^{(b)} g^{(b)}|}{\|h^{(b)}\|_{L_2(\bbP_N)}} \ge \alpha, 
    $$
    and 
    $$
    \frac{| \langle h^{(b)}, g^{(b)} \rangle |  }{ \|h^{(b)}\|_{L_2}  } < \frac{3 \alpha}{2}, \text{ if } \frac{|P_N h^{(b)} g^{(b)} |}{\|h^{(b)}\|_{L_2(\bbP_N)}} < \alpha.  
    $$  
\end{lemma}

\begin{proof}[Proof of Lemma~\ref{lem:emp-approx-pop}]
    \begin{equation}\label{eqn:emp-approx-pop}
    \begin{aligned}
         \frac{| \langle h^{(b)}, g^{(b)} \rangle |  }{ \|h^{(b)}\|_{L_2}  } & \ge  \frac{ | P_N h^{(b)} g^{(b)} | - |(P_N - P) h^{(b)} g^{(b)}| }{\|h^{(b)}\|_{L_2(\bbP_N)} } \frac{ \|h^{(b)} \|_{L_2(\bbP_N) } }{ \|h^{(b)}\|_{L_2} } \\
    & \ge \left( \frac{ | P_N h^{(b)} g^{(b)} | }{\|h^{(b)}\|_{L_2(\bbP_N)} }  - \frac{\varepsilon_1}{ c_{\sHa}} \right) \frac{ \|h^{(b)} \|_{L_2(\bbP_N) } }{ \sqrt{ \|h^{(b)}\|_{L_2(\bbP_n)} + \varepsilon_2  } }    \\
    & > \left( \frac{ | P_N h^{(b)} g^{(b)} | }{\|h^{(b)}\|_{L_2(\bbP_N)} } - \frac{\alpha}{4} \right) \cdot \frac{2}{\sqrt{5}},
    \end{aligned}
\end{equation}
provided that
\begin{align*}
& \varepsilon_1 := \sup_{h \in \sHa}  | (P_N - P )h(X) g^{(b)}(X) | \le \tfrac{1}{4}  c_{\sHa} \alpha ,  \\
& \varepsilon_2 := \sup_{h \in \sHa}  | (P_N - P )h^2 | \le \tfrac{1}{4} c_{\sHa}^2 < \tfrac{1}{4}\|h^{(b)}\|_{L_2(\bbP_N)}.   
\end{align*}
The uniform deviation $\varepsilon_1, \varepsilon_2$ vanish at rate $O_p(N^{-1/2})$, and these inequalities hold with probability at least $1-\delta$ whenever
$$
N = \Omega \left( \frac{V(\sHa^2) + \log(1/\delta)}{c_{\sHa}^2} \right), \quad N = \Omega\left(\frac{V(\sHa) + \log(1/\delta)}{ \alpha^2 c_{\sHa}^2 } \right), 
$$
the latter dominating for a small $\alpha$. Under these,  
$ |\langle h^{(b)}, g^{(b)} \rangle|/\|h^{(b)}\|_{L_2} > \alpha/2 $ if 
$ |P_N h^{(b)} g^{(b)}|/\|h^{(b)}\| > \alpha$. 

Conversely,  
\begin{align*}
           \frac{| \langle h^{(b)}, g^{(b)} \rangle |  }{ \|h^{(b)}\|_{L_2}  } & \le  \frac{ | P_N h^{(b)} g^{(b)} | + |(P_N - P) h^{(b)} g^{(b)}| }{\|h^{(b)}\|_{L_2(\bbP_N)} } \frac{ \|h^{(b)} \|_{L_2(\bbP_N) } }{ \|h^{(b)}\|_{L_2} } \\
    & \le \left( \frac{ | P_N h^{(b)} g^{(b)} | }{\|h^{(b)}\|_{L_2(\bbP_N)} }   \right) \frac{ \|h^{(b)} \|_{L_2(\bbP_N) } }{ \sqrt{ \|h^{(b)}\|_{L_2(\bbP_N)}^2 - \varepsilon_2  } }  +  \frac{\varepsilon_1}{ C_{\sHa}} \\
    & < \frac{2}{\sqrt{3}}  \frac{ | P_N h^{(b)} g^{(b)} | }{\|h^{(b)}\|_{L_2(\bbP_N)} }  + \frac{\alpha}{4},
\end{align*}
provided that 
\begin{align*}
  & \varepsilon_1 \le \tfrac{1}{4}  C_{\sHa} \alpha , \quad
     \varepsilon_2  \le \tfrac{1}{4} c_{\sHa}^2.
\end{align*}
Hence, $ |\langle h^{(b)}, g^{(b)} \rangle|/\|h^{(b)}\|_{L_2} < \tfrac{3\alpha}{2} $ if 
$ |P_N h^{(b)} g^{(b)}|/\|h^{(b)}\| < \alpha$, completing the proof.  
\end{proof}

Therefore, this result recovers the conclusions of earlier work~\cite{kim2019multiaccuracy, kim2022universal,deng2023happymap} while removing the restrictive uniform boundedness assumption. The variance floor condition is not overly stringent: we do not expect a degenerate case where $h^{(b)}(X_i) = 0, i = 1, \dots, N$, since this would imply $P_N h^{(b)} g = 0$ and by convention we treat $0/0 = 0$. The variance floor can always be ensured by normalization $\wt{h}^{(b)} = h^{(b)}/\|h^{(b)}\|_{L_2(\bbP_N)}$, or by incoporating a small stabilizing constant $c_{\sHa}$ in the denominator of the violation statistic $\wt{\Delta} = P_N g^{(b)} h/(\|h\|_{L_2(\bbP_N) } \wedge c_{\sHa} )$. In the latter case, one obtains 
$$
\frac{| \langle h^{(b)}, g^{(b)} \rangle |  }{ \|h^{(b)}\|_{L_2}  } \ge   \frac{| \langle h^{(b)}, g^{(b)} \rangle |  }{ \|h^{(b)}\|_{L_2} \wedge c_{\sHa}   } > \alpha - \frac{|(P_N - P) g^{(b)} h^{(b)}|}{c_{\sHa}}. 
$$ 
The above argument subsumes the uniformly bounded settings frequently assumed in prior analyses, where $\|h\|_{\infty} \le C_{\sHa}$; such a condition allows one to obtain bounds using classical concentration techniques like McDiarmid's inequality or Hoeffding-type arguments for simpler concentration bounds. Our analysis, by contrast, accommodates more general auditor classes. 

When data are limited, it is acceptable to evaluate the violation statistics~\eqref{eqn:violation} directly on the calibration set--an approach analogous to that in conformal inference~\citep{romano2020malice,romano2020classification,candes2023conformalized}. Furthermore, if algorithm halts after a finite number of steps $B$, and one desires joint guarantee across all iterations with probability at least $1 - \delta$, 
the per-iteration probability should be set to $1 - \delta/B$. This adjustment leads to sample size requirement 
$N = \Omega(\tfrac{V(\sHa) + \log(B/\delta)}{\alpha^2})$.

\subsection{Proof of Proposition~\ref{prop:finite} }

\begin{proof}
     Suppose the procedure either (i) finds $h^{(b)}$ with 
     $$P_N g^{(b)} h^{(b)} > \alpha\, \sqrt{ P_N (h^{(b)})^2}, $$ or (ii) certifies no such \(h\) exists (stop).
Consider case (i), with the smooth-loss step choice
$$
\eta^{(b)} = \langle s^{(b)}, h^{(b)}\rangle/(2 c_L \Expect [h^{(b)}(X)]^2 ),
$$
the descent of the excess risk can be characterized as
\begin{align*}
\Delta_b-\Delta_{b+1} =  & L(f^{(b)})-\mathcal L(f^{(b+1)}) \\
    & \ge \Expect[ (f^{(b)} - f^{(b+1)} )  s( Y, f^{(b)}) ] - c_{L} \Expect (f^{(b)} - f^{(b+1)} )^2  \\
    & = \eta^{(b)} \Expect[ h^{(b)} s(f^{(b)}, Y)] - c_{L} \Expect (f^{(b)} - f^{(b+1)} )^2  \\
    &  \stackrel{(Lem~\ref{lem:emp-approx-pop})}{>} \frac{\alpha^2}{ 16 c_L }. 
\end{align*}
Summing over iterations gives 
$$
B \cdot \frac{\alpha^2}{16c_L} < \sL(f^{(0)}) - \sL(f^{(B)}),
$$
which implies that 
$$
B \le  \tfrac{16 c_L}{\alpha^2}\,(\mathcal L(f^{(0)})-\mathcal L(f^{(\infty)})) 
$$
is sufficient for termination.  

If the loss satisfies the $\mu$-PL condition, the above inequality refines to 
$$
\Delta_b  > \frac{\alpha^2}{16 c_L} + \Delta_{b+1}.  
$$
Before halting, we must have 
$$
\frac{\alpha^2}{16 c_L} < \left(1 - \frac{\mu \kappa}{2c_L} \right)^{B} \Delta_0, 
$$
Hence, 
$$
B = O\left(\log \tfrac{\sL(f^{0}) - \sL(f^{(\infty)})}{\alpha^2} \right),
$$
which establishes a logarithmic iteration complexity and completes the proof.
\end{proof} 

\subsection{Proof of Theorem~\ref{thm:stopping} } 

\begin{proof}
The stopping rule of Algorithm~\ref{alg:MCboost} ensures 
$$
2 \alpha \stackrel{(Lem~\ref{lem:emp-approx-pop})}{>} \frac{ \langle s^{(B)}, h^{(B)}\rangle }{ \|h^{(B)} \|_{L_2}} > \kappa \| g^{(B)}\|_{*}.  
$$
Thus, for any $h \in \sF$, 
\begin{align*}
    & \sup_{h \in \sF} |\Expect h(X) s(Y, f^{(B)}(X))| = \sup_{h \in \sF} \langle h, g^{(B)} \rangle \\
    & \le \sup_{h \in \mathcal{F}} \| h \|_{L_2} \|g^{(B)}\|_{*} \\
    & \le C_{\sF}\|g^{(B)}\|_{*} <  C_{\sF} \frac{2 \alpha}{\kappa},
\end{align*}
which proves the claim.  
\end{proof}

\subsection{Proof of Theorem~\ref{thm:universal}}
\begin{proof}
In a broader view, 
define the target error
$$
\Err(\wt{f} ) = \left| \Expect_{\sD_\sT} [s(Y, \wt{f}(X))]  \right|,   
$$
For squared loss this reduces to $|  \Expect_{\sD_\sT}(\wt{f}(X) - Y) |$ for some post-processing $\wt{f}(X)$, while for pinball loss at level $\tau$ it coincides with 
$
\left| \Expect_{\sD_\sT} [\v1\{ Y \le \wt{q}_{\tau}(X)\} - \tau]  \right|, 
$
with $\wt{f} := \wt{q}_{\tau}$.

\medskip
\paragraph{(i) Universal adaptability.}  
By a standard decomposition, for any weight function $\wt w$ and any $h\in\mathcal H$,  
\begin{align*}
    & \Err_{\sT}( \wt{f} ) = | \Expect_{\sD_{\sT}} [  s(Y, \wt{f}(X)) ] | \\
    & = | \Expect_{\sD_{\sS}}[w(X)s(Y, \wt{f}(X)) ]  | \\
    & \le | \Expect_{\sD_{\sS}}[(w(X) - h(X)/ C) s(Y, \wt{f}(X))] |  + | \Expect_{\sD_{\sS}}[ \tfrac{h(X)}{C} s(Y, \wt{f}(X)) ]  |  \\ 
    & \le | \Expect_{\sD_{\sS}}[(w(X) - h(X)/ C) s_1(Y, \wt{f}(X)) ] | + | \Expect_{\sD_{\sS}}[(w(X) - h(X)/ C) s_2(Y) ] | + \frac{\alpha}{C} \\ 
    & \le | \Expect_{\sD_{\sS}}[(w(X) - h(X)/ C) s_1(Y, \wt{f}(X)) ] | + | \Expect_{\sD_{\sS}}[(\wt{w}(X) - h(X)/ C) s_2(Y) ] | \\
    & \quad +  | \Expect_{\sD_{\sS}}[(w(X) - \wt{w}(X) ) s_2(Y) ] |  + \frac{\alpha}{C} \\ 
    & \le \wt{C} \{ \|w - h/C\|_{L_2} + \|h/C - \wt{w}\|_{L_2} \} + \Err(\tau^{(ps)}(\wt{w})) + \alpha/C,
\end{align*}
where $ \Err(\tau^{(ps)}(\wt{w})) = | \Expect_{\sD_{\sS}}[(w(X) - \wt{w}(X) ) s_2(Y) ] | $ is the error of inverse propensity score reweighting. Optimizing over $h \in \sH$ yields
$$
\Err_{\sT}( \wt{f} ) \le \wt{C} \inf_{h \in \sH} \{ \|w - h/C\|_{L_2} + \|h/C - \wt{w}\|_{L_2} \}  + \Err(\tau^{(ps)}(\wt{w})) + \tfrac{\alpha}{C}. 
$$  
In particular, if $\wt{w} \in \sH$, we obtain 
$$
\Err_{\sT}( \wt{f} )  \le \Err(\tau^{(ps)}(\wt{w})) +  \|w - \wt{w}\|_{L_2} + \alpha. 
$$  
If under the stronger condition that $\|\Expect_{\sD_\sS}[s_1 | X]\|_{\infty}, \|\Expect_{\sD_\sS}[s_2 | X]\|_{\infty} \le \wt{C}$, we can use $L_1$-distance $\|\cdot\|_{L_1}$ instead of $L_2$-distance in the above formulation. 

\medskip
\paragraph{(ii) Transfer of multicalibration.} 
Suppose $\wt{f}(X)$ is $(\mathcal{H}(\Sigma) \otimes \mathcal{C}, \alpha)$-multicalibrated on the source domain, 
where 
$ \sH(\Sigma) $
is the set of shifts induced by the propensity scores $\Sigma = \{ \wt{\sigma}_{\sT}(X) := \bbP(D=\sT \mid X)  \}$, and  
$\mathcal{C}$ is a class of functions with $\|c(\cdot)\|_{L_2} \le \wt{C}$; 
then for any $c \in \mathcal{C}$ and $\wt{w} \in \mathcal{H}(\Sigma)$,  
\begin{align*}
	& | \Expect_{\sD_\sT} [ c(X) s(Y, \wt{f}(X))  ] |  = | \Expect_{\sD_\sS} [ w(X)c(X)s(Y, \wt{f}(X))   ] |  \\
	&  \le | \Expect_{\sD_\sS} [ \wt{w}(X)c(X)s(Y, \wt{f}(X))   ] | \\
	& \quad + | \Expect_{\sU_\sS} [ ( w(X) -\wt{w}(X) ) c(X)s(Y, \wt{f}(X))  ] | \\
	& \le \alpha + \wt{C}^2 \|w - \wt{w}\|_{L_2}. 
\end{align*}
Hence, $\wt{f}$ is $(\mathcal{C}, \alpha' + \wt{C}^2 \inf_{\wt{w} \in \mathcal{H}(\Sigma)}  \|w - \wt{w}\|_{L_2} )$-multicalibrated on the $\sD_\sT$ whenever it is $( \mathcal{H}(\Sigma) \otimes \mathcal{C}, \alpha )$-multicalibrated on $\sD_\sS$. 

\medskip 
\paragraph{(iii) Multiple source domains.}
Let $\sS = \{\sS_1, \dots, \sS_M\}$ with mixing weights $p_m=\bbP(D=\sS_m\mid D\in\sS)$.  
Write $w_m(X)=r_m\sigma_\sT(X)/\sigma_{\sS_m}(X)$ with $r_m=\bbP(D=\sS_m)/\bbP(D=\sT)$ and $\sigma_{\sS_m}(X)=\bbP(D=\sS_m\mid X)$.
Suppose 
\begin{align*}
        \sup_{h \in \sH} \left| \Expect_{\sD_{\sS}} h(X)s(Y, \wt{f})  \right| = \sup_{h \in \sH} \left| \sum_{m=1}^M p_m \Expect_{\sU_{\sS_m} }\left\{ h(X)  s(Y, \wt{f}(X))  \right\} \right| < \alpha,
\end{align*}  
We have 
\begin{align*}
    & \Expect_{\sD_\sT }[s(Y, \wt{f}(X)) ] = \int p(X|\sT) s(Y, \wt{f}(X)) \mu(dY\times dX) \\
    & = \sum_{m=1}^M p_m \int \frac{p(X|\sT)}{p(X|\sS_m )} \wt{g}(X) p(X|\sS_m ) \mu(dX)     \\
        & =  \sum_{m = 1}^M p_m  \Expect_{\sU_{\sS_m} } \left\{ r_m \frac{\sigma_{\sT }(X)}{\sigma_{\sS_m }(X) } \wt{g}(X) \right\} \\
        & = \sum_{m = 1}^M p_m  \Expect_{\sU_{\sS_m} } \left\{ w_m(X) \wt{g}(X) \right\}.  
\end{align*}  
Define the vector of weight functions
$$
\vw(X) = (w_1(X), w_2(X), \ldots, w_M(X)), 
$$ and by the same decomposition above, 
\begin{align*}
   & \Err_{\sT}(\wt{f} ) = \left| \Expect_{\sD_\sT} s(Y, \wt{f})\right| \\
   & =  \left| \sum_{m=1}^M p_m \Expect_{\sU_{\sS_m}} w_m \wt{g} \right| \\
   & = \inf_{h \in \sH} \left\{  \left| \sum_{m=1}^M p_m \Expect_{\sU_{\sS_m}}(w_m(X) - \frac{h(X)}{C_m} ) \wt{g}  \right| + \left| \sum_{m=1}^M p_m \Expect_{\sU_{\sS_m}} \frac{h(X)}{C_m} \wt{g}  \right|  \right\} \\
   & \le \inf_{h \in \sH} \left\{ \left| \sum_{m=1}^M p_m \Expect_{\sU_{\sS_m}}(w_m(X) - \frac{h(X)}{C_m} ) \wt{g}_1  \right| + \left| \sum_{m=1}^M p_m \Expect_{\sU_{\sS_m}}(w_m(X) - \wt{w}_m(X) + \wt{w}_m(X) - \frac{h(X)}{C_m} ) g_2  \right|  \right\} \\
   & \quad + (\max_m \frac{1}{C_m}) \alpha  \\ 
   &  \le \sum_{m=1}^M p_m \wt{C} \inf_{h \in \sH} \left\{ \| w_m - h/C_m \|_{L_2} + \| h/C_m - \wt{w}_m \|_{L_2} \right\} + (\max_m \tfrac{1}{C_m} ) \alpha + \Err(\tau^{ps}( \wt{\vw} )).  
\end{align*}
Here $\wt{g}_1 = \Expect[s_1(Y, \wt{f}) | X]$ and $g_2 = \Expect[s_2(Y)|X]$. If in addition each shift $w_m$ can be represented as $h_m(X)/\mathbb E_{\sU_{\sS_m}}h_m(X)$ with $h_m\in\mathcal H$, then choosing $C_m=\mathbb E_{\sU_{\sS_m}}h_m(X)$ and $\wt w_m=w_m$ yields
$$
\Err_{\sT}(\wt{f} ) \le \frac{\alpha}{\min_m \Expect_{\sU_{\sS_m}} h_m(X)}. 
$$ 
\end{proof}

\section{Special cases of multicalibration boosting}\label{apxsec:instances}

\subsection{BatchGCP}
Algorithm 1 of~\cite{jung2022batch}, referred to as BatchGCP, aims to enforce group-wise coverage without calibration. 
\begin{align*}
    h^{(0)}(X) & = \mathcal{A}(\Xi; (f^{(0)}(X_i), Y_i): X_i \in \Xi) \\
    & \approx \argmin_{h \in \sHa} \Expect [L(Y, f^{(0)}(X) - h(X) )],   
\end{align*}
where there is a single global bucket $\Xi$ and
$
\sHa = \left\{ \sum_{G \in \mathcal{G}} \beta_G \v1\{X \in G\} : G \in \mathcal{G} \right\}.
$ 
The updated predictor is then $\wt{f}(X) = f^{(0)}(X) -  h^{(0)}(X)$, which corresponds to a one-shot boosting refinement. 

\subsection{BatchMVP}
We next show that Algorithm 2 of~\citep{jung2022batch}, MultiMVP, is a concrete instantiation of the generalized multicalibration boosting framework (Algorithm~\ref{alg:MCboost-general}). Let $f_{\tau}^{(0)}(X)$ be a pre-trined conditional quantile predictor at level $\tau$. Each prediction is first discretized onto an equally spaced grid, 
$$
q_{\tau}^{(0)}(X) = \argmin_{v \in [1/L]} |v - f_{\tau}^{(0)}(X)|, \quad [\tfrac{1}{L}] = \left \{0, \tfrac{1}{L}, \tfrac{2}{L}, \ldots, 1 \right\}, 
$$
producing a snapped quantile estimate taking one of $L+1$ discrete values. Define covariate-prediction buckets
$$
S_{G, l} = \{X \in G, q_{\tau}^{(b)}(X) = v_l \}, \quad  G \in \mathcal{G}, v_l = \frac{l}{L}, l = 0, 1, \ldots, L.
$$
At iteration $b$, the auditing step learns piecewise-constant function within 
$$
\sHa = \{\beta_{G, l} \v1\{X \in G\}\v1\{q_{\tau}^{(b)} = v_l \} : G \in \mathcal{G}, v_l = [1/L] \}, 
$$
which correspond to group-wise constant corrections over bucket $S_{G,l}$. That is, for each pair $(G, v_l)$, the auditor fits 
$$
\wh{\beta}_{G, l} \v1\{X \in G\}\v1\{q_{\tau}^{(b)} = v_l \} \leftarrow \mathcal{A}(\Xi; \{(Y_i, q_{\tau}^{(b)}(X_i)): X_i \in G, q_{\tau}^{(b)}(X_i) = v_l \}), 
$$
which amounts to regressing the subgradient 
$$
\v1\{Y_i \le q_{\tau}^{(b)} \} -\tau, \quad X_i \in G, q_{\tau}^{(b)}(X_i) = v_l. 
$$
The constant fit for each bucket therefore equals 
$$
\wh{\beta}_{G, l} = \frac{ \sum_{S_{G,l}} \v1\{Y_i \le v_l\} }{ \# S_{G, l} } - \tau, \quad \text{ since } q_{\tau}^{(b)}(X_i) = v_l.
$$
As in~\cite{jung2022batch}, the same data are used for both calibration and validation $(\Xi = V)$, a standard practice in boosting. We do not make a distinction here. 

The evaluation step then chooses 
$$
h^{(b)}(X) = \wh{\beta}_{G^*, v_l^*} \v1\{X \in G^* \} \v1\{q_{\tau}^{(b)}(X) = v_{l^*}\}, 
$$
where $(G^*, l^*)$ is the most violated bucket:
\begin{align*}
    S_{G^*, l^*} & = \argmax_{G, l} \frac{\# S_{G, l}}{ \# V } \left( \frac{\sum_{S_{G, l}} \v1\{Y_i \le v_l \}}{\# S_{G,l}}   - \tau \right)^2  \\
    & \approx \argmax_{G, l}  \bbP(S_{G, l}) \left(\tau - \bbP(Y \le v_{l} | q_{\tau}^{(b)}(X) = v_l, X \in G) \right)^2. 
\end{align*} 
The selection rule exactly matches that of Algorithm 2 in~\cite{jung2022batch}. 

In their implementation,~\cite{jung2022batch} update the predictor by adding a small shift $\delta^{(b)}$ that re-aligns the empirical coverage:
$$
\delta^{(b)} = \argmin_{\delta \in [1/L]} \left| \frac{ \sum_{S_{G^*, l^*}} \v1 \{Y_i \le q_{\tau}^{(b)}(X_i) + \delta \}}{\# S_{G^*,l^*}}  - \tau \right|.
$$
The quantile predictor is then updated as 
$$
q_{\tau}^{(b+1)}(X) = q_{\tau}^{(b)}(X) + \delta^{(b)}\v1\{X \in S_{G^*, l^*}\}. 
$$
This is equivalent to the boosting-style update 
$$
q_{\tau}^{(b+1)} = q_{\tau}^{(b)}(X) - \eta^{(b)} \left( \tfrac{ \sum_{S_{G^*, l^*}} \v1 \{Y_i \le q_{\tau}^{(b)}(X_i)  \}}{\# S_{G^*,l^*}} - \tau \right) \v1\{X \in S_{G^*, l^*}\},  
$$
with step size $\eta^{(b)}$ chosen so that $\eta^{(b)} \wh{\beta}_{G^*, l^*} = \delta^{(b)}$. 

Under this formulation, the learned predictor belongs to $\sF$, which is the closure of class 
$$
\left\{ \sum_{j=0}^k \eta_j \beta_{G_j, v_{l_j}} \v1\{X \in G_j, q_{\tau}^{(b)}(X) \in v_{l_j} \} : G_j \in \mathcal{G}, v_{l_j} \in [\tfrac{1}{L}], k \in \mathbb{N} \right\}. 
$$
To guarantee finite-sample multicalibration for a predictor obtained after finitely many updates, one may restrict attention to functions in $\sF$ with bounded $L_2$ norm. 

\section{Miscellaneuous}

\subsection{\texorpdfstring{General convex but non-$c_L$-smooth loss}{General convex but non-cL-smooth loss}}\label{apxsubsec:nonsmooth-convex}
When $\mathcal{L}$ is convex but not $c_L$-smooth, the faster $O(B^{-1})$ and geometric rates of Theorem~\ref{thm:smooth-rate} and Theorem~\ref{thm:PL-linear} are generally unavailable. In that case, we only guarantee the rate $O(B^{-1/2})$. 

\begin{theorem}[Convex, non-smooth loss]
\label{thm:nonsmooth-rate}
Assume Conditions~\ref{ass:l2}-\ref{ass:finite-radius}, and let $O$ be closed and convex. Suppose the MCBoost iterates satisfy
$ f^{(b+1)} = \mathcal P_O\bigl(f^{(b)}-\eta^{(b)} h^{(b)}\bigr),
$ 
and for each $b$, 
$ \|h^{(b)} - g^{(b)}\|_* \le \varepsilon^{(b)}. $
Let $\Delta_b := \mathcal L(f^{(b)}) - \mathcal L(f^{(\infty)})$. Then, for any $B \ge 1$,
\begin{equation}\label{eqn:convex-nonsmooth}
2 \sum_{b=0}^{B-1} \eta^{(b)} \Delta_b
\le
R^2 + C_{\sHa}^2 \sum_{b=0}^{B-1} \eta^{(b)2} + 2R \sum_{b=0}^{B-1} \eta^{(b)} \varepsilon^{(b)}.
\end{equation}
In particular, if $\eta^{(b)} \equiv \eta$, then
$$
\min_{0 \le b < B} \Delta_b
\le
\frac{1}{B}\sum_{b=0}^{B-1}\Delta_b
\le
\frac{R^2}{2\eta B} + \frac{\eta C_{\sHa}^2}{2} + \frac{R}{B}\sum_{b=0}^{B-1}\varepsilon^{(b)}.
$$
Choosing $\eta = R/(C_{\sHa}\sqrt{B})$ gives
$$
\min_{0 \le b < B} \Delta_b
\le
\frac{1}{B}\sum_{b=0}^{B-1}\Delta_b
\le
\frac{RC_{\sHa}}{\sqrt{B}} + \frac{R}{B}\sum_{b=0}^{B-1}\varepsilon^{(b)}.
$$
Moreover, for the averaged iterate $\bar f_B = B^{-1}\sum_{b=0}^{B-1} f^{(b)}$,
$$
\mathcal L(\bar f_B) - \mathcal L(f^{(\infty)})
\le
\frac{RC_{\sHa}}{\sqrt{B}} + \frac{R}{B}\sum_{b=0}^{B-1}\varepsilon^{(b)} = O(1/\sqrt{B}),
$$
provided that $\sum_{b=0}^{B-1}\varepsilon^{(b)} = o(\sqrt{B})$. 
\end{theorem}

\begin{proof}
Since $\mathcal P_O$ is nonexpansive,
\[
\|f^{(b+1)}-f^{(\infty)}\|_{L_2}^2
\le
\|f^{(b)}-\eta^{(b)} h^{(b)}-f^{(\infty)}\|_{L_2}^2.
\]
Expanding the square yields
\[
\|f^{(b+1)}-f^{(\infty)}\|_{L_2}^2
\le
\|f^{(b)}-f^{(\infty)}\|_{L_2}^2
-2\eta^{(b)}\langle h^{(b)}, f^{(b)}-f^{(\infty)}\rangle
+\eta^{(b)2}\|h^{(b)}\|_{L_2}^2.
\]
By convexity of $\mathcal L$ and $g^{(b)}\in\partial \mathcal L(f^{(b)})$,
\[
\Delta_b \le \langle g^{(b)}, f^{(b)}-f^{(\infty)}\rangle.
\]
Therefore,
\begin{align*}
2\eta^{(b)}\Delta_b
&\le 2\eta^{(b)}\langle h^{(b)}, f^{(b)}-f^{(\infty)}\rangle
+2\eta^{(b)}\langle g^{(b)}-h^{(b)}, f^{(b)}-f^{(\infty)}\rangle \\
&\le \|f^{(b)}-f^{(\infty)}\|_{L_2}^2 - \|f^{(b+1)}-f^{(\infty)}\|_{L_2}^2
+\eta^{(b)2}\|h^{(b)}\|_{L_2}^2 + 2\eta^{(b)} R\varepsilon^{(b)} \\
&\le \|f^{(b)}-f^{(\infty)}\|_{L_2}^2 - \|f^{(b+1)}-f^{(\infty)}\|_{L_2}^2
+\eta^{(b)2} C_{\sHa}^2 + 2\eta^{(b)} R\varepsilon^{(b)},
\end{align*}
where we used $\|f^{(b)}-f^{(\infty)}\|_{L_2}\le R$ and $\|h^{(b)}\|_{L_2}\le C_{\sHa}$. Summing over $b=0,\dots,B-1$ gives~\eqref{eqn:convex-nonsmooth}. The bound for constant step size follows immediately. Finally, by convexity of $\mathcal L$,
\[
\mathcal L(\bar f_B) - \mathcal L(f^{(\infty)})
\le
\frac{1}{B}\sum_{b=0}^{B-1}\Delta_b,
\]
which gives the stated bound for the averaged iterate.
\end{proof}

\begin{rem}[Unknown horizon and diminishing step size]
\label{rem:nonsmooth-stepsize}
If the horizon $B$ is unknown in advance, one may use a diminishing step size $\eta^{(b)} \asymp (b+1)^{-1/2}$. Then
$$
\sum_{b=0}^{B-1} \eta^{(b)} \asymp \sqrt{B},
\qquad
\sum_{b=0}^{B-1} \eta^{(b)2} \asymp \log B,
$$
and~\eqref{eqn:convex-nonsmooth} yields
$$
\frac{\sum_{b=0}^{B-1}\eta^{(b)}\Delta_b}{\sum_{b=0}^{B-1}\eta^{(b)}}
\lesssim
\frac{\log B}{\sqrt{B}}
+
\frac{R\sum_{b=0}^{B-1}\eta^{(b)}\varepsilon^{(b)}}{\sum_{b=0}^{B-1}\eta^{(b)}}.
$$
Hence, if $\sum_{b=0}^{B-1}\eta^{(b)}\varepsilon^{(b)} = o(\log B)$, the weighted average iterate
$$
\bar f_B = \frac{\sum_{b=0}^{B-1} \eta^{(b)} f^{(b)}}{\sum_{b=0}^{B-1}\eta^{(b)}}
$$
achieves
$$
\mathcal L(\bar f_B) - \mathcal L(f^{(\infty)}) = O\!\left(\frac{\log B}{\sqrt{B}}\right).
$$
\end{rem}
Thus, in the non-smooth convex setting, one needs $B = O(\alpha^{-2})$ iterations to reach excess risk of order $\alpha$, up to the auditing approximation term. If this population bound is combined with the same VC-type concentration argument used in Proposition~\ref{prop:finite}, the natural sample-size scale remains
$$
N = \Omega\!\left(\frac{\log(B/\delta)+V(\sHa)}{\alpha^2}\right).
$$

\paragraph{Classification losses.}
In classification, logistic and exponential losses are common examples. It is useful to record their score functions explicitly. 
\begin{itemize}
    \item The logistic loss:
    \begin{align*}
        & L(Y, f) = \log(1 + \exp(-fY)),  \text{ for } Y \in \{-1, 1\}, \\
        & L(Y, f) = \log(1 + \exp(-f(2Y - 1))),  \text{ for } Y \in \{0, 1\}.
    \end{align*}
    The corresponding score functions are
    \begin{align*}
        s(Y, f(X)) &= \frac{-Y}{1 + \exp(f(X) \cdot Y)},
         \,\, Y \in \{-1, 1\}, \\
        s(Y, f(X)) &= (1 - 2Y) \cdot \frac{1}{1 + \exp(f(X)(2Y-1))}, \,\, Y \in \{0, 1\}.
    \end{align*}
    \item Exponential loss: 
    \begin{align*}
        & L(Y, f) = \exp(-f Y),  \text{ for } Y \in \{-1, 1\},  \\ 
        & L(Y, f) = \exp(-f(2Y-1)),   \text{ for } Y \in \{0, 1\}.
    \end{align*}
    The corresponding score functions are
    \begin{align*}
        s(Y, f(X)) &= -Y\exp(-f(X) \cdot Y), \,\,
         Y \in \{-1, 1\}, \\[4pt]
        s(Y, f(X)) &= (1 - 2Y) \cdot \exp(-f(X)(2Y-1)),
         \,\, Y \in \{0, 1\}.
    \end{align*}
\end{itemize}

\section{Additional numerical results}

\subsection{Disregarding the group splitting and group features}

\begin{figure}[ht!]
    \begin{subfigure}[b]{\textwidth}
        \includegraphics[width=\textwidth]{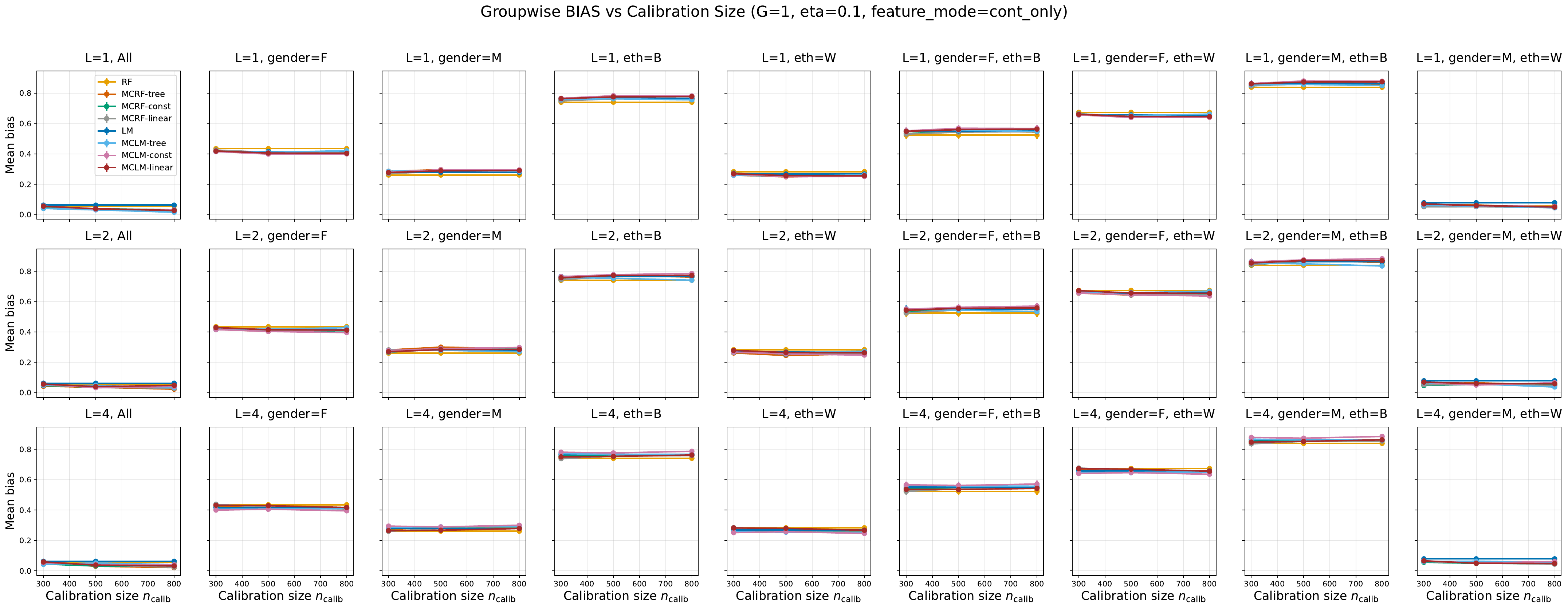}
    \end{subfigure}
    \begin{subfigure}[b]{\textwidth}
        \includegraphics[width=\textwidth]{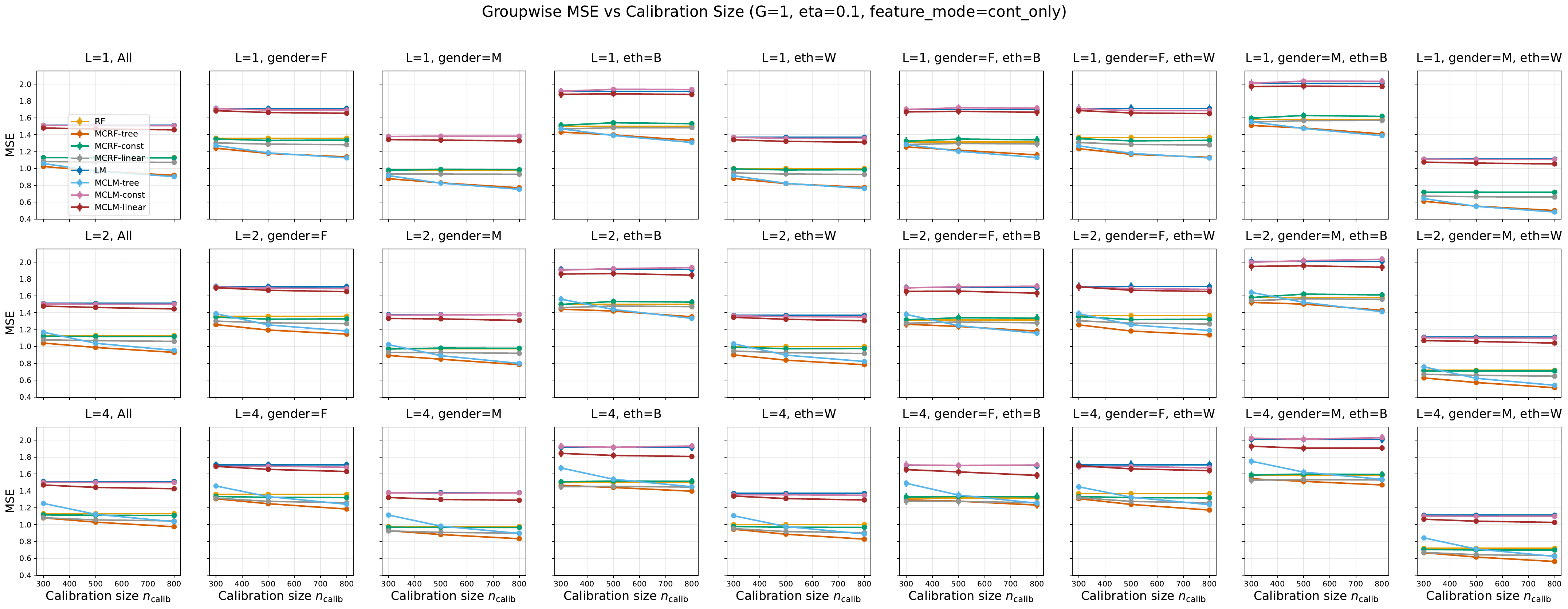}
    \end{subfigure}
    \caption{ 
     Bucket-based calibration without partition ($|\mathcal{G}| = 1$) and categorical features excluded from model fitting.   
     Top:  Mean groupwise biases across subpopulations  
    vs. calibration size for different initial predictors (linear model and random forest) and auditors (decision tree, constant, and linear).
    Bottom: Mean squared error across subpopulations  
    vs. calibration size for different initial predictors (linear model and random forest) and auditors (decision tree, constant, and linear).
    }
    \label{fig:bias_mse_vs_calib_nogroup}
\end{figure}
\FloatBarrier 

\subsection{Excess risk and interaction between initial predictors and auditors}
\begin{figure}[ht!]
    \centering
    \begin{subfigure}[b]{0.48\textwidth}
    \includegraphics[width=\linewidth]{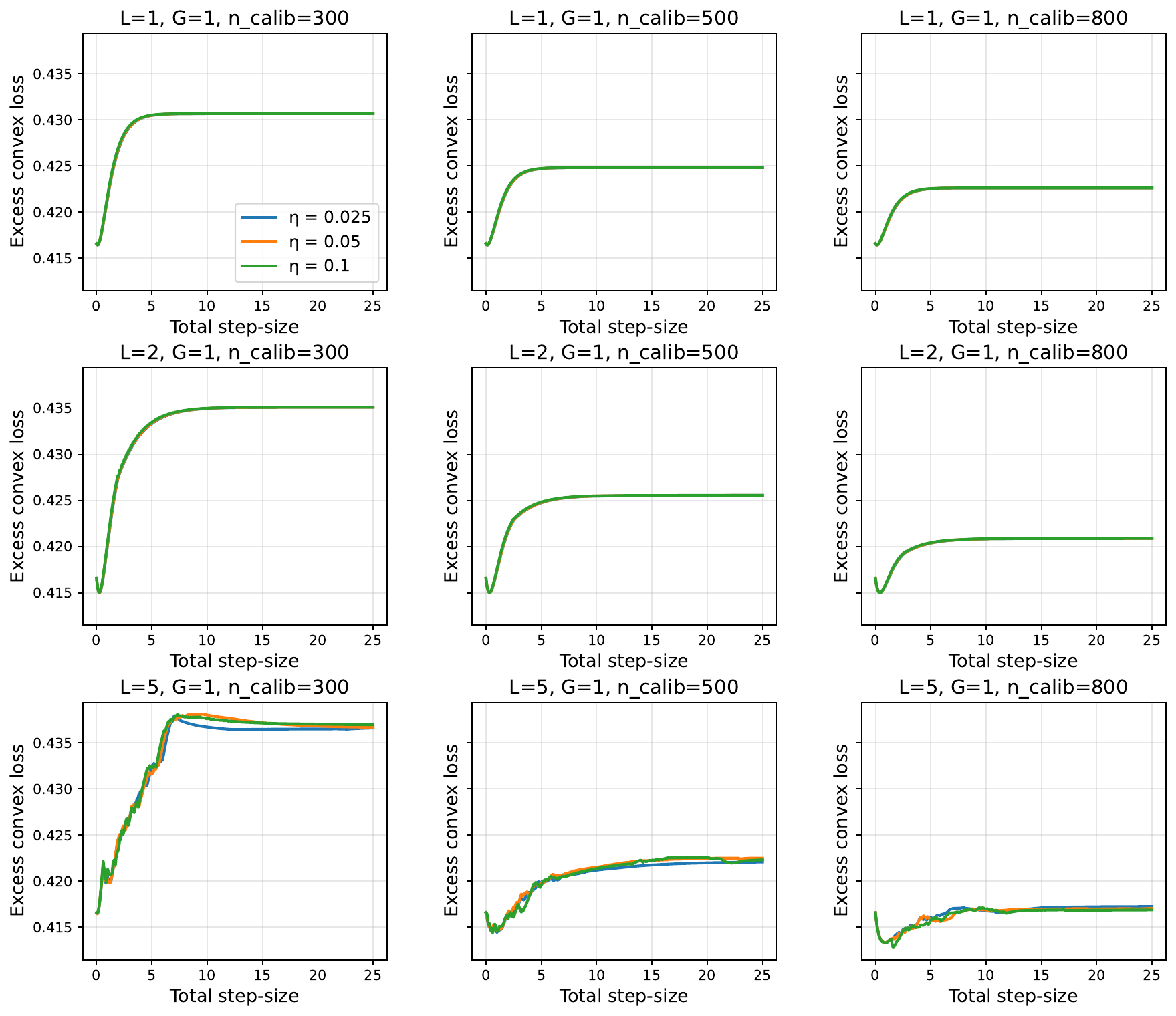}
    \subcaption*{(a)}
    \end{subfigure}
    \begin{subfigure}[b]{0.48\textwidth}
    \includegraphics[width=\linewidth]{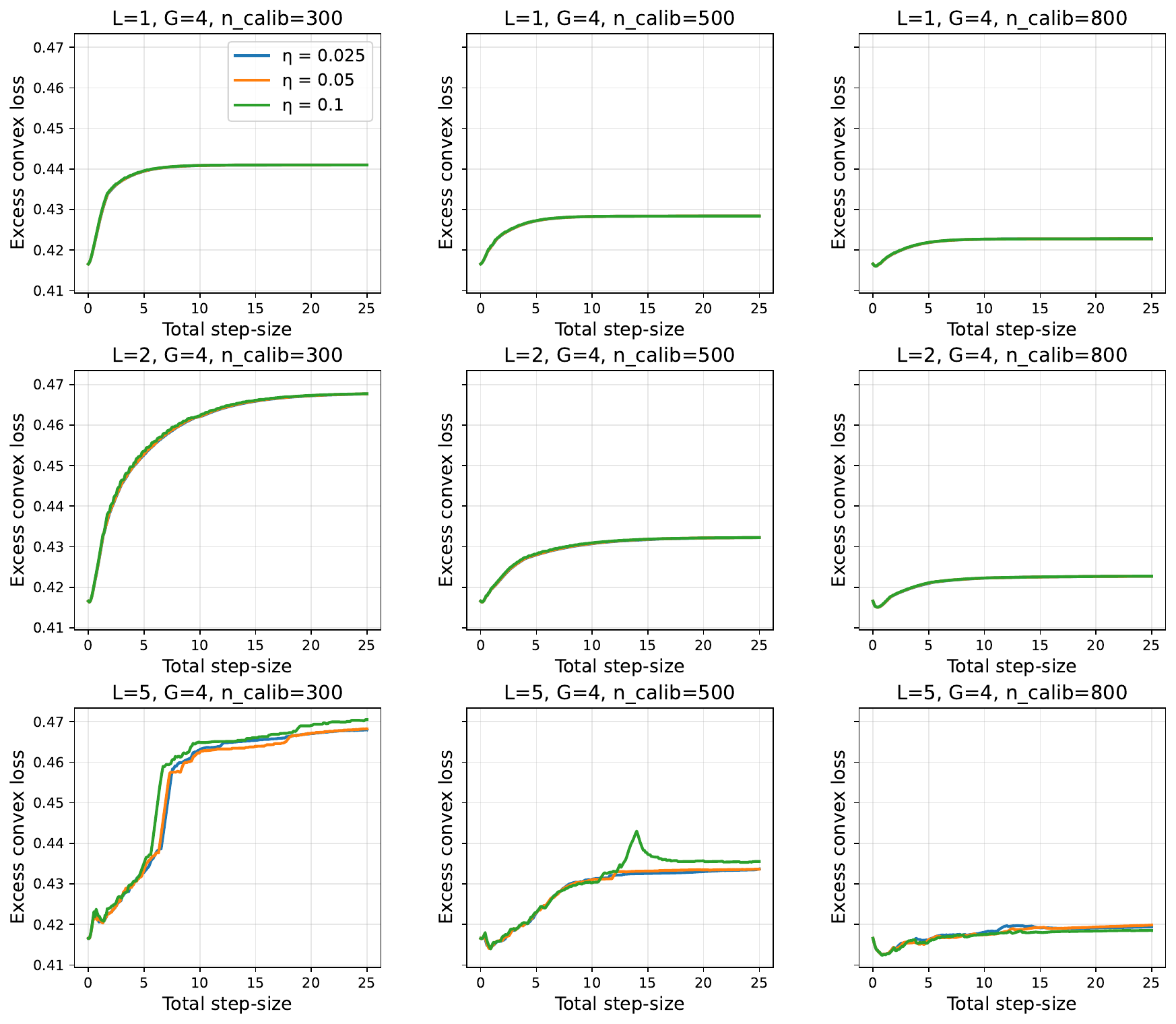}
    \subcaption*{(b)}
    \end{subfigure} 
    \begin{subfigure}[b]{0.48\textwidth}
   \includegraphics[width=\linewidth]{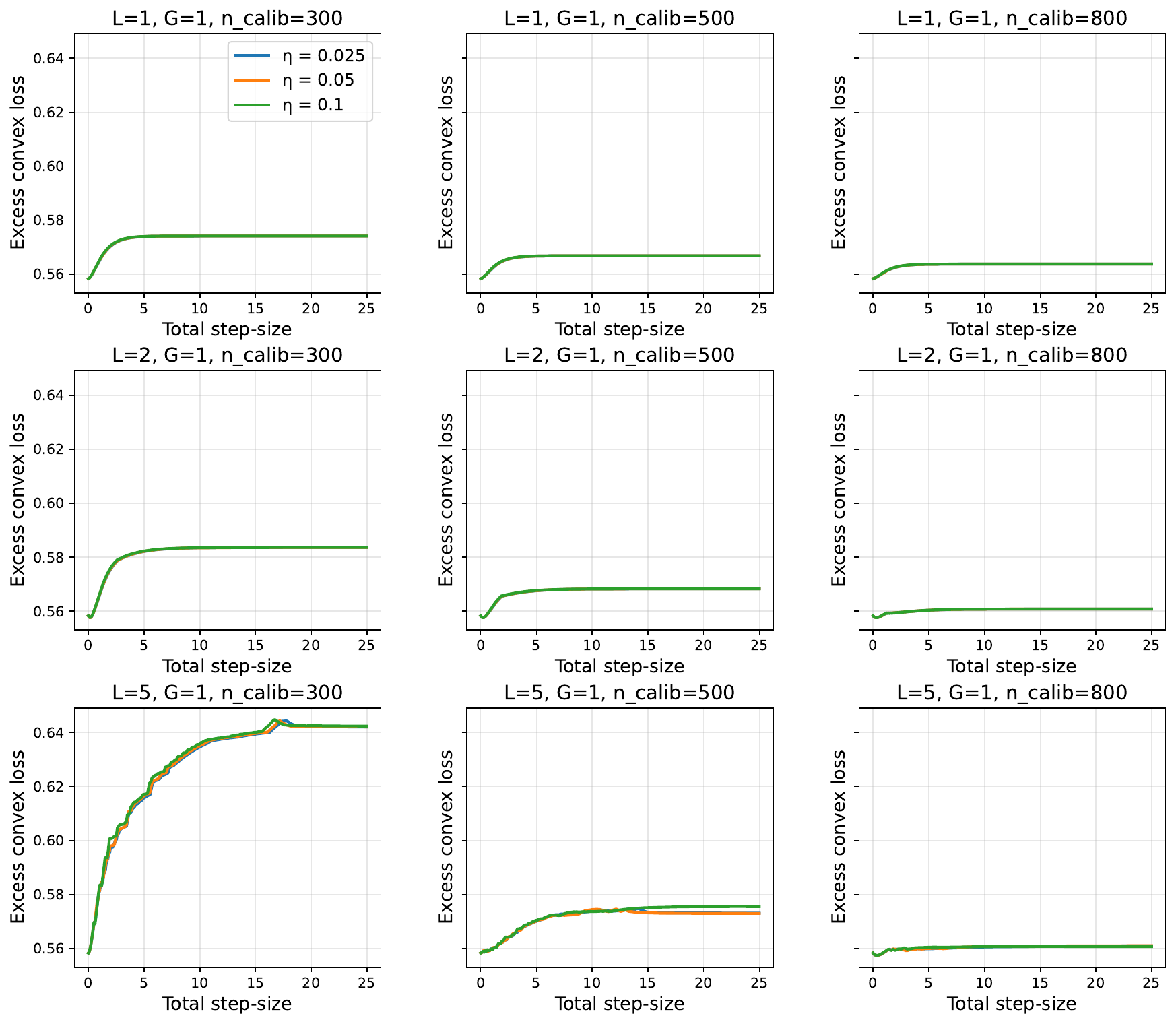}
    \subcaption*{(c)}
    \end{subfigure}
    \begin{subfigure}[b]{0.48\textwidth}
    \includegraphics[width=\linewidth]{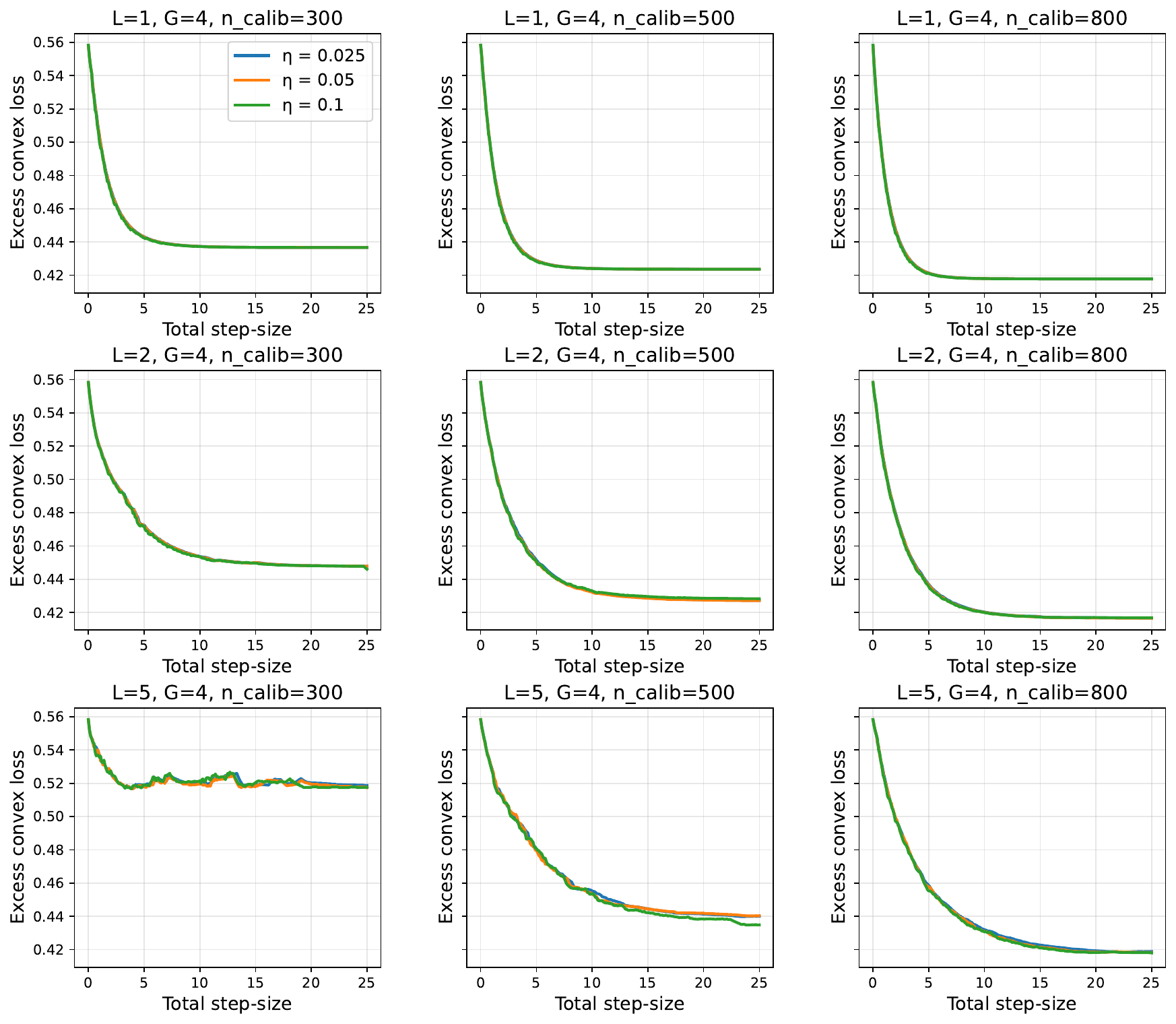}
    \subcaption*{(d)}
    \end{subfigure}
    
    \caption{ Excessive convex loss vs. total step-size. Initial predictor: linear model.  Auditor: constant. Panels correspond to different numbers of groups, buckets, and sample sizes.   
    (a)-(b): $| \mathcal{G}| = 1$ and $ |\mathcal{G}| = 4$: with all covariates. 
    (c)-(d): $|\mathcal{G}| = 1$ and $|\mathcal{G}| = 4$: excluding categorical covariates. 
    }
    \label{fig:excess_vs_stepsize-linear-constant}
\end{figure}

\begin{figure}[ht!]
    \centering
    \begin{subfigure}[b]{0.48\textwidth}
    \includegraphics[width=\linewidth]{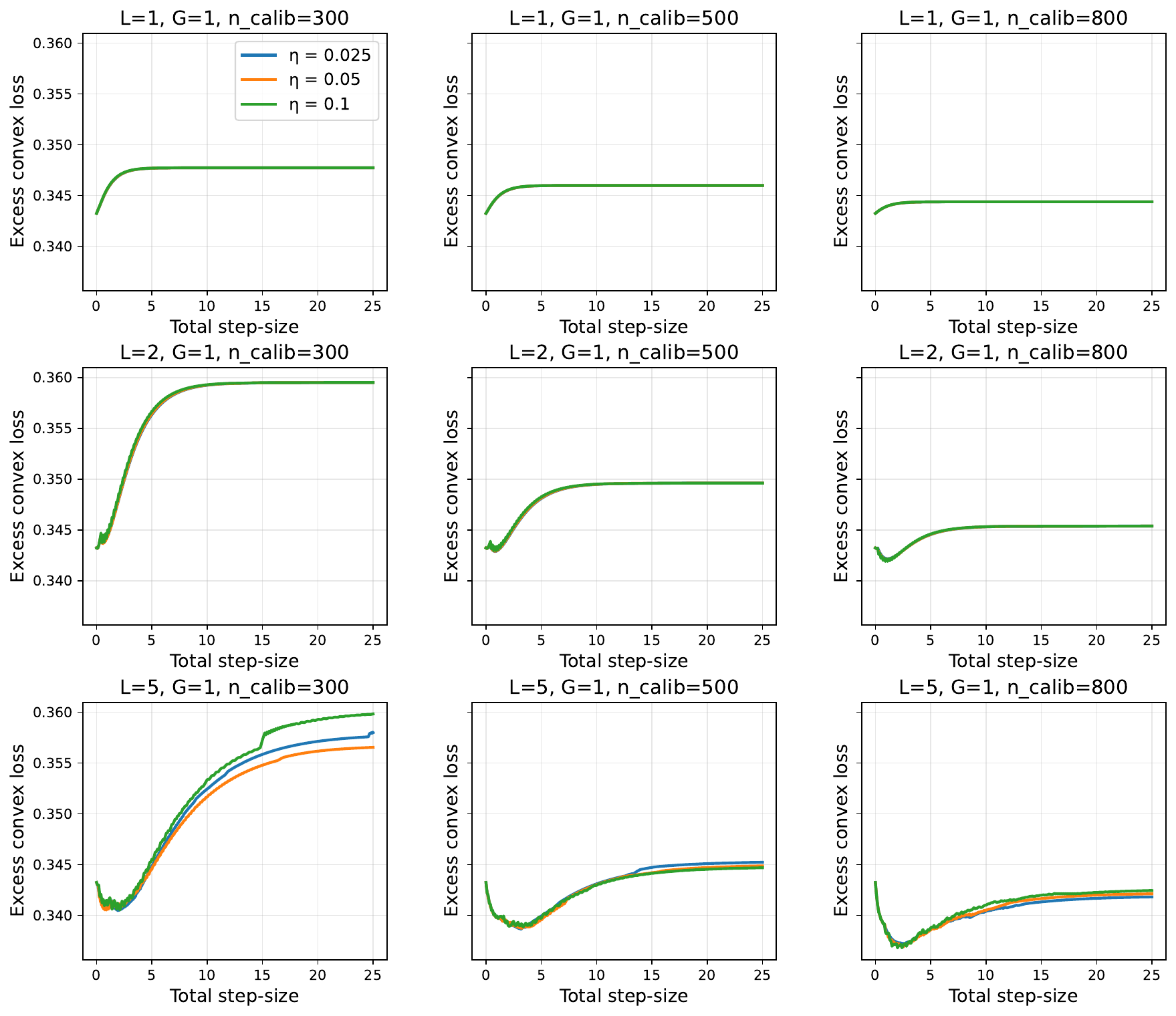}
    \subcaption*{(a)}
    \end{subfigure}
    \begin{subfigure}[b]{0.48\textwidth}
    \includegraphics[width=\linewidth]{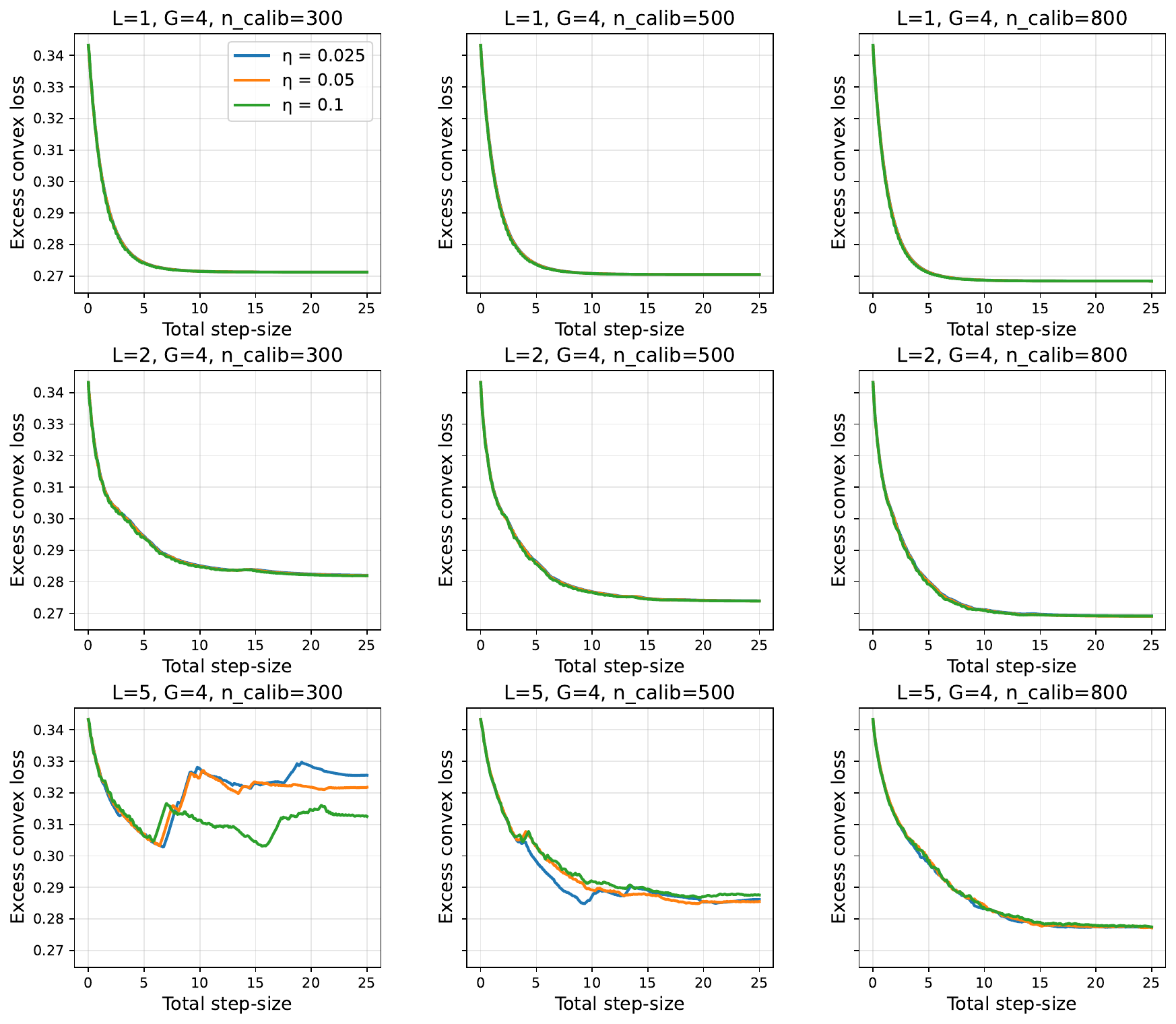}
    \subcaption*{(b)}
    \end{subfigure} 
    \begin{subfigure}[b]{0.48\textwidth}
   \includegraphics[width=\linewidth]{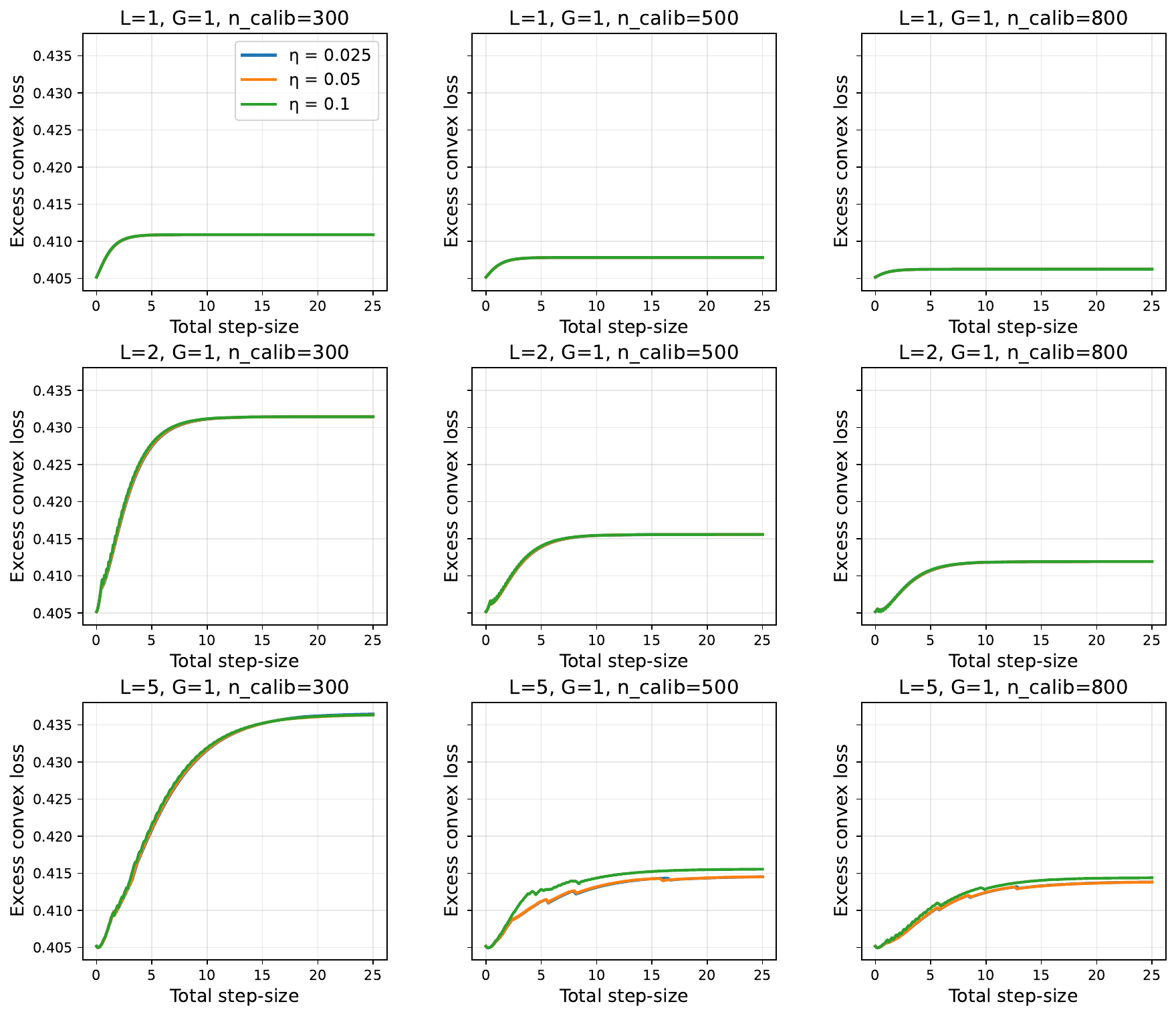}
    \subcaption*{(c)}
    \end{subfigure}
    \begin{subfigure}[b]{0.48\textwidth}
    \includegraphics[width=\linewidth]{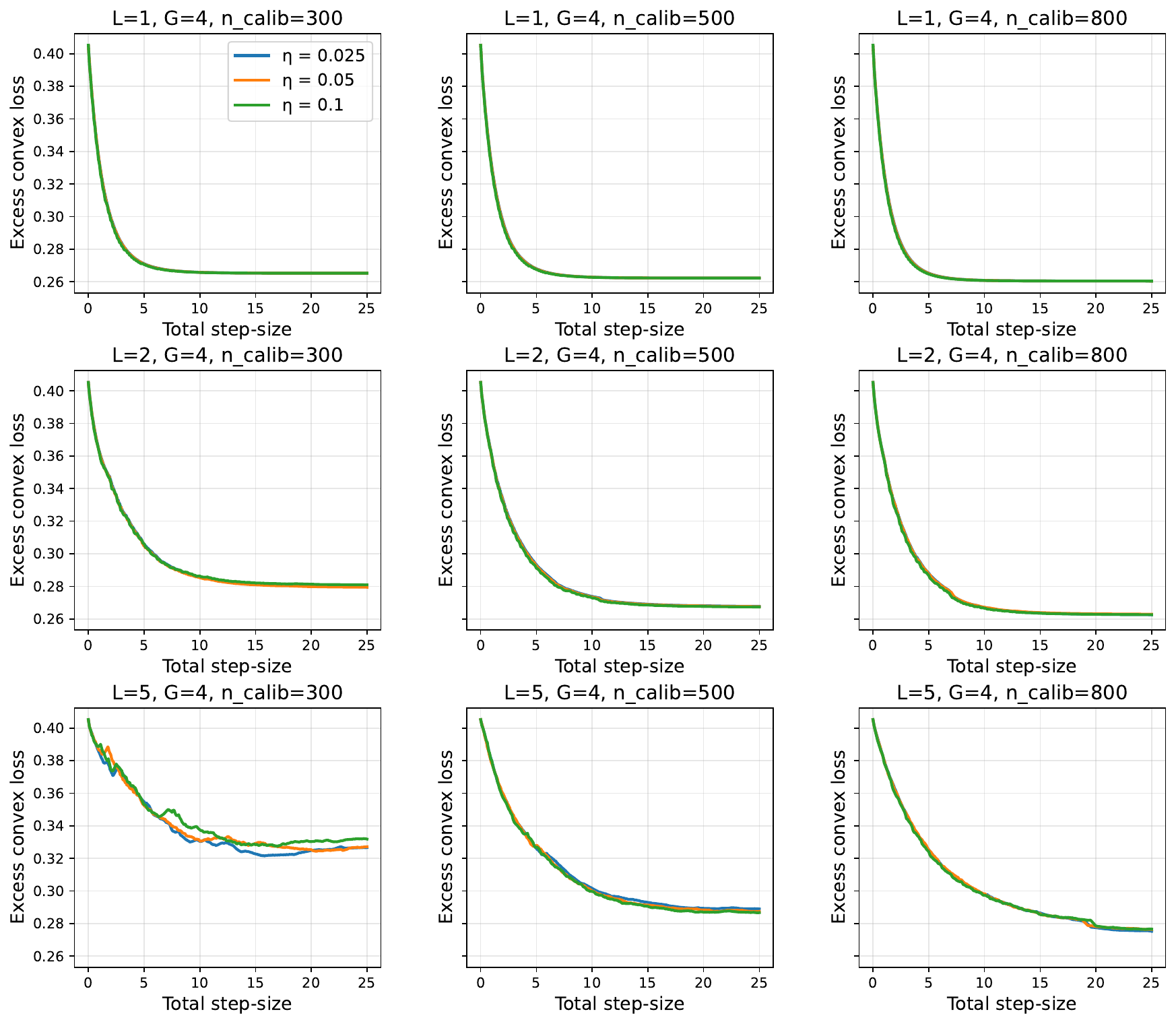}
    \subcaption*{(d)}
    \end{subfigure}
    \caption{ Excessive convex loss vs. total step-size. Initial predictor: random forest learner.  Auditor: constant refinement. Panels correspond to different numbers of groups, buckets, and calibration sample sizes. 
    (a)-(b): $|\mathcal{G}| = 1$ and $ |\mathcal{G}| = 4$: with all covariates. 
    (c)-(d): $|\mathcal{G}| = 1$ and $|\mathcal{G}| = 4$: excluding categorical covariates. 
    }
    \label{fig:excess_vs_stepsize-rf-constant}
\end{figure}

\begin{figure}[ht!]
    \centering
    \begin{subfigure}[b]{0.48\textwidth}
    \includegraphics[width=\linewidth]{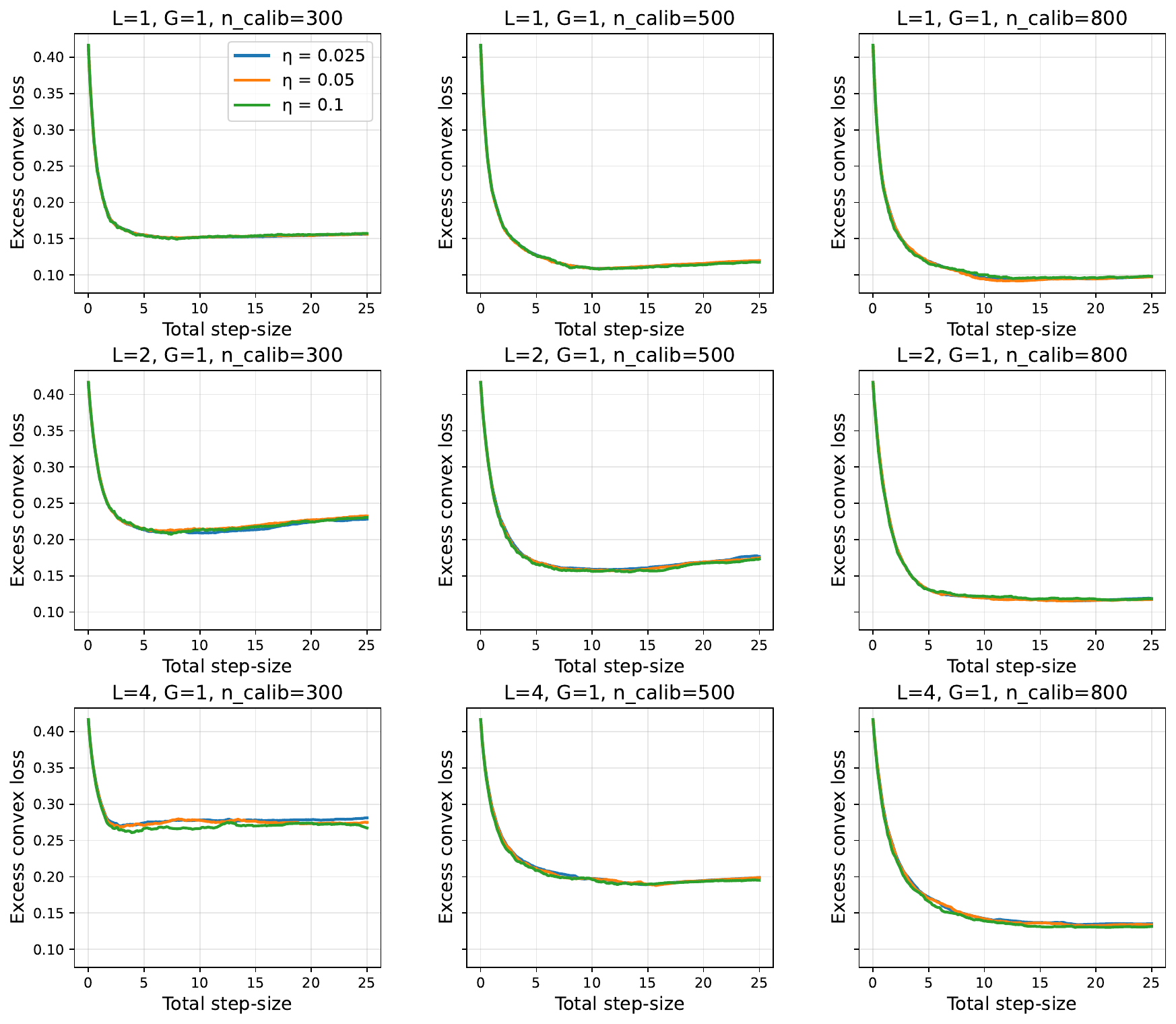}
    \subcaption*{(a)}
    \end{subfigure}
    \begin{subfigure}[b]{0.48\textwidth}
    \includegraphics[width=\linewidth]{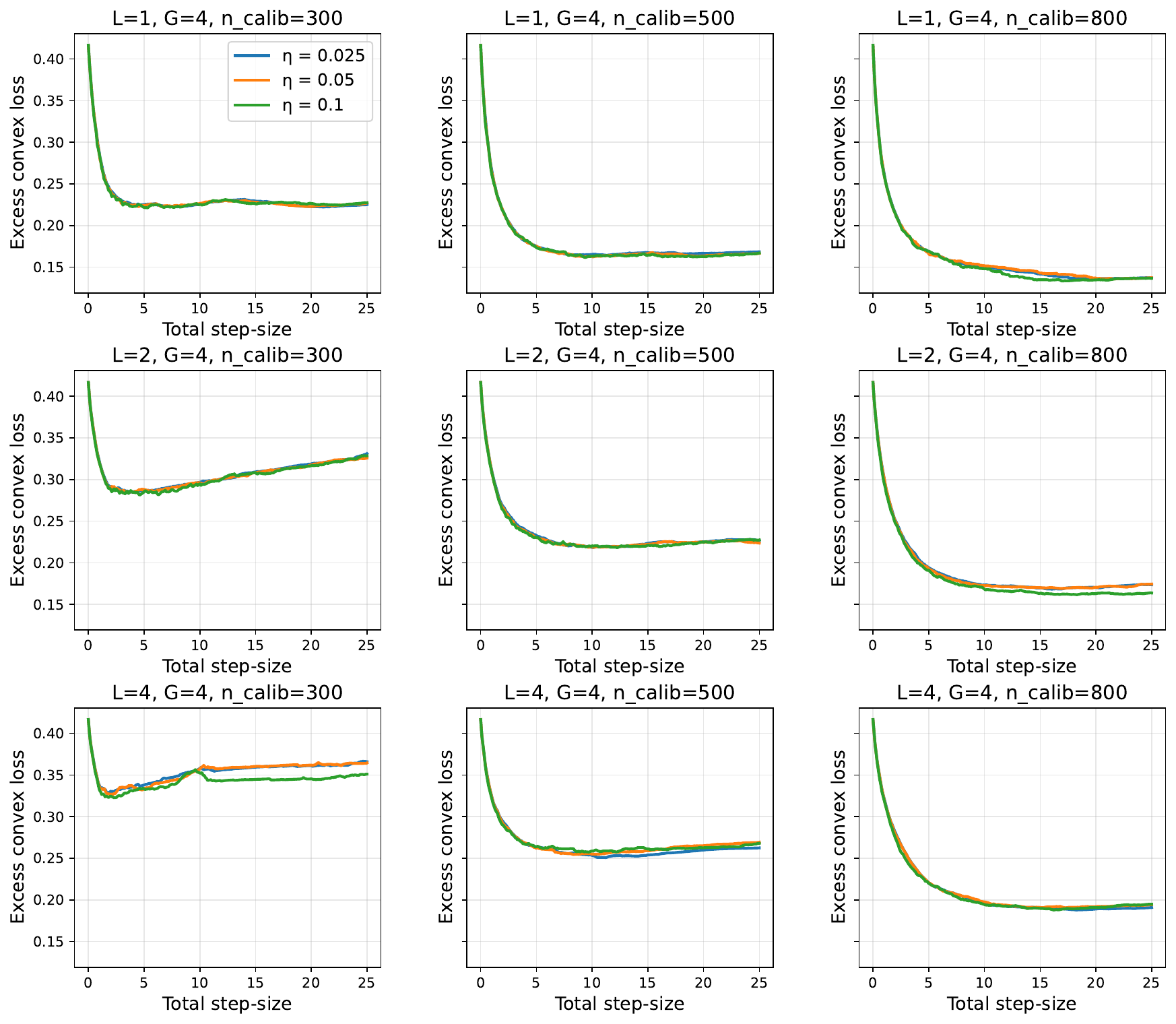}
    \subcaption*{(b)}
    \end{subfigure} 
    \begin{subfigure}[b]{0.48\textwidth}
   \includegraphics[width=\linewidth]{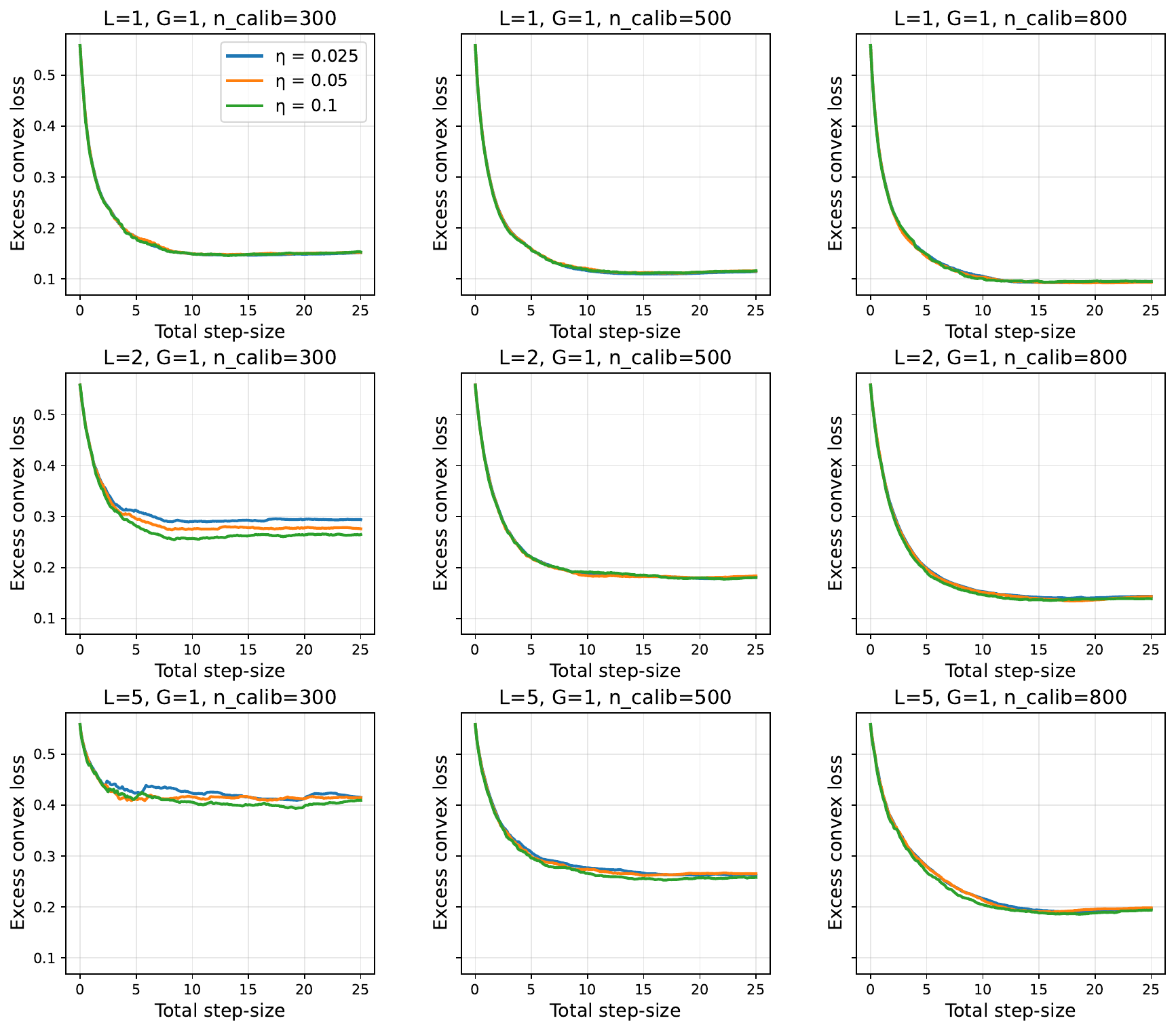}
    \subcaption*{(c)}
    \end{subfigure}
    \begin{subfigure}[b]{0.48\textwidth}
    \includegraphics[width=\linewidth]{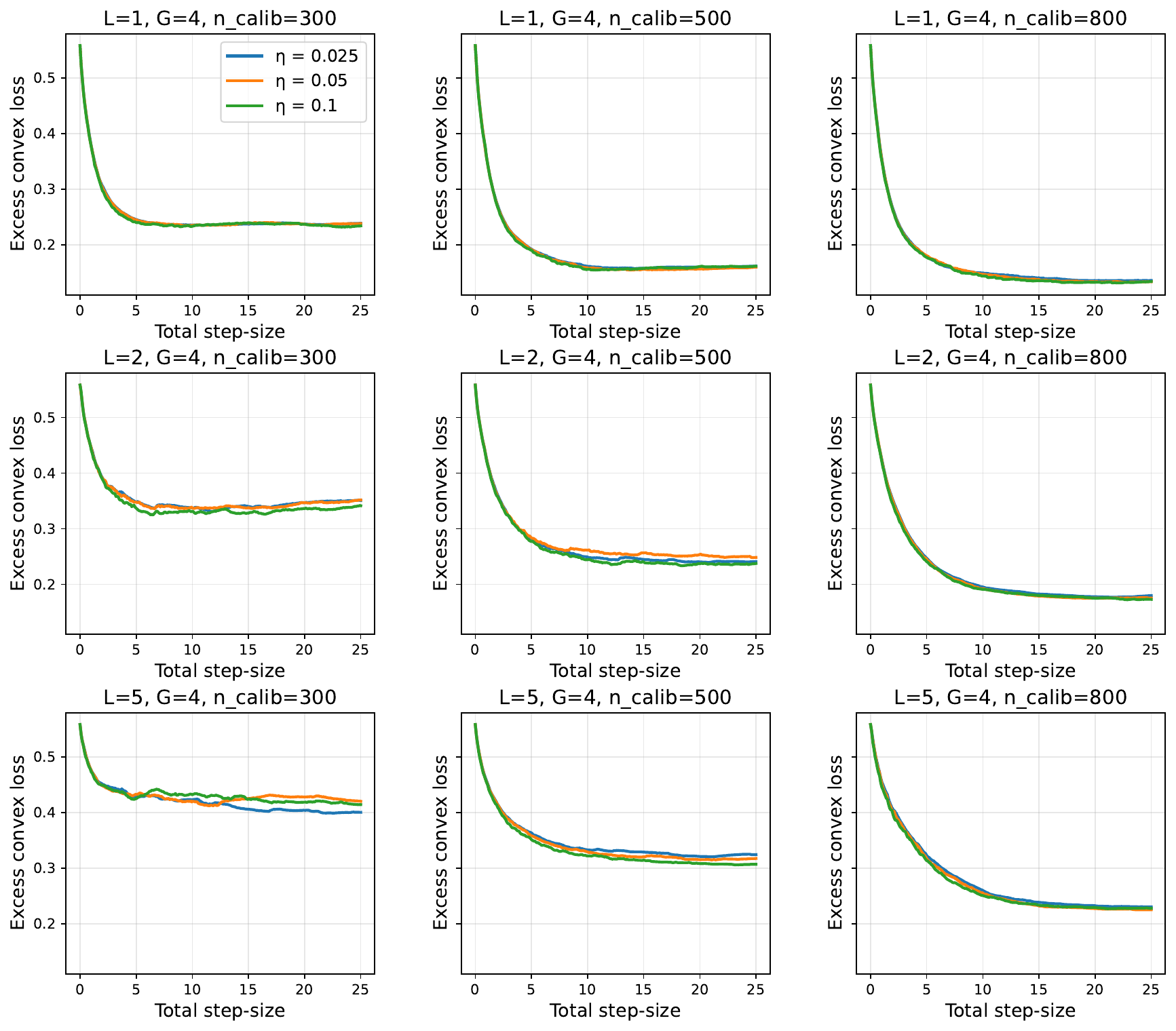}
    \subcaption*{(d)}
    \end{subfigure}
    \caption{ Excessive convex loss vs. total step-size. Initial predictor: linear model.  Auditor: tree-based learner. Panels correspond to different numbers of groups and buckets, as well as calibration sample sizes. (a)-(b): $|\mathcal{G}| = 1$ and $ |\mathcal{G}| = 4$: with all covariates. 
    (c)-(d): $|\mathcal{G}| = 1$ and $|\mathcal{G}| = 4$: excluding categorical covariates.
    }
    \label{fig:excess_vs_stepsize-lm-tree}
\end{figure}

\begin{figure}[ht!]
    \centering
    \begin{subfigure}[b]{0.48\textwidth}
    \includegraphics[width=\linewidth]{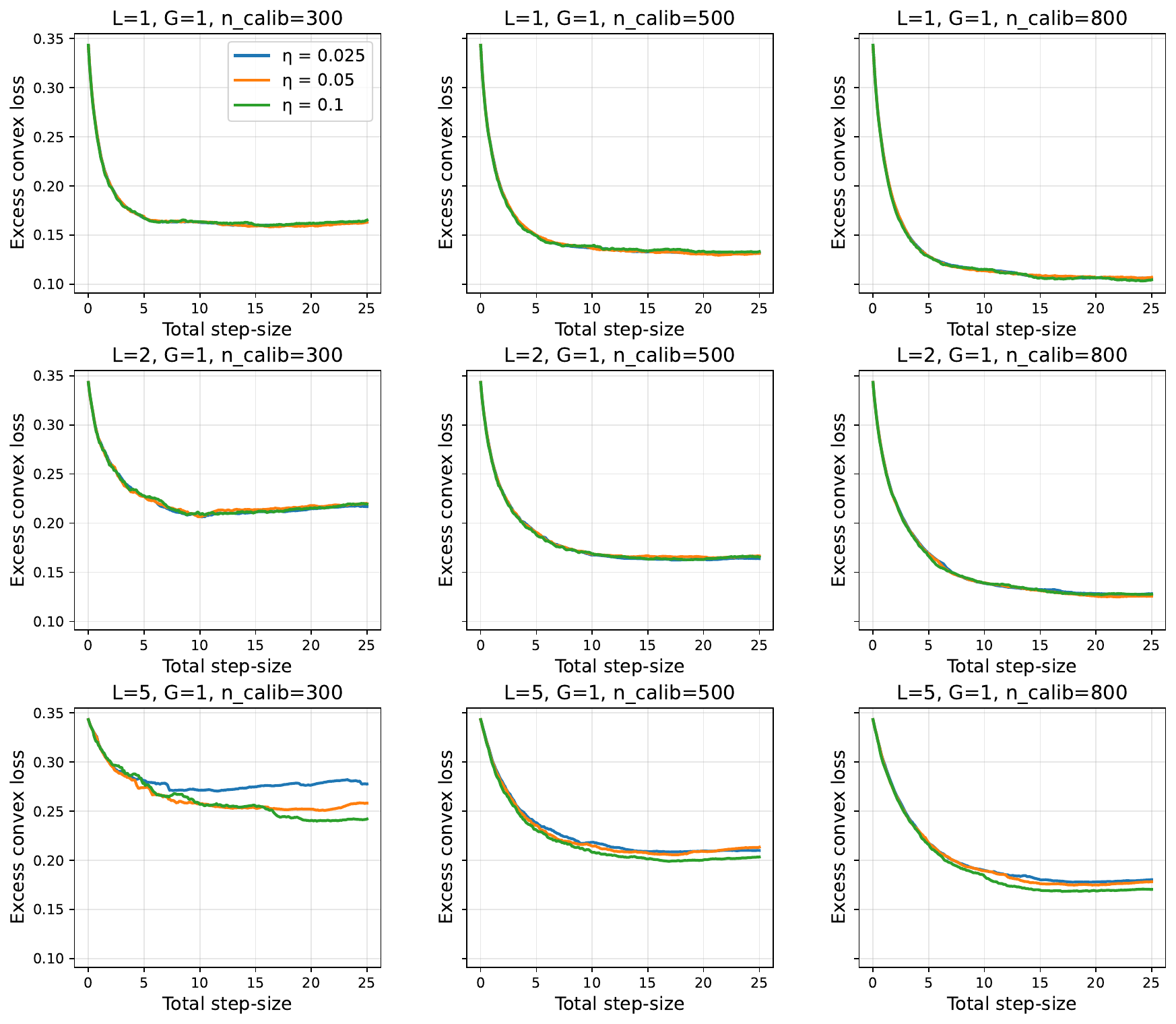}
    \subcaption*{(a)}
    \end{subfigure}
    \begin{subfigure}[b]{0.48\textwidth}
    \includegraphics[width=\linewidth]{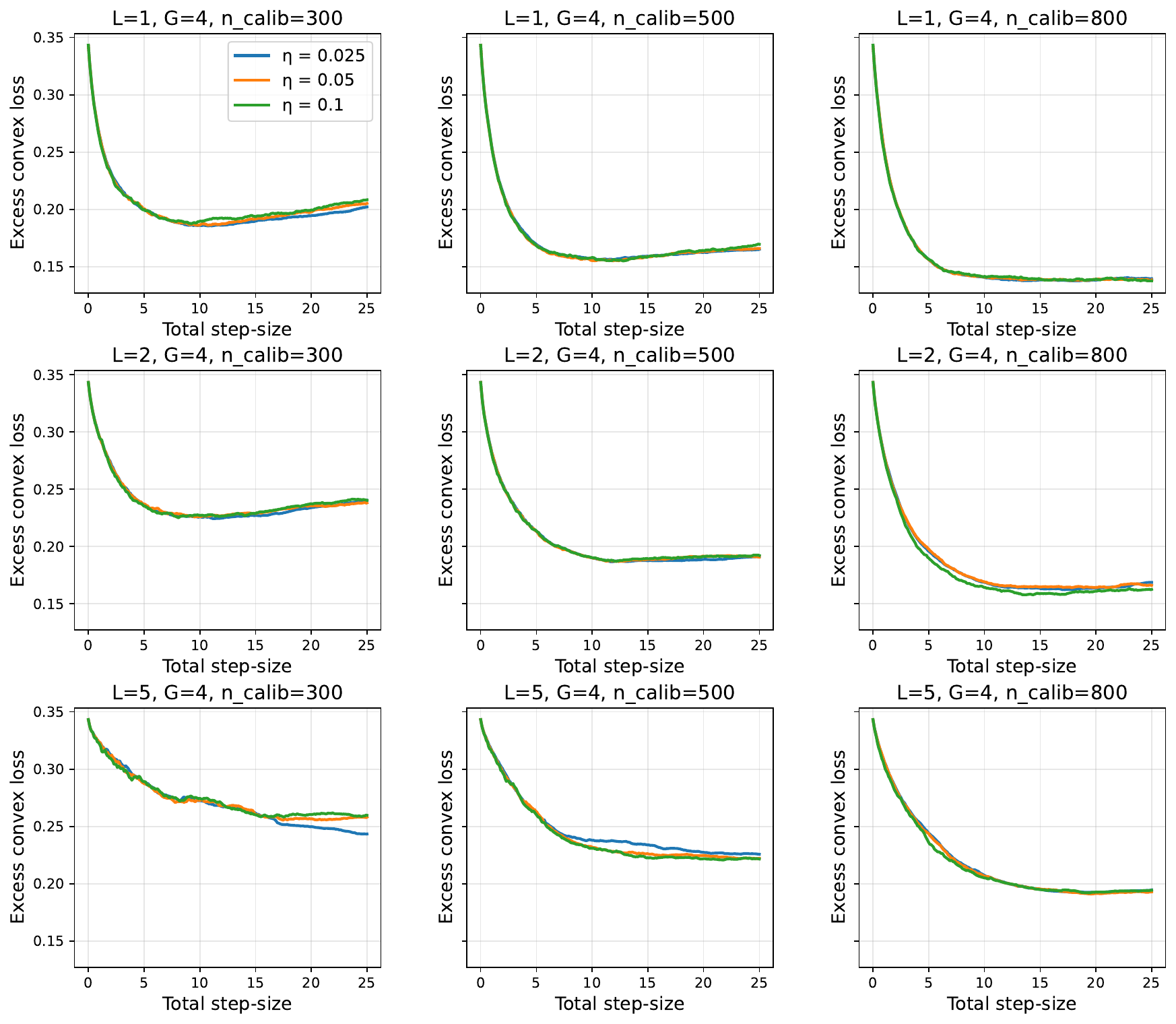}
    \subcaption*{(b)}
    \end{subfigure} 
    \begin{subfigure}[b]{0.48\textwidth}
   \includegraphics[width=\linewidth]{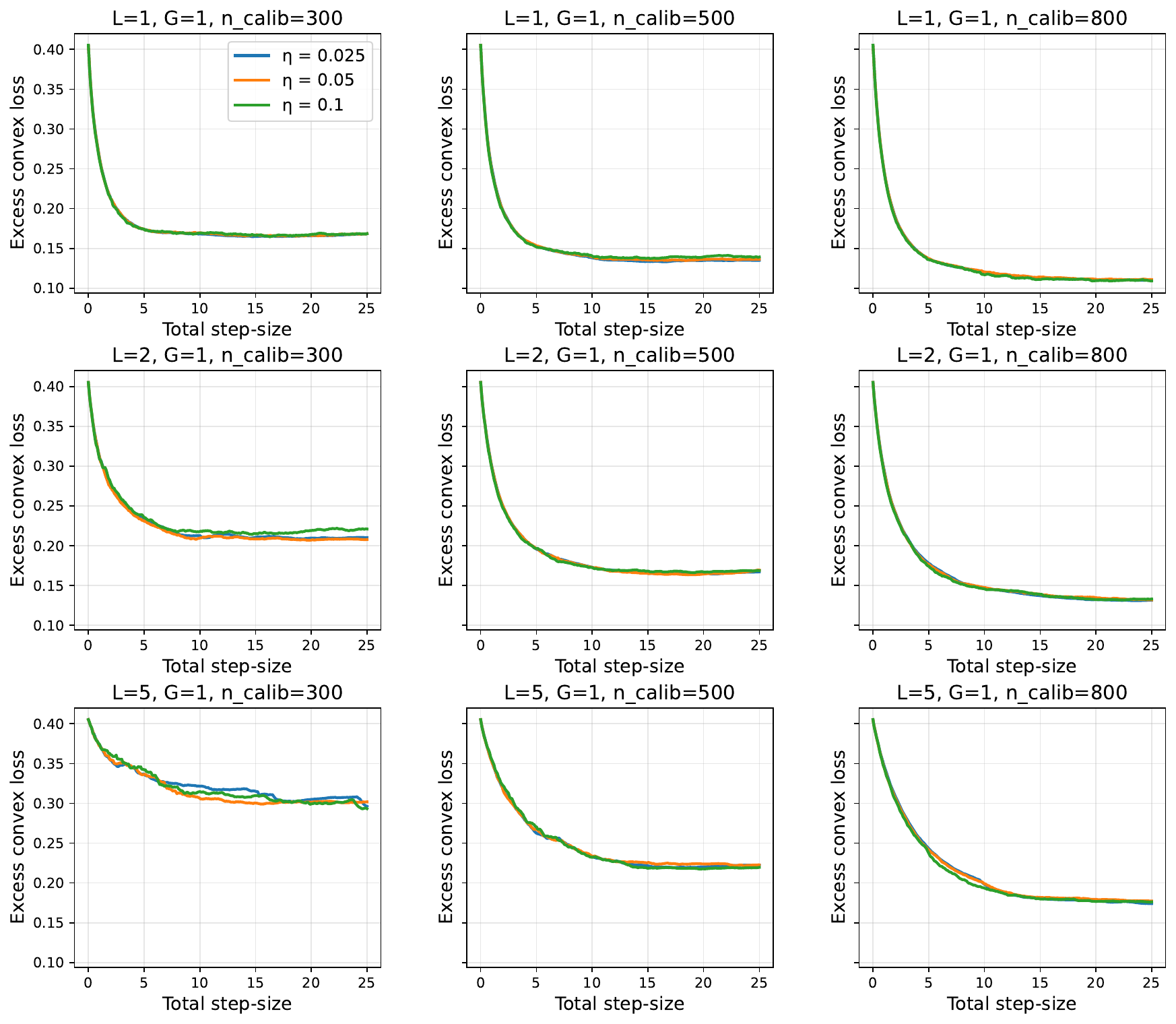}
    \subcaption*{(c)}
    \end{subfigure}
    \begin{subfigure}[b]{0.48\textwidth}
    \includegraphics[width=\linewidth]{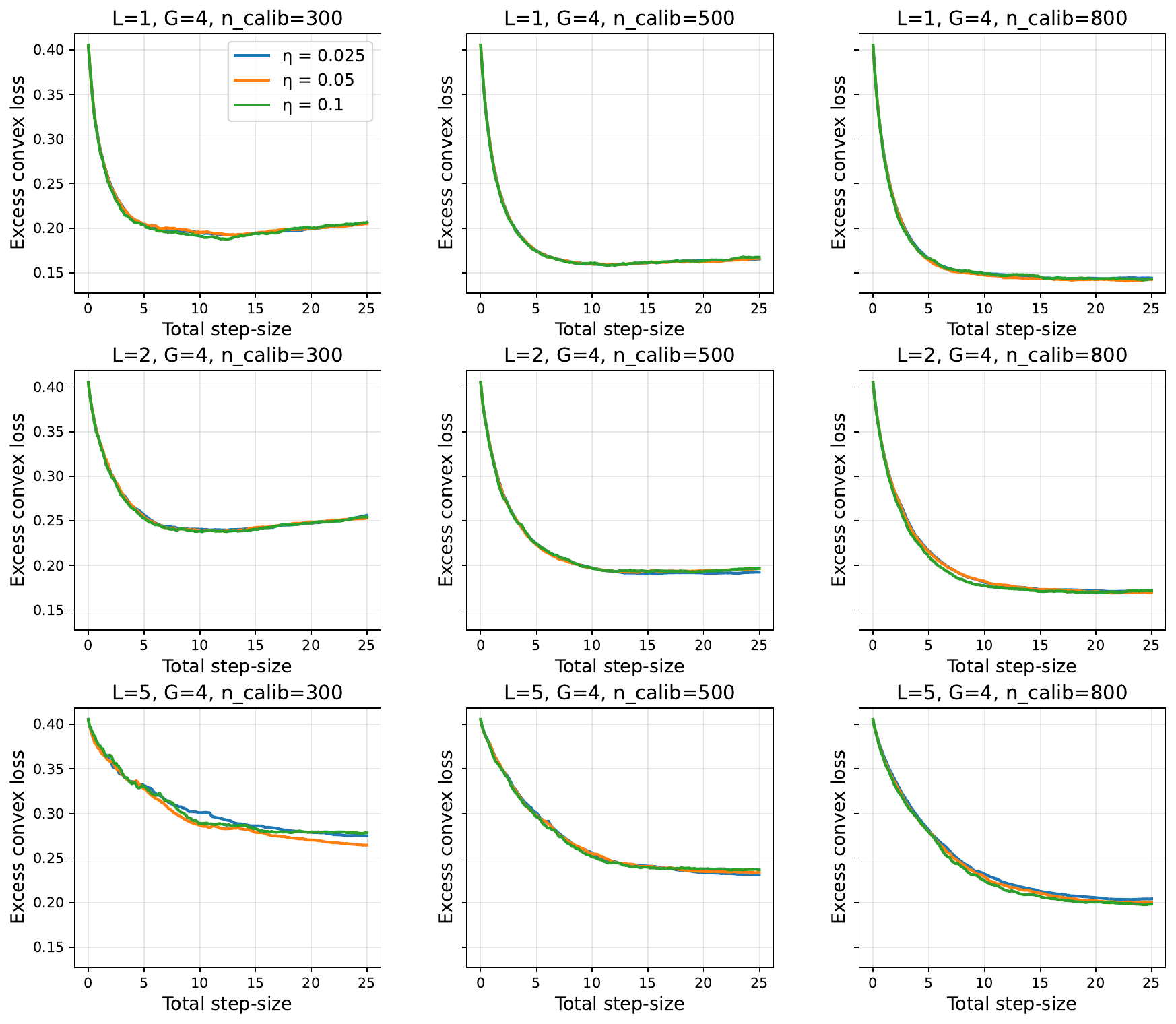}
    \subcaption*{(d)}
    \end{subfigure}
    \caption{ Excessive convex loss vs. total step-size. Initial predictor: random forest.  Auditor: tree-based learner. Panels correspond to different numbers of groups and buckets, as well as calibration sample sizes. (a)-(b): $|\mathcal{G}| = 1$ and $ |\mathcal{G}| = 4$: with all covariates. 
    (c)-(d): $|\mathcal{G}| = 1$ and $|\mathcal{G}| = 4$: excluding categorical covariates. }
    \label{fig:excess_vs_stepsize-rf-tree}
\end{figure}
The phenomena in the main numerical finding are also reflected in the excess risk curves based on MSE. When no group partition is used  ($|\mathcal{G}| = 1$), constant refinement yields negligible improvement in MSE. When group structure is introduced ($|\mathcal{G}| = 4$), constant refinement can reduce the MSE of the initial random forest estimator regardless of whether the categorical features $X^{(d)}$ are explicitly included in MCBoost. 
However, it improves the MSE of the initial linear model only if $X^{(d)}$ covariates are not taken into account.
With these variables included, the linear model already satisfies the corresponding moment conditions, leaving little room for further correction by a constant auditor. 
This highlights the interaction between the initial predictor and the auditor class: the effectiveness of post-processing depends on the extent to which the auditor introduces new directions not already captured by the initial model.   
\FloatBarrier

\subsection{Structural groups and weighted shifts under covariate shift}\label{apxsubsec:shift-construct} 
In Section~\ref{subsec:simu-shift}, we evaluate MCBoost beyond prespecified demographic groups by using structural subgroups and weighted shifts constructed from standardized continuous covariates $Z=(Z_1,Z_2)$, where $Z_j$ is obtained by $z$-scoring the corresponding component of $X^{(c)}$ using the reference sample. These targets are designed to probe interaction-dominated regions, high-curvature regions, and localized departures from the source distribution. Unless otherwise stated, all thresholds are computed from the sample after standardization.

We consider the following representative structural subgroups:
\begin{itemize}
    \item \emph{Interaction region} (\texttt{interaction\_reg}): observations for which the interaction term $Z_1Z_2$ is unusually large in magnitude, corresponding to the upper or lower tail of the interaction score.
    \item \emph{interaction\_neg} (\texttt{interaction\_neg}): observations with $Z_1Z_2 \le q_{0.20}(Z_1Z_2)$.
    \item \emph{Hard region} (\texttt{hard\_region}): observations with a large composite difficulty score
    $$
    d(Z)=0.5|Z_1| + 0.3 Z_2^2 + 0.2|Z_1Z_2|
    $$
    satisfying $d(Z) \ge q_{0.85}(d(Z))$.
\end{itemize}

We also consider the following weighted shifts. In each case, we form raw weights, clip them to the interval $[0.1, 10]$, and renormalize them to have empirical mean $1$: 
\begin{itemize}
    \item \emph{Curvature tilt} (\texttt{curvature\_tilt}): $w(X) \propto \exp(0.4 Z_2^2)$. 
    \item \emph{Hard mixed tilt} (\texttt{hard\_mixed\_tilt}):
    $$
    w(X) \propto \exp\big(0.30|Z_1| + 0.25 Z_2^2 + 0.25 Z_1Z_2 + 0.25 X_6 + 0.25 X_7\big).
    $$
    \item \emph{Local bump} (\texttt{local\_bump}): a Gaussian bump centered at $(1.2,-1.0)$ in the $(Z_1,Z_2)$ plane with bandwidth $0.8$, namely
    $$
    w(X) \propto \exp\left(-\frac{(Z_1-1.2)^2 + (Z_2+1.0)^2}{2\times 0.8^2}\right).
    $$
\end{itemize}

\subsection{Early-stopping with different auditors in quantile regression}

\begin{figure}[ht!]
    \centering
    \begin{subfigure}[b]{0.48\textwidth}
    \includegraphics[width=\linewidth]{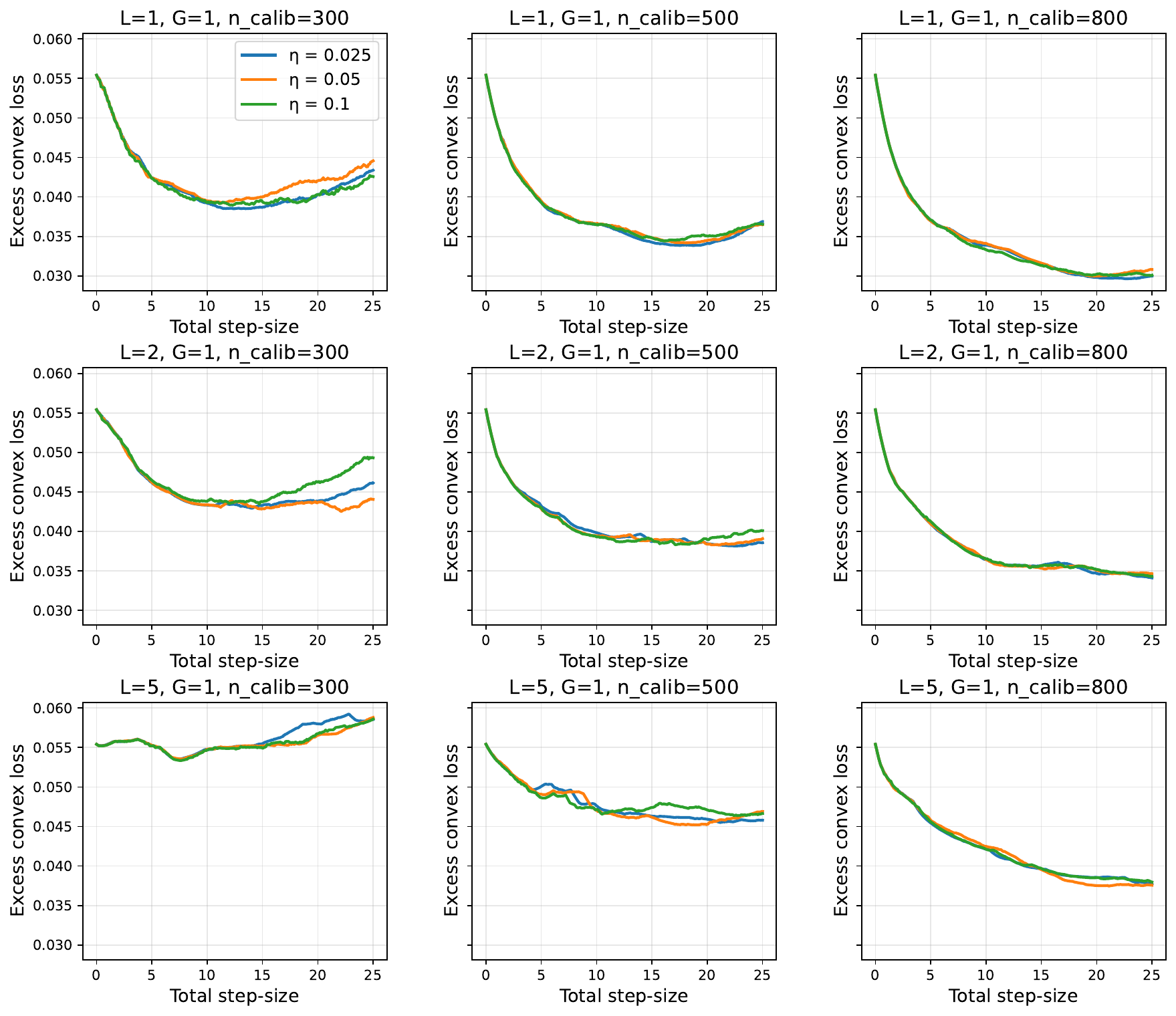}  
    \subcaption*{(a)}
    \end{subfigure}
    \begin{subfigure}[b]{0.48\textwidth}
    \includegraphics[width=\linewidth]{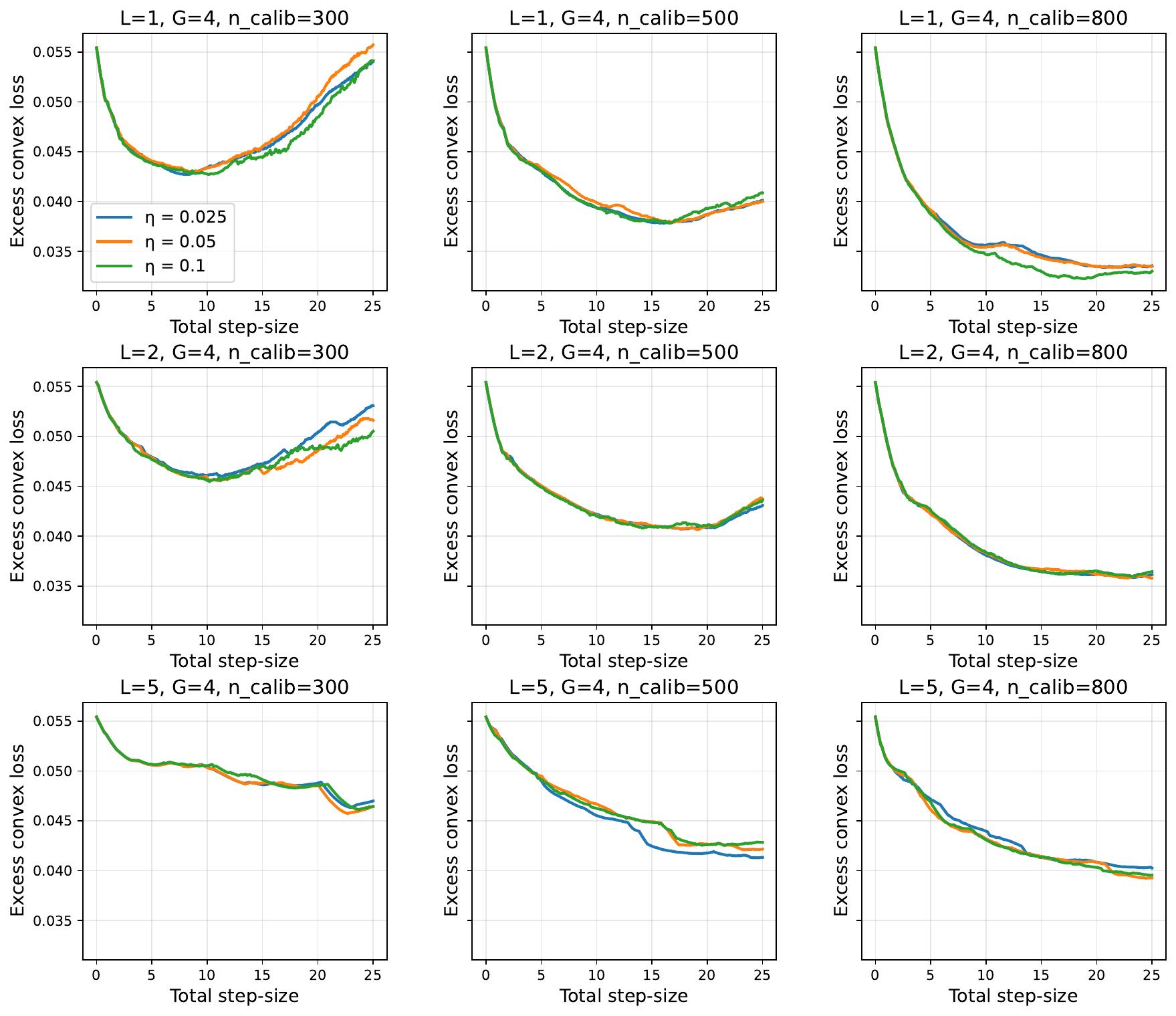}
    \subcaption*{(b)}
    \end{subfigure} 
    \begin{subfigure}[b]{0.48\textwidth}
    \includegraphics[width=\linewidth]{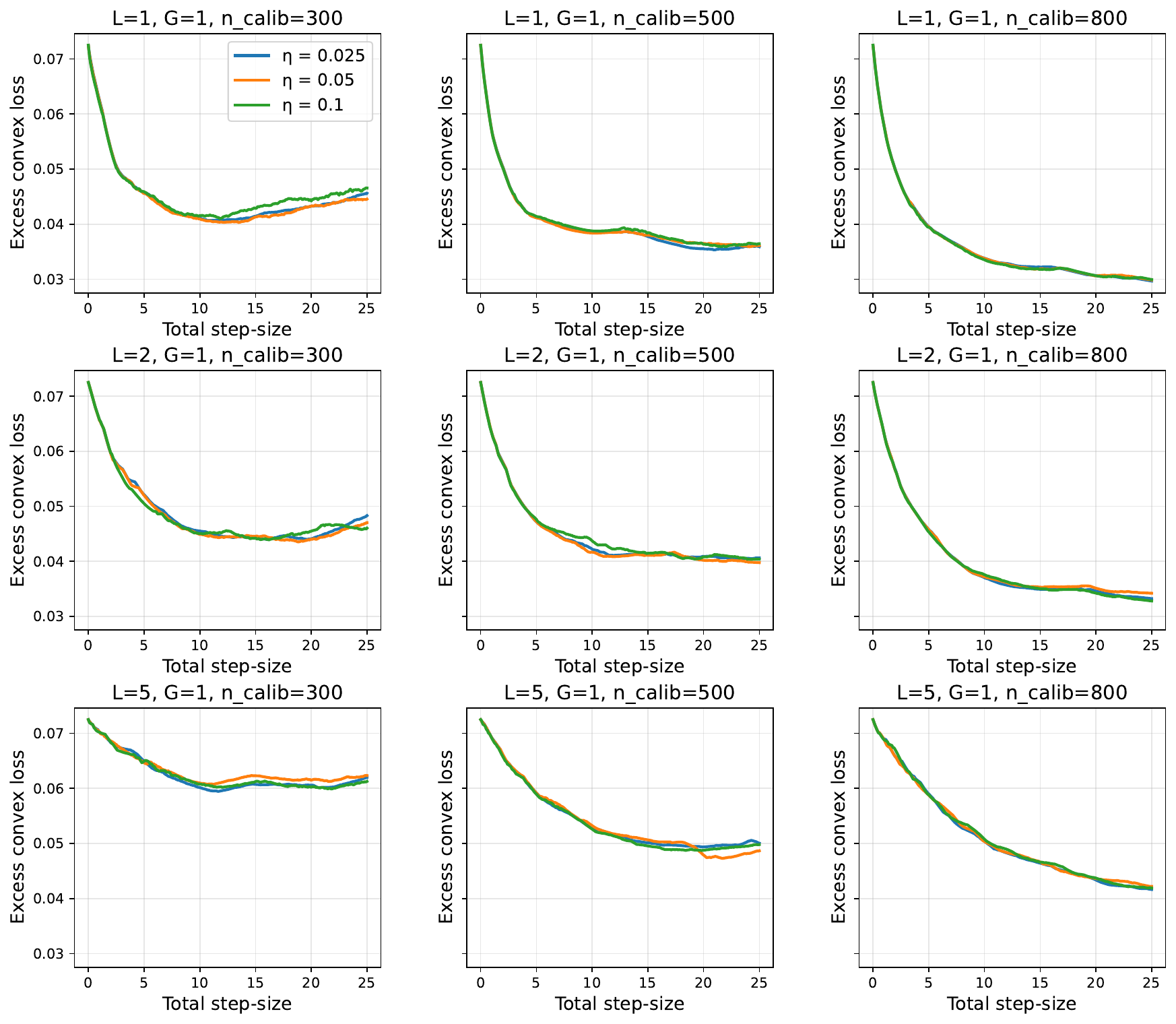} 
    \subcaption*{(c)}
    \end{subfigure}
    \begin{subfigure}[b]{0.48\textwidth}
    \includegraphics[width=\linewidth]{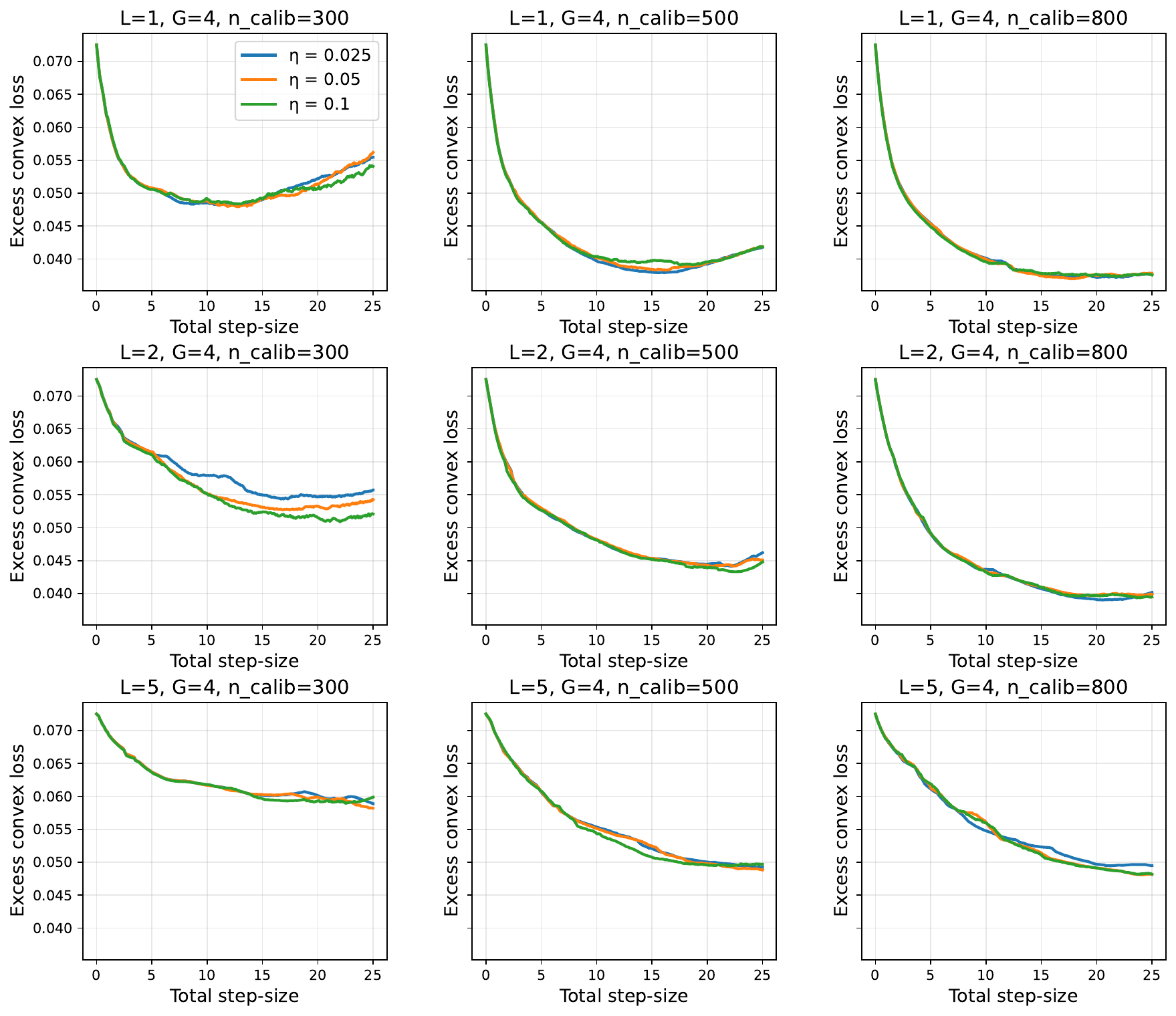} 
    \subcaption*{(d)}
    \end{subfigure}
    
    \caption{ Excessive convex pinball loss vs. total step-size. Initial predictor: quantile random forest model. Auditors: Decision tree learners. Panels correspond to different numbers of groups and buckets, as well as calibration sample sizes. 
    (a)-(b): $|\mathcal{G}| = 1$ and $ |\mathcal{G}| = 4$: with all covariates. 
    (c)-(d): $|\mathcal{G}| = 1$ and $|\mathcal{G}| = 4$: excluding categorical covariates.
    }
    \label{fig:qr-excess_vs_stepsize-rf-tree}
\end{figure}

\begin{figure}[ht!]
    \centering
    \begin{subfigure}[b]{0.48\textwidth}
    \includegraphics[width=\linewidth]{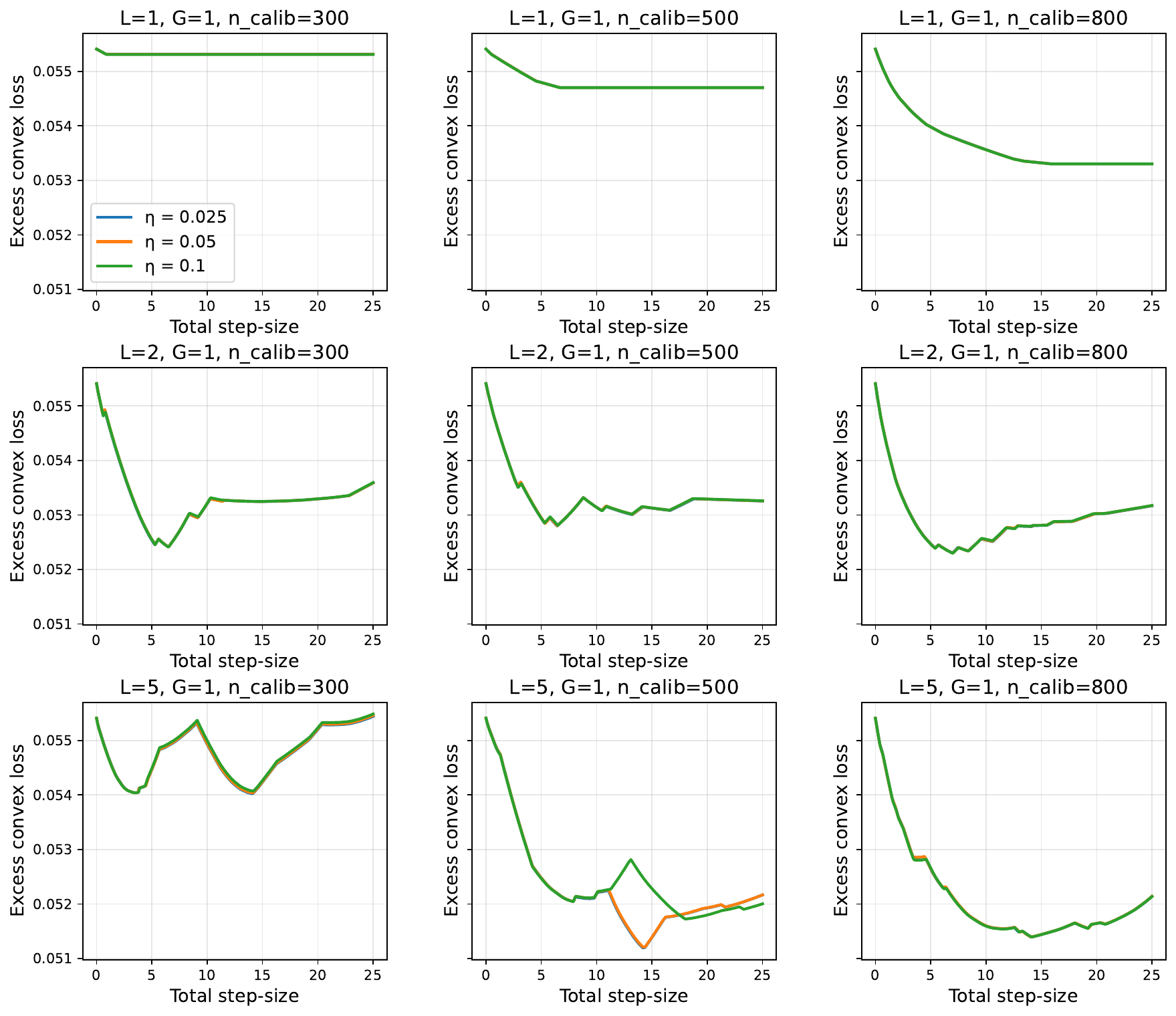}  
    \subcaption*{(a)}
    \end{subfigure} 
    \begin{subfigure}[b]{0.48\textwidth}
    \includegraphics[width=\linewidth]{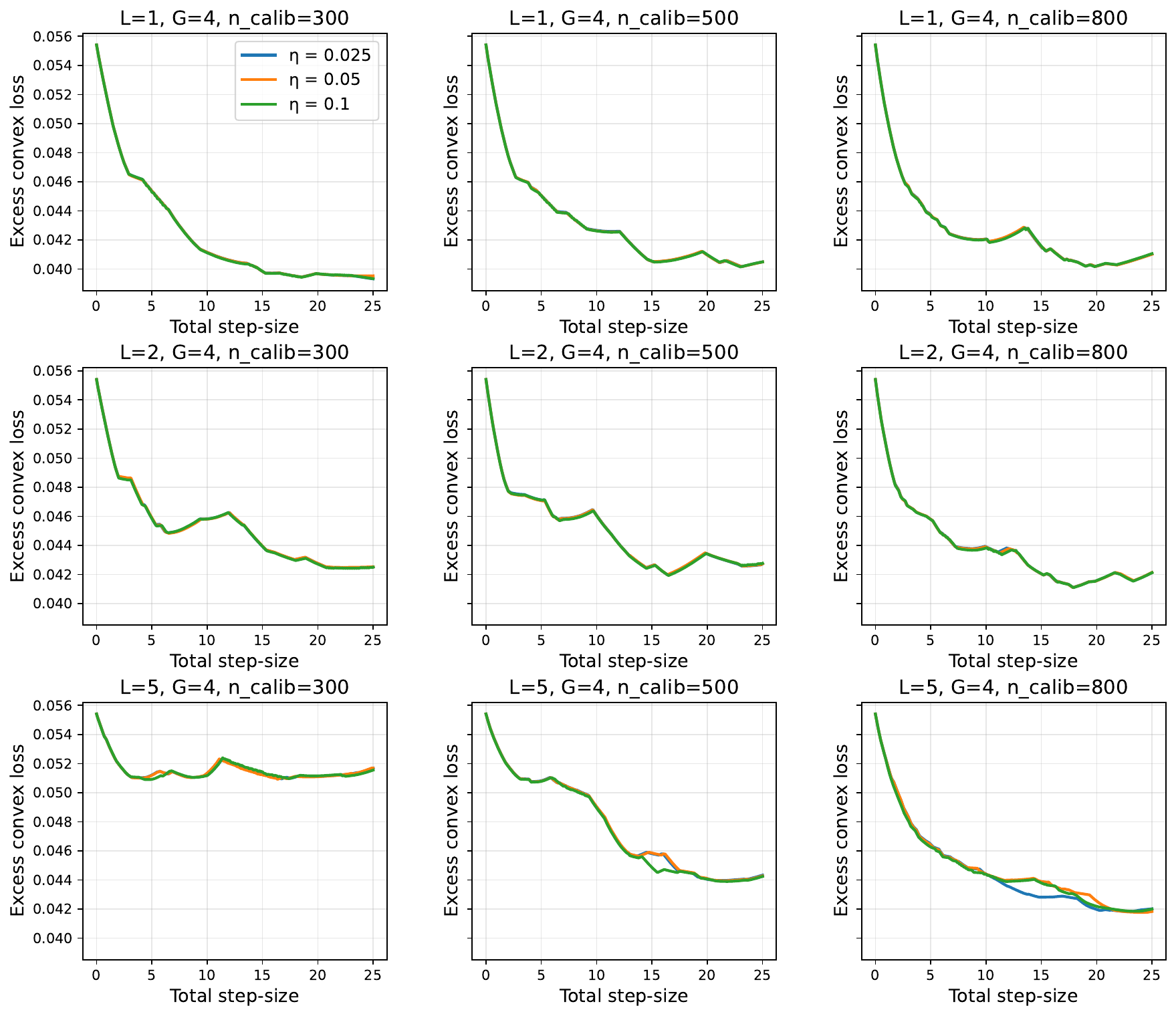}  
    \subcaption*{(b)}
    \end{subfigure} 
    \begin{subfigure}[b]{0.48\textwidth}
    \includegraphics[width=\linewidth]{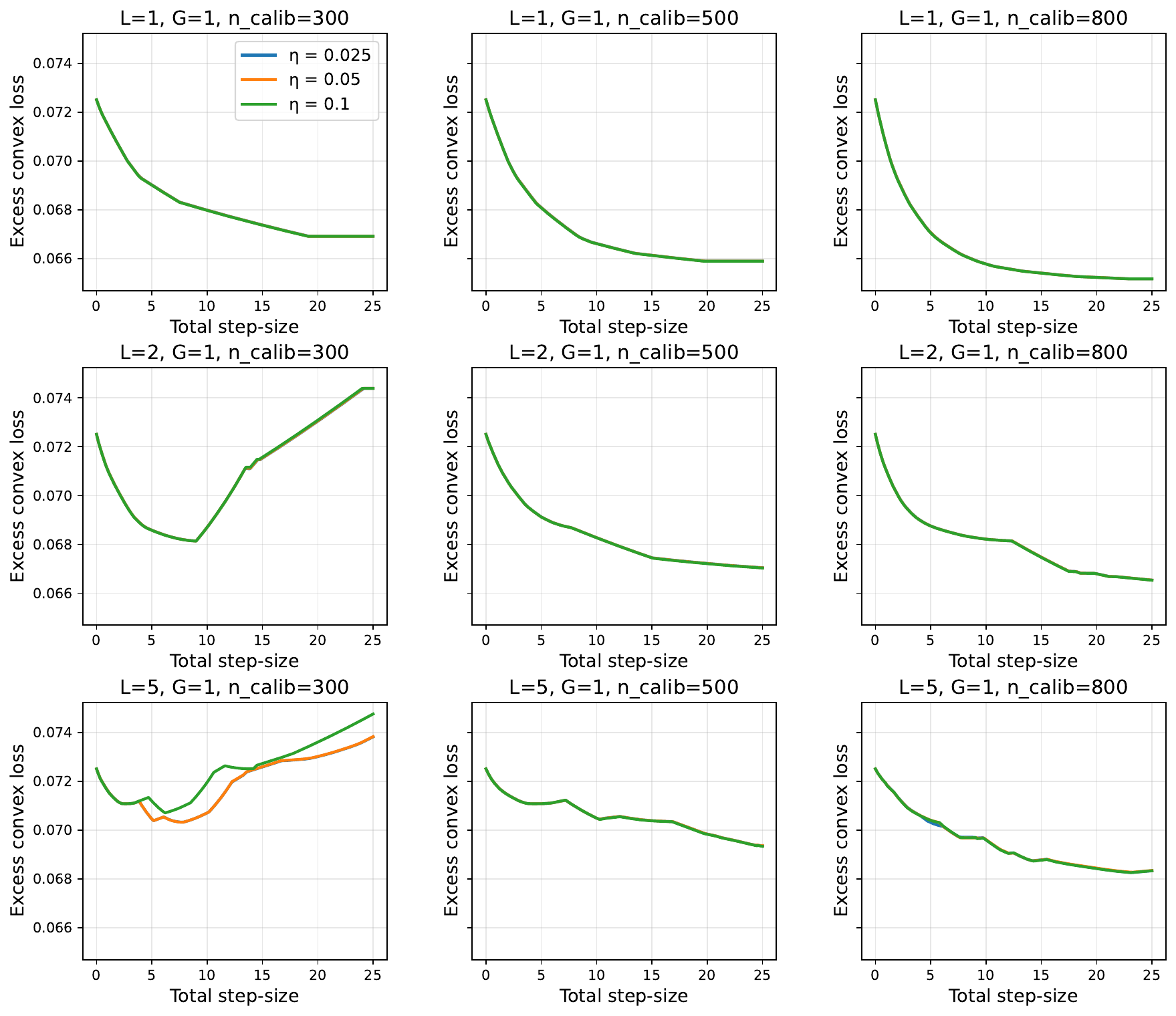}  
    \subcaption*{(c)}
    \end{subfigure} 
    \begin{subfigure}[b]{0.48\textwidth}
    \includegraphics[width=\linewidth]{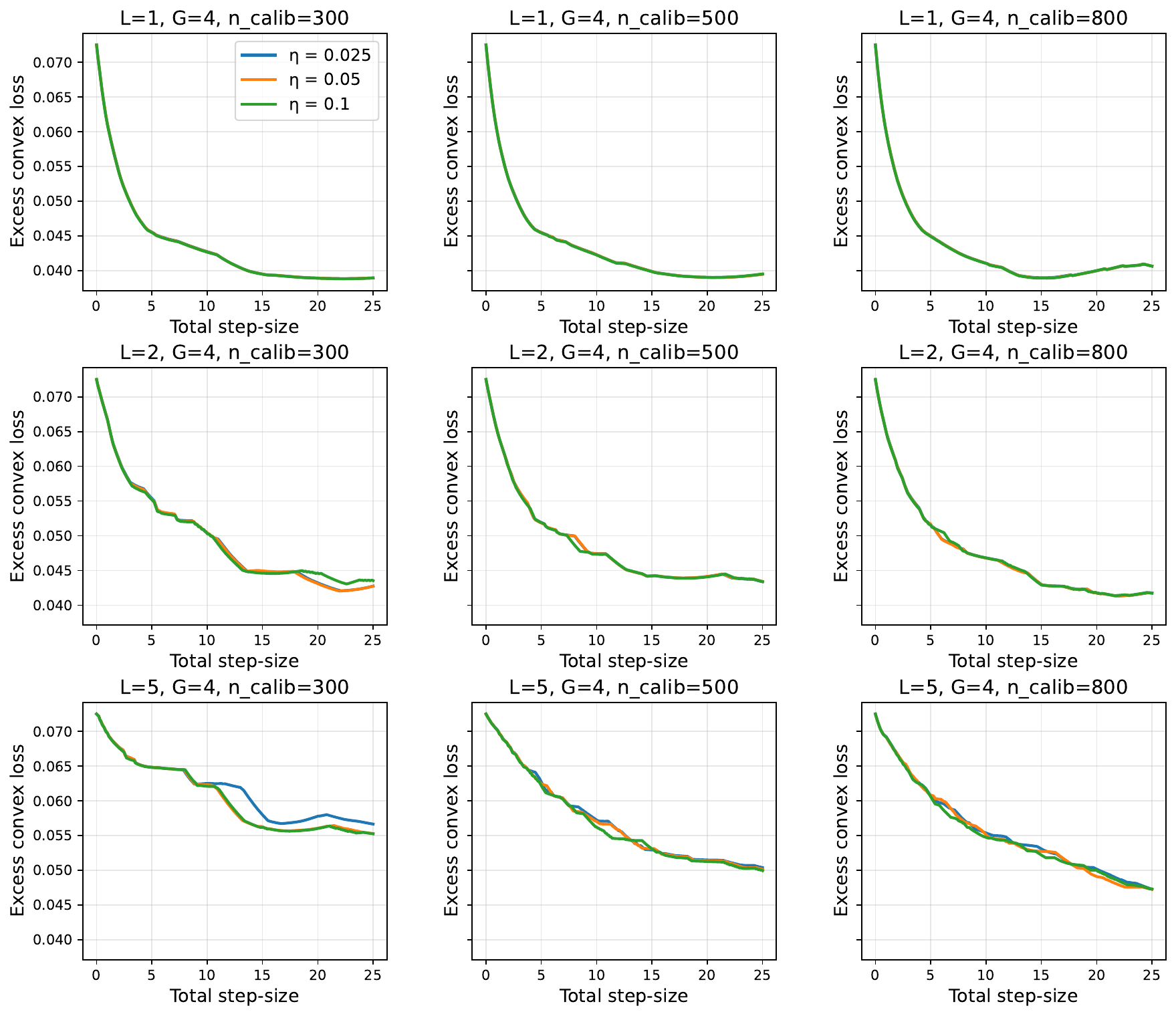}  
    \subcaption*{(c)}
    \end{subfigure} 
    
    \caption{
    Excessive convex pinball loss vs. total step-size. Initial predictor: quantile random forest model. Auditors: constant. Panels correspond to different numbers of groups and buckets, as well as calibration sample sizes.  
    (a)-(b): $|\mathcal{G}| = 1$ and $ |\mathcal{G}| = 4$: with all covariates. 
    (c)-(d): $|\mathcal{G}| = 1$ and $|\mathcal{G}| = 4$: excluding categorical covariates.
    }
    \label{fig:qr-excess_vs_stepsize-rf-constant}
\end{figure}  

\subsection{Quantile coverage} 
\begin{figure}[ht!]
    \begin{subfigure}[b]{\textwidth}
        \includegraphics[width=\textwidth]{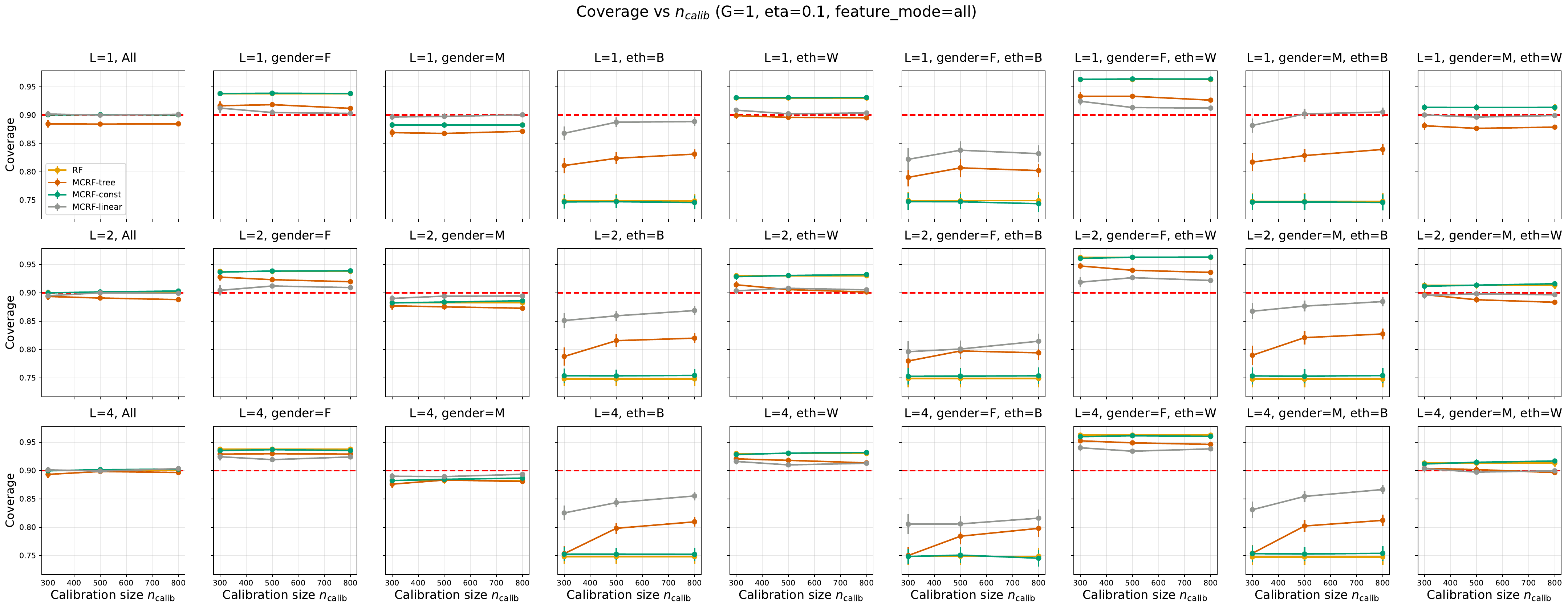}
    \end{subfigure}
    \begin{subfigure}[b]{\textwidth}
        \includegraphics[width=\textwidth]{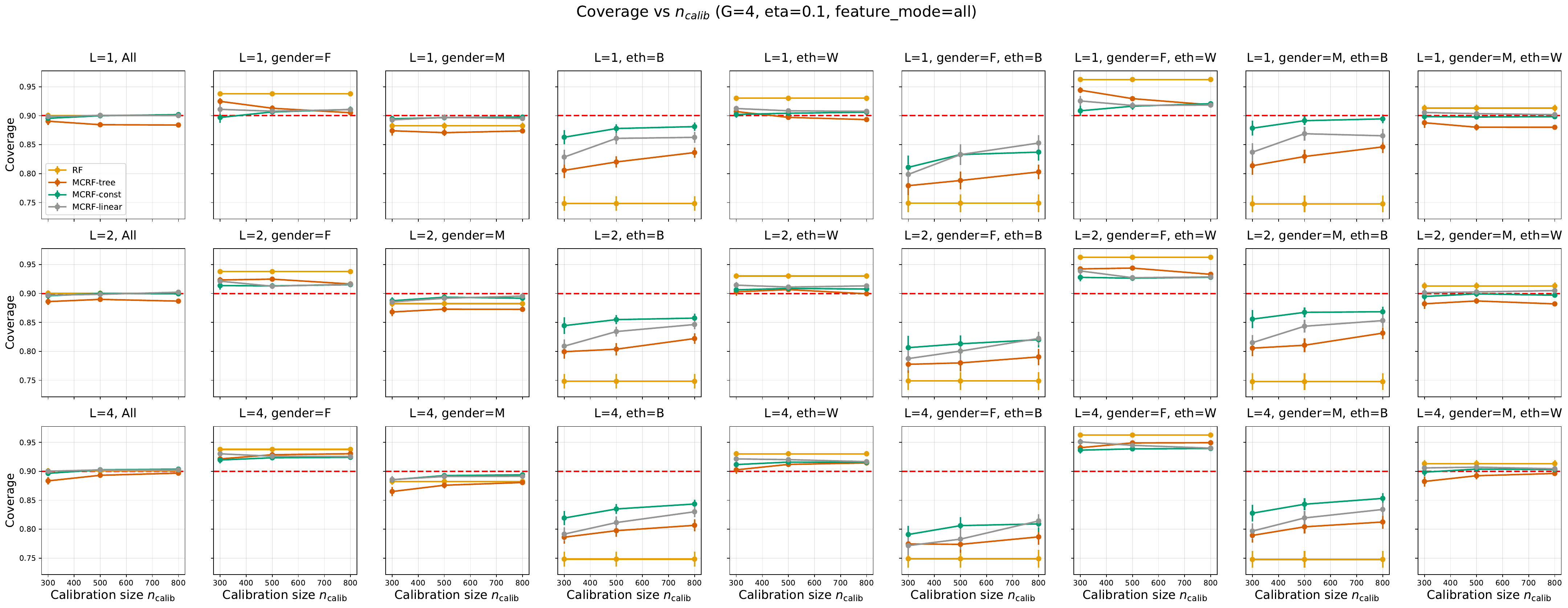}
    \end{subfigure}
    \begin{subfigure}[b]{\textwidth}
        \includegraphics[width=\textwidth]{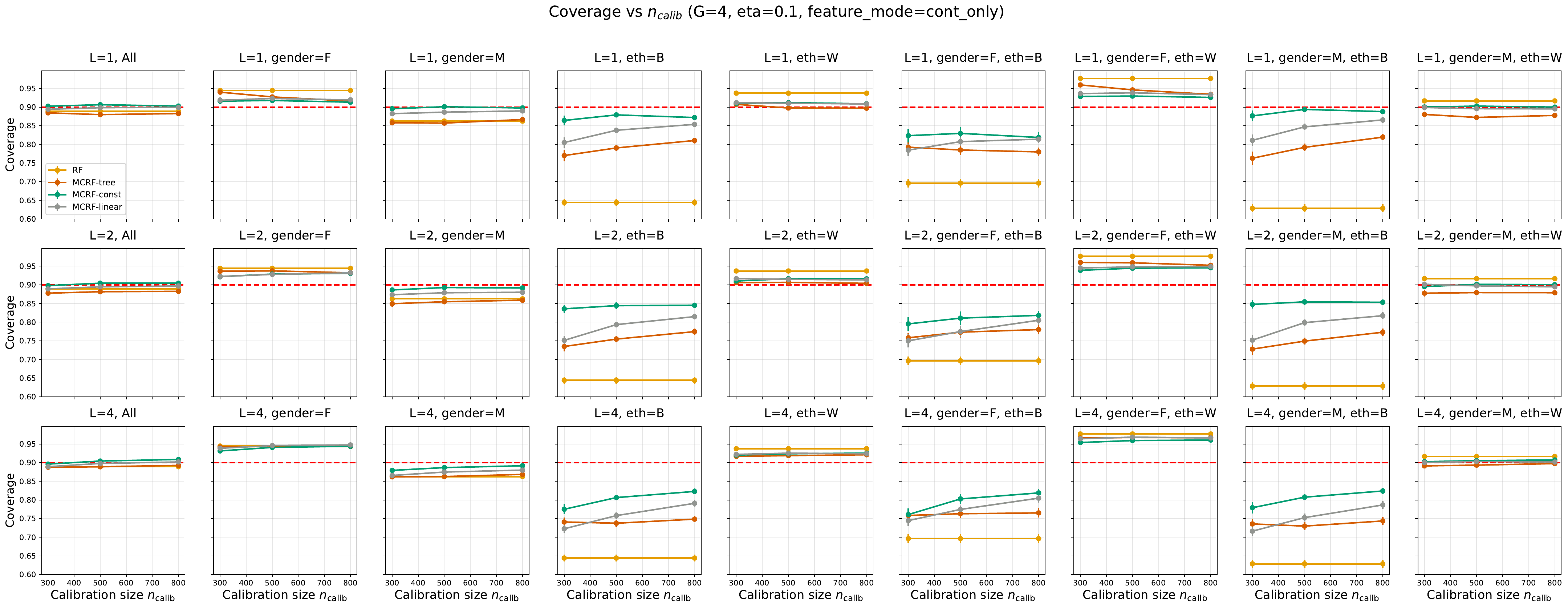}
    \end{subfigure}
    \caption{Coverage across subpopulations  
    vs. calibration size for random forest quantile initial predictors with different auditors (decision tree, constant, linear). 
    Top: bucket-based calibration without group partition.
    Middle: calibration with both bucket and group partition.
    Bottom: no group partition and categorical features excluded.
    }
    \label{fig:qr-coverage_vs_calib}
\end{figure} 

\begin{figure}[ht!]
    \centering
    \begin{subfigure}[b]{0.49\textwidth}
    \includegraphics[width=\linewidth]{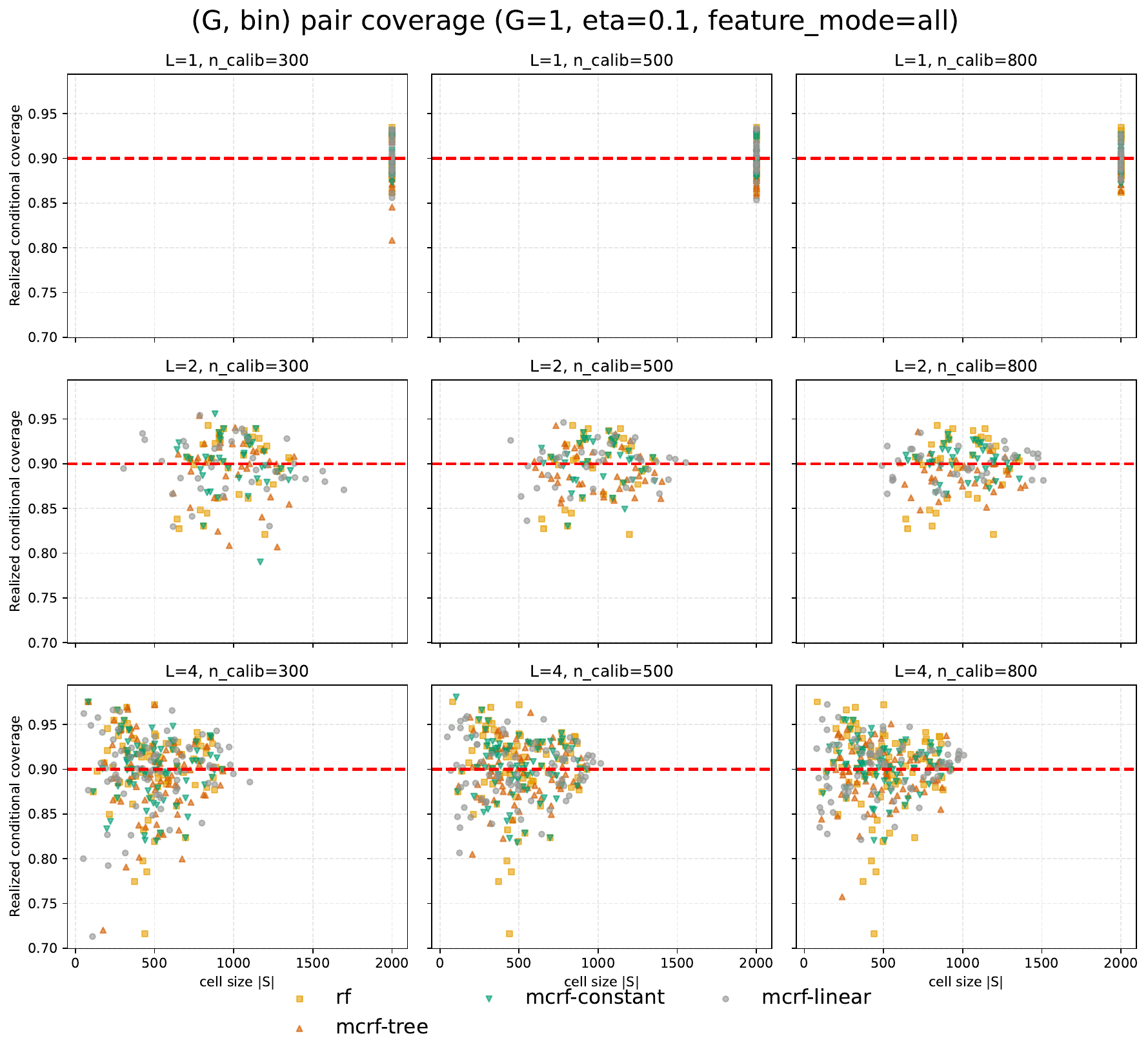}  
    \caption*{(a)}
    \end{subfigure}
    \begin{subfigure}[b]{0.49\textwidth}
    \includegraphics[width = \linewidth]{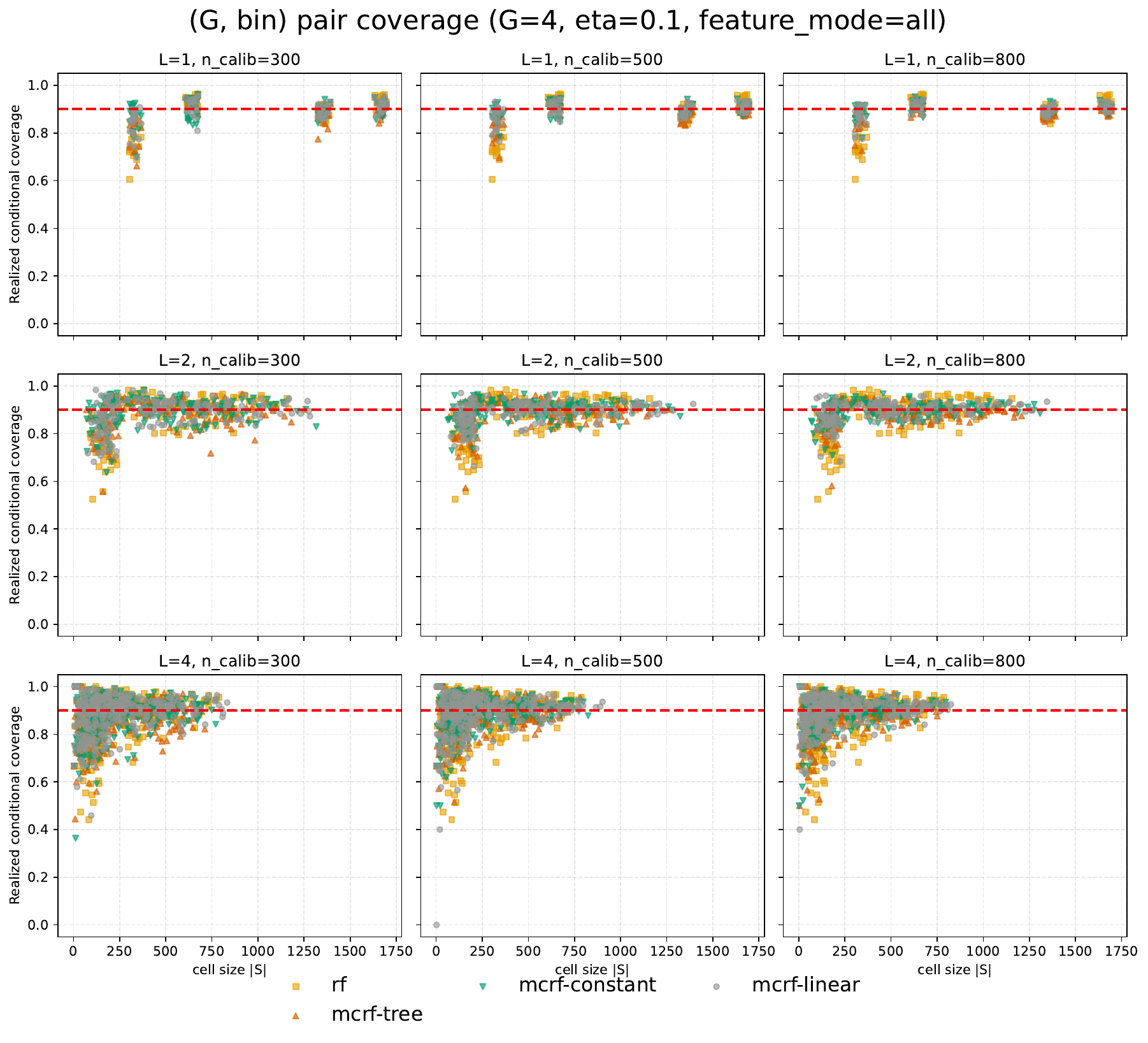}    
    \caption*{(b)}
    \end{subfigure}
    \begin{subfigure}[b]{0.49\textwidth}
    \includegraphics[width=\linewidth]{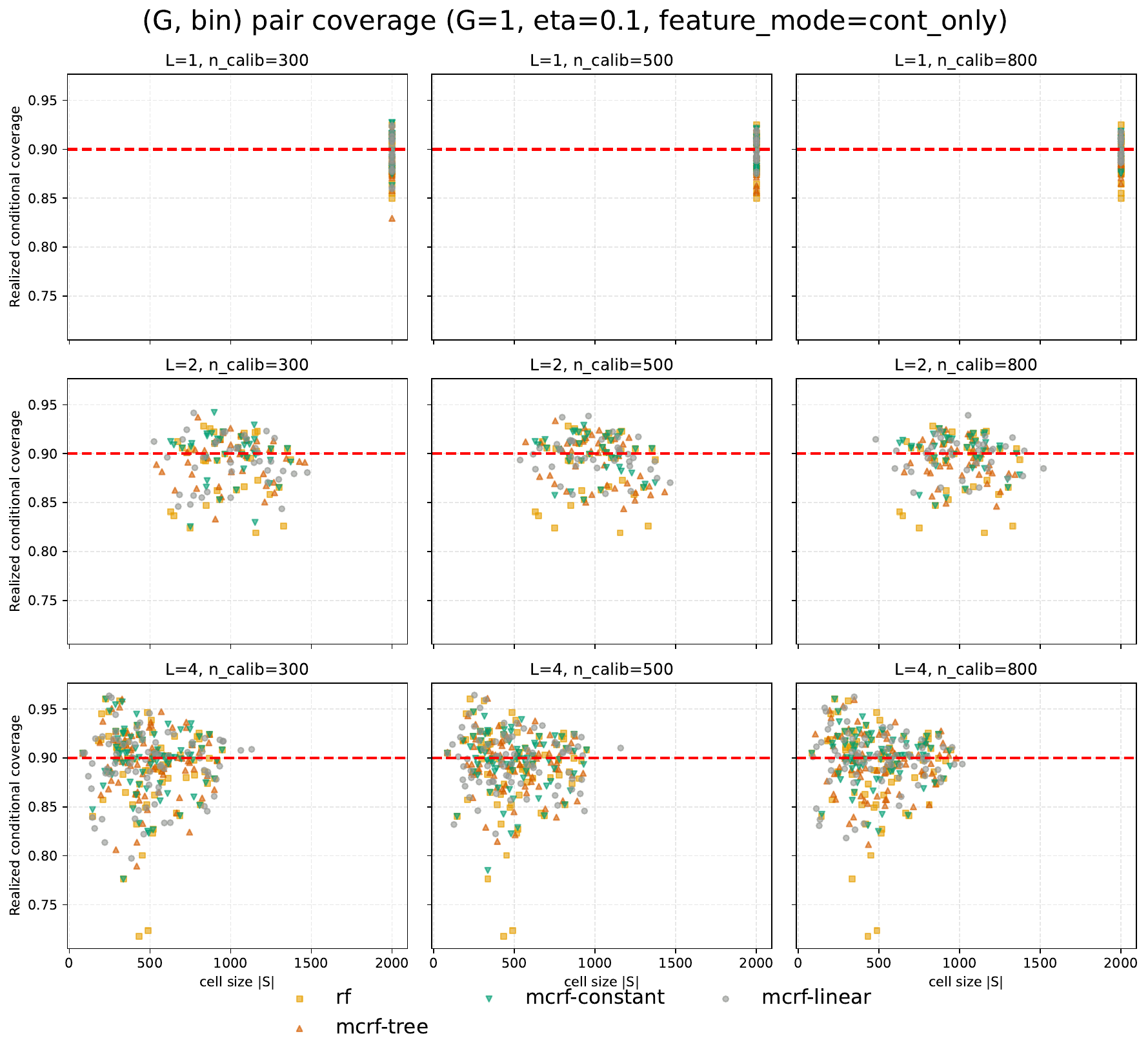}  
    \caption*{(c)}
    \end{subfigure}
    \begin{subfigure}[b]{0.49\textwidth}
    \includegraphics[width=\linewidth]{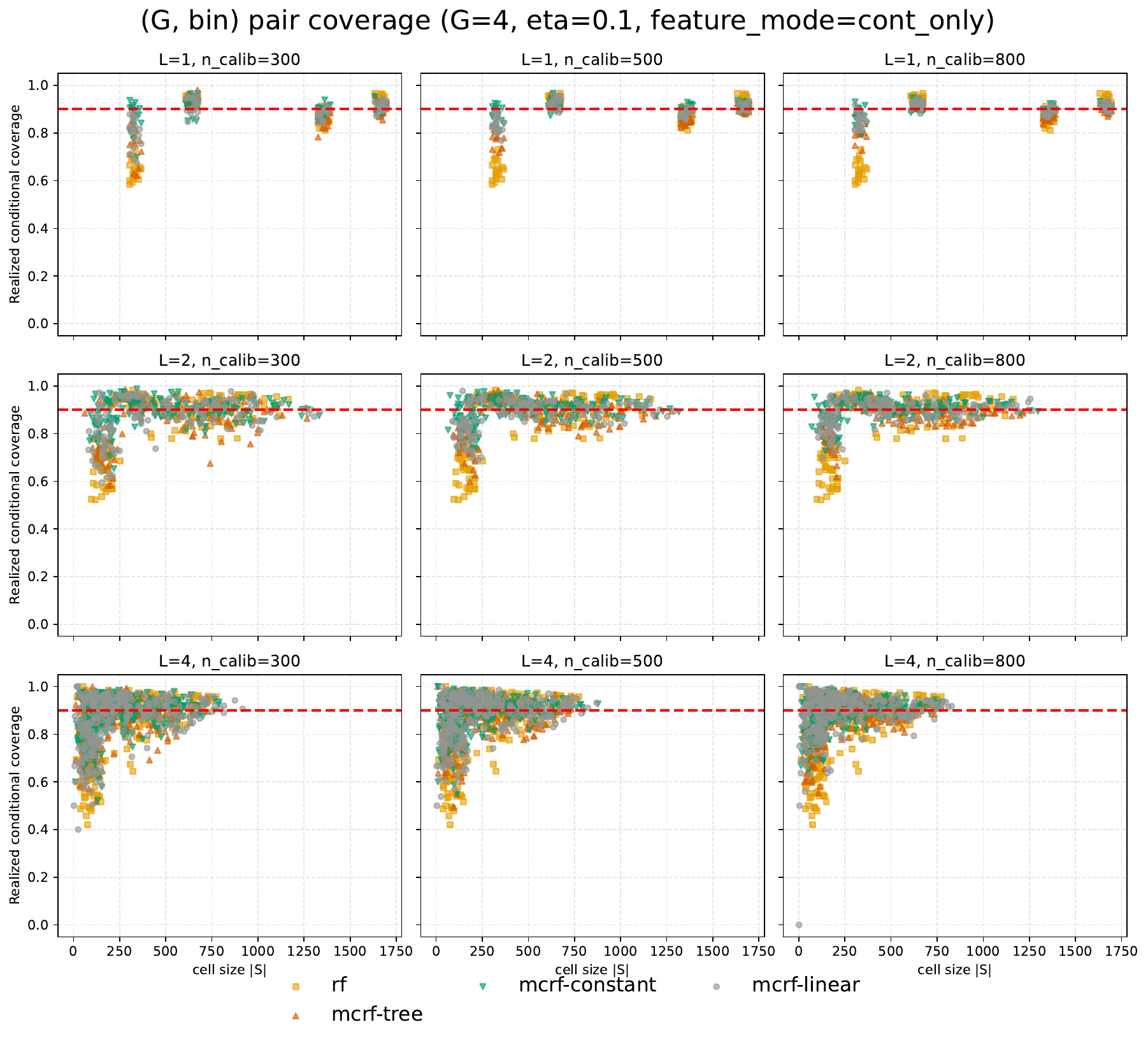}  
    \caption*{(d)}
    \end{subfigure}
    \caption{Scatter plot for empirical calibration error of $(G, l)$ pairs. Left: without group partition. Right: with group partition. }
    \label{fig:qr-coverage-scatter}
\end{figure} 
\FloatBarrier

\vskip 0.2in
\bibliography{reference}

\end{document}